\documentclass{amsart} 
\usepackage{amssymb}
\usepackage{graphicx}

\usepackage{subcaption} 

\newtheorem{theorem}{Theorem}[section]
\newtheorem{corollary}[theorem]{Corollary}
\newtheorem{lemma}[theorem]{Lemma}
\newtheorem{proposition}[theorem]{Proposition}

\newtheorem{assumption}[theorem]{Assumption}

\theoremstyle{definition}
\newtheorem{definition}[theorem]{Definition}

\theoremstyle{remark}
\newtheorem{remark}[theorem]{Remark}

\numberwithin{equation}{section}

\usepackage{enumerate}
\usepackage[shortlabels]{enumitem}

\usepackage[hidelinks]{hyperref}

\hypersetup{colorlinks,linkcolor={blue},citecolor={red}}

\usepackage{color,soul}
\soulregister\cite7

\usepackage{cleveref}

\usepackage{makecell}  

\usepackage{multirow}

\usepackage{longtable,etoolbox}

\makeatletter
\patchcmd{\LT@makecaption}
  {\sbox}
  {\normalsize\sbox}
  {}{}
\makeatother

\usepackage{threeparttablex}

\let\oldtocsubsection=\tocsubsection
\renewcommand{\tocsubsection}[2]{\hspace{1.8em}\oldtocsubsection{#1}{#2}}

\usepackage{csquotes}
\renewcommand{\mkbegdispquote}[2]{\openautoquote}

\allowdisplaybreaks

\usepackage{mathtools} 
\newcommand{\defeq}{\coloneqq}

\newcommand{\NN}{\mathbb{N}} 

\newcommand{\RR}{\mathbb{R}}
\newcommand{\CC}{\mathbb{C}}

\newcommand{\abs}[1]{\left\lvert #1 \right\rvert}
\newcommand{\norm}[1]{\left\lVert #1 \right\rVert} 
\newcommand{\maxnorm}[1]{{\left\lVert #1 \right\rVert}_{\max}} 
\newcommand{\Fnorm}[1]{{\left\lVert #1 \right\rVert}_{\mathrm{F}}} 
\newcommand{\onenorm}[1]{{\left\lVert #1 \right\rVert}_1} 
\newcommand{\twonorm}[1]{{\left\lVert #1 \right\rVert}_2} 
\newcommand{\infnorm}[1]{{\left\lVert #1 \right\rVert}_{\infty}} 

\newcommand{\card}[1]{\left\lvert #1 \right\rvert}

\newcommand{\transp}[1]{#1^{\mathrm{T}}} 
\newcommand{\inv}[1]{#1^{-1}} 
\newcommand{\pseudoinv}[1]{#1^{\dagger}} 
\newcommand{\rank}[1]{\operatorname{rank}(#1)} 
\newcommand{\tr}[1]{\operatorname{tr}(#1)} 
\newcommand{\determ}[1]{\operatorname{det}(#1)} 

\newcommand{\minimize}{\operatorname*{minimize}}

\newcommand{\subjectto}{\operatorname{subject~to}}

\begin{document}

\title[Learning Asymptotics with Convergence-Rate Guarantees using LLSQ]{Learning Asymptotics with Convergence-Rate Guarantees using Linear Least Squares}


\author{Christos N. Efrem}
\address{National Technical University of Athens, Greece}
\email{\{the first two letters of the author's first name followed by the first three letters of the author's last name (all in lowercase)\}@mail.ntua.gr}





\begin{abstract}
We introduce a new research area that is called \emph{Asymptotics Learning Theory (ALT)} and combines optimization with asymptotic analysis. In particular, ALT provides a unified approach for computing unknown constants/parameters in proven asymptotic expansions using optimization theory. In this paper, we focus on a general asymptotic form which includes a broad class of asymptotics. Furthermore, we study two powerful numerical methods, namely, \emph{sliding Linear Least Squares (sLLSQ)} and \emph{sliding Tikhonov Linear Least Squares (sT-LLSQ)}. For these techniques we rigorously prove asymptotic estimates that lead to sufficient conditions for convergence (to the correct values of unknown parameters) and convergence-rate guarantees. Despite their strengths, both methods have also limitations, e.g., slow convergence---or even, counterintuitively, divergence---in some cases. Moreover, we present fundamental applications in analytic combinatorics, a beautiful field of mathematics that deals with asymptotic enumeration of discrete structures using complex analysis. The proposed techniques complement existing approaches, such as the ratio method and its variants. Numerical examples also verify the theoretical results. Finally, we discuss interesting research directions in ALT.  
\end{abstract}

\maketitle

\setcounter{tocdepth}{2} 
\tableofcontents \label{toc}

\section{Introduction}

In this paper, we deal with a \emph{numerical problem of asymptotic nature}: given an asymptotic expansion with a finite number of unknown constants, design numerical methods---based on available data (i.e., the function's values) and optimization techniques---in order to compute these parameters; preferably with convergence-rate guarantees. This problem can be concisely stated as \emph{``learning asymptotics from data''} and is driven by the general questions: 

\vspace{1.5mm}
\noindent\fbox{%
	\parbox{\textwidth - 2\fboxsep}{%
		\centering\emph{When infinity reveals its form, can we compute its parameters? If so, how fast?}
	}%
}

\subsection{Motivation} 
\label{subsec:Motivation}

This paper is primarily inspired by \emph{Analytic Combinatorics}, which studies the asymptotics of enumeration sequences by analyzing their generating functions in the complex plane. Analytic combinatorics has emerged from the \emph{Analysis of Algorithms}---a field of computer science initiated by Donald Knuth---and further developed by many researchers, especially Philippe Flajolet. Both disciplines were greatly influenced by the classic book of de~Bruijn~\cite{deBruijn_1958} on asymptotic methods. Indicative works include the comprehensive surveys by Bender~\cite{Bender_1974} and Odlyzko~\cite{Odlyzko_1995}, the seminal paper of Flajolet~and~Odlyzko~\cite{Flajolet-Odlyzko_1990}, the important contributions by Drmota~\cite{Drmota_1997} and Hwang~\cite{Hwang_1998}, the excellent books of Wilf~\cite{Wilf_1990}, Graham,~Knuth,~and~Patashnik~\cite{Graham-Knuth-Patashnik_1989}, Greene~and~Knuth~\cite{Greene-Knuth_1990}, Knuth~\cite{Knuth_1997,Knuth_1998a,Knuth_1998b}, Szpankowski~\cite{Szpankowski_2001}, Drmota~\cite{Drmota_2009}, Flajolet~and~Sedgewick~\cite{Flajolet-Sedgewick_2009}, and Pemantle,~Wilson,~and~Melczer~\cite{Pemantle-Wilson-Melczer_2024}, as well as the numerous references therein.     

In analytic combinatorics, we study the enumeration sequence $\{f(n)\}_{n=0}^{\infty}$ of a combinatorial structure in the limit as $n$ tends to infinity, where $f$ takes nonnegative integer values. Usually, it is more convenient to study this problem using a suitable generating function, e.g., $F(z) \defeq \sum_{n=0}^{\infty} f(n) z^n$ (ordinary generating function) or $F(z) \defeq \sum_{n=0}^{\infty} \frac{f(n)}{n!} z^n$ (exponential generating function, for labeled structures), $z \in \CC$. The most common and powerful tools for this purpose belong to complex analysis, including singularity analysis and the saddle-point method. In many cases, we are able to prove an asymptotic expansion of the form    
\begin{equation} \label{eq:Common_asymptotic_expansion}
	f(n) \;\; \mathrm{or} \;\; \frac{f(n)}{n!} = \widehat{f}(n;\boldsymbol{\alpha}) \left( 1 + O(g(n)) \right)  , \quad \mathrm{as}\ n \to \infty ,
\end{equation}
where $\widehat{f}(\cdot;\boldsymbol{\alpha}) \colon \NN \to \RR_{\geq 0}$, $\boldsymbol{\alpha}$ is a finite-dimensional vector of real constants (or parameters), and the function $g \colon \NN \to \RR$ tends to zero as $n \to \infty$. For example,   
\begin{equation} \label{eq:Common_asymptotic_form}
	\widehat{f}(n;\boldsymbol{\alpha}) \defeq \alpha_1  n^{\alpha_2}  \alpha_3^n , 
\end{equation}
where $\boldsymbol{\alpha} = \transp{[\alpha_1,\alpha_2,\alpha_3]}$, $n^{\alpha_2}$ is a sub-exponential term, and $\alpha_3^n$ is an exponential term. Note that $\alpha_1$ and $\alpha_3$ are positive real numbers, whereas $\alpha_2$ is a real number. 

In some cases (e.g., enumeration problems amenable to the saddle-point method), the constant vector~$\boldsymbol{\alpha}$ is \emph{explicitly} known. Nevertheless, in other cases (e.g., when generating functions satisfy complicated functional equations), complex-analytic techniques can prove \emph{only the existence of $\boldsymbol{\alpha}$ in a particular set}. In such circumstances, there is a \emph{lack of numerical methods} to compute the unknown constants (i.e., the vector $\boldsymbol{\alpha}$) \emph{with convergence-rate guarantees}.  

Furthermore, several researchers have discovered \emph{universal phenomena} in analytic combinatorics for various kinds of discrete structures, such as frequent exponents in asymptotics and Gaussian limit laws. Many of these combinatorial problems involve \emph{systems of functional equations} that are satisfied by generating functions (of several complex variables). The main tools required to prove such results are singularity analysis (the transfer lemma of Flajolet~and~Odlyzko~\cite{Flajolet-Odlyzko_1990}) and \emph{non-constructive} techniques, namely, the \emph{implicit function theorem} and the \emph{Weierstrass preparation theorem}. A pioneering work in this direction---concerning systems of finitely many equations and whose dependency graph is strongly connected---was carried out by Drmota~\cite{Drmota_1997} and, independently, by Lalley~\cite{Lalley_1993,Lalley_2001} and Woods~\cite{Woods_1997}. For the specific case of polynomial systems, this result is known as the \emph{Drmota-Lalley-Woods (DLW) theorem}~\cite[p.~489, Theorem~VII.6]{Flajolet-Sedgewick_2009}. The aforementioned works were later extended to systems of infinitely many equations by Lalley~\cite{Lalley_2004}, Morgenbesser~\cite{Morgenbesser_2010}, and Drmota,~Gittenberger,~and~Morgenbesser~\cite{Drmota-Gittenberger-Morgenbesser_2012}. Another remarkable extension of the DLW theorem, to polynomial systems whose dependency graph is \emph{not} strongly connected, was given by Banderier~and~Drmota~\cite{Banderier-Drmota_2015}. 

Recently, universal phenomena continue to be observed in other papers. For example, Drmota,~Noy,~and~Yu~\cite{Drmota-Noy-Yu_2022} as well as Drmota~and~Hainzl~\cite{Drmota-Hainzl_2023} studied functional equations with one \emph{catalytic} variable, which are frequently used in the enumeration of lattice paths and planar maps. In particular, under appropriate conditions, the number of such structures is asymptotically given by~\eqref{eq:Common_asymptotic_expansion}---without $n!$---and~\eqref{eq:Common_asymptotic_form}, where $n$ also belongs to a suitable residue class, $g(n) = n^{-1}$, and $\alpha_2 = -3/2$ (resp.~$\alpha_2 = -5/2$) for linear (resp.~non-linear) catalytic equations. Moreover, Castellv{\'i},~Drmota,~Noy,~and~Requil{\'e}~\cite{Castellvi-Drmota-Noy-Requile_2024} showed that, given any integers $t \geq 1$ and $0 \leq k \leq t$, the number of $k$-connected chordal graphs with $n$ labeled vertices and tree-width at most $t$, denoted by $f_{t,k}(n)$, admits the asymptotic expansion   
\begin{equation} \label{eq:Chordal_graphs_asymptotic_expansion}
	\frac{f_{t,k}(n)}{n!} = \alpha_{1,t,k} \, n^{-5/2} \, \alpha_{2,t,k}^n \left( 1 + O\left( \frac{1}{n} \right) \right)  , \quad \mathrm{as}\ n \to \infty ,
\end{equation}
for some real constants $\alpha_{1,t,k} > 0$ and $\alpha_{2,t,k} > 1$, which depend on $t$ and $k$. Observe that the exponent $-5/2$ is invariant, i.e., it remains the same for all $t$ and $k$.

However, the unknown constants appearing in such (universal) asymptotic expansions have \emph{not} been rigorously computed until now. In fact, mostly rough estimates were obtained.\footnote{A well-known exception is the special case of (computable) \emph{algebraic} constants, i.e., numbers that are solutions to polynomial equations with rational (or, equivalently, integer) coefficients. For example, see~\cite[p.~483, Theorem~VII.5]{Flajolet-Sedgewick_2009} about context-free structures.} The reason is that these constants are given as solutions to (possibly non-polynomial) systems of analytic equations as in~\cite[Theorem~1]{Drmota_1997}, or they are almost completely unknown (except for their existence) as in \cite[Theorem~1.1]{Castellvi-Drmota-Noy-Requile_2024}/\eqref{eq:Chordal_graphs_asymptotic_expansion}. In addition, existing works (e.g.,~\cite[Section~4.2]{Bodirsky-Fusy-Kang-Vigerske_2007}~and~\cite[Section~9]{Drmota-Fusy-Kang-Kraus-Rue_2011}) mainly use an iterative procedure, based on \emph{truncated Taylor series}, to approximate the exponential-growth constant (i.e., the parameter $\alpha_3$ in~\eqref{eq:Common_asymptotic_form}). The convergence of this approach has been observed numerically, but \emph{without theoretical guarantees} (i.e., convergence to the desired value as the order of truncation tends to infinity). Therefore, a unified methodology for the numerical computation of unknown parameters is an \emph{important open problem} that deserves further investigation. 

Last but not least, analytic combinatorics and optimization theory have been so far developed \emph{separately}, as two distinct fields of mathematics. As a result, their intersection remains a \emph{largely (if not completely) unexplored area}.

\subsection{Related Work}
\label{subsec:Related_work}

The most widely-used technique to compute unknown constants in asymptotic expansions is the \emph{Ratio Method (RM)}. Specifically, RM was invented by Domb~and~Sykes~\cite{Domb-Sykes_1956,Domb-Sykes_1957,Domb-Sykes_1961}, and therefore it is also referred to as the \emph{Domb-Sykes method}. This technique was originally used to numerically compute the radius of convergence of a power series $\sum_{n=0}^{\infty} c_n z^n$, $z \in \CC$, with nonnegative real coefficients $\{c_n\}_{n=0}^{\infty}$.\footnote{In general, the radius of convergence $R \in [0,\infty]$ of a power series $\sum_{n=0}^{\infty} c_n z^n$, with complex coefficients $\{c_n\}_{n=0}^{\infty}$, is given by the \emph{Cauchy-Hadamard formula}: $R^{-1} = \limsup\limits_{n \to \infty} \abs{c_n}^{1/n}$.} In particular, if the coefficients satisfy\footnote{In such a case, the radius of convergence is equal to the inverse of $\alpha_3$, i.e., $R = \alpha_3^{-1}$.} 
\begin{equation*}
	c_n \sim {\alpha_1} n^{\alpha_2} {\alpha_3^n} , \quad \mathrm{as}\ n \to \infty ,
\end{equation*}
for some constants $\alpha_1,\alpha_3 \in \RR_{>0}$ and $\alpha_2 \in \RR$ (which are called the \emph{critical parameters}), then the ratio 
\begin{equation*}
	r_n \defeq \frac{c_{n+1}}{c_n} \sim \left( 1 + \frac{1}{n} \right)^{\alpha_2} {\alpha_3} \sim \left( 1 + \frac{\alpha_2}{n} \right) {\alpha_3} , \quad \mathrm{as}\ n \to \infty .
\end{equation*}
In the second asymptotic equivalence, we have used the fact that $(1+x)^{\theta} = 1 + \theta x + O(x^2)$, as $x \to 0$, for every $\theta \in \RR$, which implies that $(1+x)^{\theta} \sim 1 + \theta x$, as $x \to 0$. As a result, if we plot the ratio $r_n$ versus $\frac{1}{n}$ (which is known as the \emph{Domb-Sykes plot} according to Hinch~\cite[\S{8.1}]{Hinch_1991}), then the resulting curve tends to a straight line as $n \to \infty$ or, equivalently, as $\frac{1}{n}\to 0$. The straight-line asymptote, $y = \alpha_3 + ({\alpha_2}{\alpha_3})x$ with $x = \frac{1}{n}$, intersects the vertical axis at $\alpha_3$ and has slope ${\alpha_2}{\alpha_3}$. The latter procedure can be used to (roughly) estimate the parameters $\alpha_2$ and $\alpha_3$. A more detailed analysis will be given in Section~\ref{sec:Ratio_method}. Despite its simplicity, RM and its variants were successfully applied in several areas, such as fluid mechanics (see the papers by Van~Dyke~\cite{VanDyke_1974}, Mercer~and~Roberts~\cite{Mercer-Roberts_1990}) and statistical physics (see the works of Gaunt~and~Guttmann~\cite[\S{II}]{Gaunt-Guttmann_1974}, Guttmann~\cite[\S{2}]{Guttmann_1989}).  

Numerical techniques---based on power series---appeared in the scientific literature with multiple but similar names: \emph{``asymptotic analysis of coefficients''} by Gaunt~and~Guttmann~\cite{Gaunt-Guttmann_1974}, \emph{``asymptotic analysis of power-series expansions''} by Guttmann~\cite{Guttmann_1989}, and \emph{``series analysis''} by Guttmann~and~Jensen~\cite{Guttmann-Jensen_2009}. Besides RM, other significant approaches include \emph{sequence extrapolation}, \emph{Pad{\'e} approximants}, \emph{differential approximants}, and \emph{variable transformations} \cite{Gaunt-Guttmann_1974,Guttmann_1989,Guttmann-Jensen_2009}. In addition, the \emph{direct fitting} (also known as multi-parameter fitting~\cite[\S{3.8}]{Guttmann_1989}) and \emph{series extension} (i.e., coefficient prediction) are powerful methods used by Guttmann~\cite{Guttmann_2015,Guttmann_2016} to analyze power series with non-algebraic singularities, involving a stretched-exponential term. All these techniques are based on \emph{unproven} asymptotic expansions (nature of singularities), hence their convergence analysis is impossible.\footnote{Nevertheless, with \emph{proven} asymptotics, RM is amenable to convergence analysis as will be shown in Section~\ref{sec:Ratio_method}.} On the one hand, they can provide strong intuition about asymptotics, which has been (in several cases) rigorously proved later. On the other hand, from a strictly mathematical point of view, they are heuristic techniques (i.e., \emph{without convergence guarantees}) that lead to \emph{conjectures} about asymptotic forms and \emph{rough estimations} of unknown constants (which may not exist after all). For example, see the papers by Defant,~Elvey~Price,~and~Guttmann~\cite{Defant-ElveyPrice-Guttmann_2021}, Guttmann~and~Jensen~\cite{Guttmann-Jensen_2022}, as well as Conway~and~Guttmann~\cite{Conway-Guttmann_2025}. 

Furthermore, Pivoteau~and~Salvy~\cite[Section~5]{Pivoteau-Salvy_2025} have recently proposed a numerical algorithm to compute the radius of convergence of generating functions that satisfy a system of equations, which are not necessarily polynomial. The algorithm takes as input the system of equations and uses an \emph{oracle} to make decisions. The main limitation is that its correctness relies on an \emph{unproven} hypothesis, i.e., that Schanuel's conjecture is true. However, this assumption is not required if the system does not include constructions of sets or cycles.

In general, it is of paramount importance to guarantee the convergence of a numerical method, otherwise it cannot be rigorously trusted. Motivated by recent advances in analytic combinatorics where asymptotics is \emph{provable in advance} (as described earlier), we propose a different approach. Based on a \emph{proven} asymptotic expansion with unknown constants, we first study their \emph{uniqueness}, and then design numerical methods to compute them. In this way, we are able to provide \emph{convergence-rate guarantees}, whenever possible. An advantage of the proposed methodology is its close connection with \emph{optimization theory}, a key characteristic that is entirely missing from current approaches. We also cover a broad spectrum of problems with a unified and flexible theory, instead of solving problems on a case-by-case basis.

\subsection{Summary of Contributions and Outline}

In this paper, we present a new research field---at the intersection of optimization and asymptotic analysis---that is called \emph{Asymptotics Learning Theory (ALT)}. Specifically, the objective of ALT is to efficiently compute unknown constants/parameters in proven asymptotic expansions using optimization techniques. Moreover, we study two numerical methods: \emph{sliding Linear Least Squares (sLLSQ)} and \emph{sliding Tikhonov Linear Least Squares (sT-LLSQ)}, which are applicable to a variety of asymptotic forms. The main results consist of Theorems~\ref{thm:The_sLLSQ_method}~and~\ref{thm:The_sT-LLSQ_method}, for each method respectively. The basic advantage of sLLSQ and sT-LLSQ is that every optimal solution is unique and admits a closed-form expression. In this way, we can derive rigorous asymptotic estimates, which are instrumental in the convergence analysis of both methods. We provide several applications in analytic combinatorics and emphasize the complementarity with well-known techniques, e.g., the ratio method and its variants. Interestingly, sLLSQ and sT-LLSQ have also their own limitations, despite remarkable success in most cases. For example, sLLSQ may not converge at all, or the convergence may be quite slow. In addition, numerical experiments confirm the theoretical findings.  

The paper is organized as shown in the~\hyperref[toc]{\contentsname}. Note that preliminary material (i.e., the required background) is given in Appendix~\ref{sec:Mathematical_background}.

\subsection{Notation}

We write $\defeq$ to state that the left-hand side is defined to be equal to the right-hand side. Italic lowercase and calligraphic uppercase letters represent real scalars and sets, respectively, e.g., $x$ and $\mathcal{S}$. Moreover, boldface lowercase (resp. uppercase) letters denote finite-dimensional real column-vectors (resp. matrices), for example, $\mathbf{x}$ and $\mathbf{A}$. Every component of a vector or matrix is represented by the corresponding (italic) lowercase letter with a subscript, e.g., $x_i$ is the $i$-th element of a vector $\mathbf{x}$ and $a_{i,j}$ is the $(i,j)$-th entry of a matrix $\mathbf{A}$. Given an $m \times k$ matrix $\mathbf{A}$, we use the shorthand for row selection, that is, $\mathbf{A}(i,:) \defeq [a_{i,1},\dots,a_{i,k}]$, for all $i \in \{1,\dots,m\}$. $\mathbf{I}$, $\mathbf{O}$, $\mathbf{1}$ and $\mathbf{0}$ are the identity and zero matrix, all-ones and zero vector, respectively, of appropriate dimensions. For a sequence of vectors $\{\mathbf{x}_n\}_{n = 0}^{\infty}$, we denote by $x_{i,n}$ the $i$-th component of the vector $\mathbf{x}_n$. In addition, $\abs{x}$ stands for the absolute value of a real number $x$, while $\card{\mathcal{S}}$ represents the cardinality of a set $\mathcal{S}$. The exponential function is denoted by $e^{\cdot}$ or $\exp(\cdot)$, and the natural logarithm by $\log(\cdot)$. In general, given an $m$-dimensional real vector $\mathbf{x}$, the component-wise logarithm is defined by $\log(\mathbf{x}) \defeq \transp{[\log(x_1),\dots,\log(x_m)]}$ (assuming $x_i > 0$, for all $i$) and, with a slight abuse of notation, $\log(1+\mathbf{x}) \defeq \log(\mathbf{1}+\mathbf{x}) = \transp{[\log(1+x_1),\dots,\log(1+x_m)]}$ (assuming $x_i > -1$, for all $i$). The Cartesian product of a finite collection of sets $\{\mathcal{S}_j\}_{j=1}^k$, $k \in \NN$, is denoted by $\prod_{j=1}^k {\mathcal{S}_j} \defeq {\mathcal{S}_1} \times \cdots \times {\mathcal{S}_k}$. 

The sets of real and complex numbers are denoted by $\RR$ and $\CC$, respectively. $\RR_{>0} \defeq \{x \in \RR : x > 0\}$, $\RR_{\geq 0} \defeq \{x \in \RR : x \geq 0\}$, $\NN \defeq \{1,2,\dots\}$ and $\NN_0 \defeq \NN \cup \{0\}$. Given any positive integers $m$ and $k$, $\RR^m$ (resp. $\RR^{m \times k}$) is the set of all $m$-dimensional real vectors (resp. real matrices with $m$ rows and $k$ columns). Also, $\tr{\cdot}$, $\determ{\cdot}$ and $\inv{(\cdot)}$ represent the trace, determinant and inverse of a square matrix, respectively. The rank of a matrix $\mathbf{A} \in \RR^{m \times k}$ is denoted by $\rank{\mathbf{A}}$, where $0 \leq \rank{\mathbf{A}} \leq \min(m,k)$. $\transp{\mathbf{A}}$ and $\pseudoinv{\mathbf{A}}$ denote the transpose and pseudoinverse (also known as Moore-Penrose inverse) of a matrix $\mathbf{A}$, respectively. In particular, when $\mathbf{A}$ has linearly independent columns (equivalently, $\transp{\mathbf{A}} \mathbf{A}$ is invertible) its pseudoinverse is given by $\pseudoinv{\mathbf{A}} = \inv{(\transp{\mathbf{A}} \mathbf{A})} \transp{\mathbf{A}} \in \RR^{k \times m}$; in such a case $\pseudoinv{\mathbf{A}}$ is a left inverse of $\mathbf{A}$, that is, $\pseudoinv{\mathbf{A}} \mathbf{A} = \mathbf{I}$. Also, $\norm{\cdot}$ denotes a vector/matrix norm, and $\sum_{i=l}^u {a_i} \defeq 0$, $\prod_{i=l}^u {a_i} \defeq 1$ whenever $l > u$. 

Furthermore, we use the following asymptotic notation given any functions \linebreak $f,g \colon \NN \to \RR$. The $O$-notation: $f(n) = O(g(n))$ as $n \to \infty$ if, and only if, there exist numbers $c \in \RR_{>0}$ and $n_0 \in \NN$, both independent of $n$, such that $\abs{f(n)} \leq c \abs{g(n)}$ for all integers $n \geq n_0$.\footnote{The $O$-notation can be extended to any functions $f,g \colon \RR \to \RR$ and any (finite) limiting value $L \in \RR$: $f(x) = O(g(x))$ as $x \to L$ if, and only if, there exist numbers $c,\delta \in \RR_{>0}$, both independent of $x$, such that $0 < \abs{x-L} < \delta \implies \abs{f(x)} \leq c \abs{g(x)}$.} The $\Omega$-notation: $f(n) = \Omega(g(n))$ if, and only if, $g(n) = O(f(n))$. The $\Theta$-notation: $f(n) = \Theta(g(n))$ as $n \to \infty$ if, and only if, there exist numbers $c_1,c_2 \in \RR_{>0}$ and $n_0 \in \NN$, all independent of $n$, such that $c_1 \abs{g(n)} \leq \abs{f(n)} \leq c_2 \abs{g(n)}$ for all integers $n \geq n_0$; this is equivalent to $f(n) = \Omega(g(n))$ and $f(n) = O(g(n))$. The $o$-notation: $f(n) = o(g(n))$ as $n \to \infty$ if, and only if, $\lim_{n \to \infty} \frac{f(n)}{g(n)} = 0$. Note that $f(n) = o(g(n))$ implies $f(n) = O(g(n))$. We say that a function $f$ \emph{vanishes asymptotically} if $f(n) = o(1)$, i.e., it converges to zero as $n \to \infty$. The asymptotic equivalence: $f(n) \sim g(n)$ as $n \to \infty$ if, and only if, $\lim_{n \to \infty} \frac{f(n)}{g(n)} = 1$; this can be equivalently written as $f(n) = g(n)(1+o(1))$, provided that $g(n) \neq 0$ for sufficiently large $n$. To indicate (component-wise) lower order terms, we write $\mathbf{r}_n = \transp{[s_1(n),\dots,s_k(n)]} + \mathbf{l.o.t.}$ whenever $\mathbf{r}_n = \transp{[s_1(n)(1+o(1)),\dots,s_k(n)(1+o(1))]}$ as $n \to \infty$; when $k=1$, we simply write $\mathrm{l.o.t.}$ in normal font.   

Finally, asymptotic notation appearing on both sides of an equation has the following meaning, as $n \to \infty$. Given any functions $h_1,h_2,h_3,h_4 \colon \NN \to \RR$, the equation $h_1(n) + O(h_2(n)) = h_3(n) + O(h_4(n))$ means that for every function $f(n) = O(h_2(n))$, there exists a function $g(n) = O(h_4(n))$ such that $h_1(n) + f(n) = h_3(n) + g(n)$, for all $n \in \NN$. We have similar interpretations if we replace an $O$-term with an $\Omega$/$\Theta$/$o$-term. In general, equations with multiple asymptotic terms on each side are treated as follows. Each asymptotic term on the \emph{left-hand side} of an equation is interpreted as \emph{``for every function''} (i.e., using a universal quantifier $\forall$), whereas each one on the \emph{right-hand side} as \emph{``there exists a function''} (i.e., using an existential quantifier $\exists$), so that the resulting equation (after all function substitutions) holds for all $n \in \NN$. In other words, the left-hand side provides a more accurate estimate than the right-hand side.

\section{Asymptotics Learning Theory (ALT)}
\label{sec:Asymptotics_Learning_Theory}

\subsection{The General Problem}

Let us suppose that we have a \emph{proven asymptotic expansion} of order $p \in \NN_0$, that is, 
\begin{equation} \label{eq:Proven_asymptotic_expansion}
	f(n) = \widehat{f}(n;\boldsymbol{\alpha}) \left( 1 + \sum_{l=1}^p \beta_l g_l(n) + O(g_{p+1}(n)) \right)  , \quad \mathrm{as}\ n \to \infty ,
\end{equation}
where the functions $f \colon \NN \to \RR_{\geq 0}$ and $\widehat{f} \colon {\NN \times \mathcal{S}} \to \RR_{\geq 0}$ are \emph{asymptotically positive} (i.e., $f(n), \widehat{f}(n;\cdot) > 0$ for sufficiently large $n$), and $\boldsymbol{\alpha} = \transp{[\alpha_1,\dots,\alpha_k]} \in \mathcal{S}$ is the vector of \emph{unknown} parameters/constants (\emph{independent} of $n$) with $k \in \NN$. The coefficients $\beta_l \in \RR$, $l \in \{1,\dots,p\}$, are also \emph{independent} of $n$ (only their existence is necessary).\footnote{The coefficients $\{\beta_l\}_{l=1}^p$ are treated as \emph{symbolic real constants} in our analysis, e.g., when computing Puiseux series (indicative examples will be given in Section~\ref{sec:Fundamental_applications_in_analytic_combinatorics}).} Moreover, the functions $\{g_l \colon \NN \to \RR\}_{l=1}^{p+1}$ constitute a \emph{decreasing asymptotic scale}  as $n \to \infty$, that is, $g_{l+1}(n) = o(g_l(n))$, for all $l \in \{1,\dots,p\}$, with $g_1(n) = o(1)$.\footnote{All functions may be defined for large enough integer $n$, instead of all $n \in \NN$; this does not affect the asymptotic analysis.} Observe that $g_l(n) = o(1)$, for all $l \in \{1,\dots,p+1\}$, and therefore $f(n) = \widehat{f}(n;\boldsymbol{\alpha}) (1+o(1))$ or, equivalently, $f(n) \sim \widehat{f}(n;\boldsymbol{\alpha})$ as $n \to \infty$. For this reason, the leading/dominant function $\widehat{f}$ is called an \emph{asymptotic form} of $f$. Note that the numbers $p \in \NN_0$, $k \in \NN$, and the \emph{nonempty} set $\mathcal{S} \subseteq \RR^k$ (not necessarily compact, i.e., closed and bounded) are all given. It is also emphasized that the existence of $\boldsymbol{\alpha} \in \mathcal{S}$ is \emph{guaranteed in advance}, so that \eqref{eq:Proven_asymptotic_expansion} is satisfied.\footnote{For example, the existence of $\boldsymbol{\alpha}$ may be proved using techniques from complex analysis, as mentioned in Section~\ref{subsec:Motivation}.} The simplest case is when $p=0$, that is, $f(n) = \widehat{f}(n;\boldsymbol{\alpha}) \left( 1 + O(g_1(n)) \right)$.

Table~\ref{tbl:Asymptotics_learning_theory_General_problem} summarizes the general context of Asymptotics Learning Theory (ALT), thus presenting the scope of this work at a glance. In simple words, the problem of learning the asymptotics of $f$ is how, and how fast, we can compute the unknown vector $\boldsymbol{\alpha}$, given the infinite sequence $\{f(n)\}_{n \in \NN}$ and the proven asymptotic expansion \eqref{eq:Proven_asymptotic_expansion}. ALT can be considered as an \emph{inverse problem involving asymptotics}: we try to find/reconstruct the unknown parameters, which are not directly measurable, by observing the function's values (which play the role of indirect measurements). 

In practice, however, we only know a \emph{finite} number of values $\{f(1),\dots,f(N)\}$, for some $N \in \NN$, due to computational limitations (e.g., finite time and space resources). In theory, we assume that an infinite sequence is available in order to study the convergence/divergence (and possibly the convergence rate) of a numerical method as $n \to \infty$. Moreover, the data sequence $\{f(n)\}_{n \in \NN}$ is assumed to be  \emph{exact} (e.g., using symbolic arithmetic), or \emph{of sufficiently high accuracy} (e.g., using variable/arbitrary precision arithmetic) so that the theoretical analysis is preserved. 

\begin{remark}
	The data sequence $\{f(n)\}_{n \in \NN}$ can be obtained from: 1) \emph{explicit formulas}, 2) \emph{recursive formulas}, 3) coefficient extraction from \emph{explicit generating functions}, or 4) iteration of \emph{(systems of) functional equations} that are satisfied by generating functions; see the works of~Pivoteau,~Salvy,~and~Soria~\cite[Sections~3~and~4]{Pivoteau-Salvy-Soria_2008}, \cite[Algorithm~newtonSeries]{Pivoteau-Salvy-Soria_2012}. The computational complexity of each procedure is crucial and depends on the problem at hand. Normally, higher number of terms requires more computational resources (i.e., time and/or space). However, the data sequence is \emph{given} in ALT, which primarily focuses on convergence analysis.   
\end{remark}

\renewcommand{\arraystretch}{1.5}
\begin{table}[t!]
	\centering
	\caption{The general problem in Asymptotics Learning Theory. Additional assumptions may be required to ensure the uniqueness of $\boldsymbol{\alpha}$ and the applicability of a numerical method.} 
	\label{tbl:Asymptotics_learning_theory_General_problem}
	\resizebox{\textwidth}{!}{%
		\begin{tabular}{|l l|} 
			\hline 
			\textbf{Input data (\emph{given}):} & \makecell[l]{A sequence $\{f(n)\}_{n \in \NN}$, $p \in \NN_0$, $k \in \NN$, a nonempty \\ set $\mathcal{S} \subseteq \RR^k$, functions $\widehat{f}$ and $\{g_l\}_{l=1}^{p+1}$, and fixed real \\ coefficients  $\{\beta_l\}_{l=1}^{p+1}$ (only their existence is necessary).}   \\
			\hline
			
			\makecell[l]{\textbf{Asymptotic expansion of} \\ \textbf{the function $f$ (\emph{proven}):}} & \makecell[l]{$f(n) = \widehat{f}(n;\boldsymbol{\alpha}) \left( 1 + \sum_{l=1}^p \beta_l g_l(n) + O(g_{p+1}(n)) \right)$, \\ as $n \to \infty$, where the vector $\boldsymbol{\alpha} \in \mathcal{S}$ is \emph{guaranteed} \\ \emph{to exist}, and the functions $\{g_l\}_{l=1}^{p+1}$ constitute a \\ \emph{decreasing asymptotic scale}, i.e., $g_{l+1}(n) = o(g_l(n))$, \\ for all $l \in \{1,\dots,p\}$, with $g_1(n) = o(1)$.}  \\
			\hline
			
			\makecell[l]{\textbf{Finite-dimensional vector} \\ \textbf{of \emph{unknown} parameters:}} & $\boldsymbol{\alpha} = \transp{[\alpha_1,\dots,\alpha_k]} \in \mathcal{S}$, which is \emph{independent} of $n$.  \\
			\hline
			
			\textbf{Goal/Objective:} & \makecell[l]{Efficient computation of the unknown vector $\boldsymbol{\alpha}$, \\ either completely (i.e., all its components) or \\ partially (i.e., some of its components, but not all), \\ with \emph{convergence-rate guarantees}. Preferably, \\ complete computation of $\boldsymbol{\alpha}$ with low complexity \\ (per iteration) and fast convergence.}  \\
			\hline
		\end{tabular}%
	}
\end{table}

Furthermore, ALT is intimately connected with \emph{optimization}, which is considered its cornerstone. In particular, the function $\widehat{f} (n;\cdot) \colon {\mathcal{S} \to \RR_{\geq 0}}$ can be seen as an estimator of $\widehat{f}(n;\boldsymbol{\alpha})$, for every $n \in \NN$. Given a fixed $m \in \NN$, which is called the \emph{length of the sliding window} or the \textit{number of samples per iteration}, we define the $m$-dimensional vector function $\widehat{\mathbf{f}}(n;\mathbf{x}) \defeq \transp{[\widehat{f}(n;\mathbf{x}),\dots,\widehat{f}(n+m-1;\mathbf{x})]}$ of the variable $\mathbf{x} \in \mathcal{S}$ and the $m$-dimensional vector $\mathbf{f}(n) \defeq \transp{[f(n),\dots,f(n+m-1)]}$, for each $n \in \NN$. Then, by choosing a suitable objective function $F_n(\mathbf{x}) = d \left( \widehat{\mathbf{f}}(n;\mathbf{x}),\mathbf{f}(n) \right)$, which is called a \emph{loss function} and quantifies the divergence (or relative distance) of $\widehat{\mathbf{f}}(n;\mathbf{x})$ with respect to $\mathbf{f}(n)$, we formulate a sequence of optimization problems: 
\begin{align} \label{eq:General_optimization_problem}
	\begin{split}
		& \minimize_{\mathbf{x}} \;\; F_n(\mathbf{x}) = d \left( \widehat{\mathbf{f}}(n;\mathbf{x}),\mathbf{f}(n) \right)   \\
		& \subjectto \;\; \mathbf{x} \in \mathcal{S}  .  
	\end{split} 
\end{align}
Provided that a global minimizer exists for sufficiently large $n$ (i.e., there is an $\mathbf{x}_n^* \in \mathcal{S}$ such that $F_n(\mathbf{x}_n^*) \leq F_n(\mathbf{x})$ for all $\mathbf{x} \in \mathcal{S}$), we study the generated sequence of optimal solutions $\{\mathbf{x}_n^*\}_{n \in \NN}$. The ultimate goal is to guarantee that $\mathbf{x}_n^*$ converges to the unknown vector $\boldsymbol{\alpha}$, i.e., $\lim_{n \to \infty} \mathbf{x}_n^* = \boldsymbol{\alpha}$. The \emph{convergence rate} as well as the \emph{iteration complexity} (i.e., the complexity of computing the optimal solution $\mathbf{x}_n^*$) play important roles in ALT. Preferably, the former should be high, whereas the latter should be low. The choice of the loss function $F_n$ definitely affects these factors, and therefore it must be selected carefully based on the specific form of $\widehat{f}$. 

In general, the \emph{driving force} in ALT consists of the following research questions:  
\begin{itemize}
	\item \emph{Uniqueness}---Given that the unknown parameters exist, are they unique?  
	
	\item \emph{Design of computational techniques}---Given that the unknown parameters exist and are unique, can we design numerical methods (i.e., iterative algorithms based on optimization) to compute these parameters? 
	
	\item \emph{Analysis of computational techniques}---Can these algorithms achieve \emph{provable convergence} to the correct values of parameters? If so, what is their speed of convergence (\emph{convergence-rate guarantees})? Do they have low computational complexity (\emph{algorithmic efficiency})? What are their \emph{fundamental limitations} (i.e., when do they fail)? How can we improve them? Such questions are of the utmost importance.
\end{itemize}

\subsection{Perspective}

A surprising fact about ALT is that the existence of constants in an asymptotic expansion is initially proved using \emph{non-constructive} techniques (e.g., the implicit function theorem and the Weierstrass preparation theorem) but, based on their proven existence, we can usually compute them, i.e., design sequences that converge to them. Roughly speaking, we can turn implicit constants into ``explicit'' ones by applying a \emph{constructive} numerical method to the function's values (i.e., learning asymptotics from data). In general, this is a \emph{rare phenomenon}: non-constructive existence proofs may have constructive outcomes by using them in an entirely different way, through the lens of computation.

\section{Theoretical Analysis and Main Results}

In this paper, we study the special case where the asymptotic form $\widehat{f}$ in \eqref{eq:Proven_asymptotic_expansion} can be written as a product of exponential functions, that is,   
\begin{equation} \label{eq:Proven_asymptotic_form}
	\widehat{f}(n;\mathbf{x}) \defeq \prod_{j=1}^k \exp\left( \varphi_j(n) u_j(x_j) \right) = \exp\left( \sum_{j=1}^k \varphi_j(n) u_j(x_j) \right)  ,
\end{equation}
where $\varphi_j \colon \NN \to \RR$, $u_j \colon \mathcal{S}_j \to \RR$, and $\mathcal{S}_j \subseteq \RR$ is a \emph{nonempty} set, for all $j \in \{1,\dots,k\}$. We also define the vector function $\mathbf{u} \colon \mathcal{S} \to \RR^k$ by 
\begin{equation} \label{eq:Vector_function_u}
	\mathbf{u}(\mathbf{x}) \defeq \transp{[u_1(x_1),\dots,u_k(x_k)]} \in \RR^k ,
\end{equation}
for all $\mathbf{x} \in \mathcal{S}$, where the set $\mathcal{S} \defeq \prod_{j=1}^k {\mathcal{S}_j} \subseteq \RR^k$ is \emph{nonempty}. 

Observe that $\widehat{f}(n;\mathbf{x}) > 0$ for all $n \in \NN$ and $\mathbf{x} \in \mathcal{S}$, which implies that $f(n) > 0$ for sufficiently large $n$, since $f(n) \sim \widehat{f}(n;\boldsymbol{\alpha})$ as $n \to \infty$.\footnote{As will be shown in Section~\ref{sec:Fundamental_applications_in_analytic_combinatorics}, \eqref{eq:Proven_asymptotic_form} covers a \textit{broad class of asymptotics} (especially in analytic combinatorics), despite its seemingly restrictive nature.}  

The asymptotic form \eqref{eq:Proven_asymptotic_form} evaluated at $\mathbf{x} = \boldsymbol{\alpha}$ can be expressed as 
\begin{equation} \label{eq:Proven_asymptotic_form_with_gamma}
	\widehat{f}(n;\boldsymbol{\alpha}) = \exp\left( \sum_{j=1}^k \varphi_j(n) u_j(\alpha_j) \right) = \exp\left( \sum_{j=1}^k \varphi_j(n) \gamma_j \right)  ,
\end{equation}
where $\gamma_j \defeq u_j(\alpha_j) \in \RR$, for all $j \in \{1,\dots,k\}$, or equivalently 
\begin{equation*}
	\boldsymbol{\gamma} = \transp{[\gamma_1,\dots,\gamma_k]} \defeq \mathbf{u}(\boldsymbol{\alpha}) = \transp{[u_1(\alpha_1),\dots,u_k(\alpha_k)]} \in \RR^k .
\end{equation*}

Throughout this section, we will need the following \emph{standing assumptions}. Note that these hypotheses are quite general, since they hold very often (especially in analytic combinatorics).

\begin{assumption} \label{assum:Asymptotic_expansion_and_form}
	The function $f \colon \NN \to \RR_{\geq 0}$ admits the asymptotic expansion~\eqref{eq:Proven_asymptotic_expansion}, where its asymptotic form $\widehat{f} \colon {\NN \times \mathcal{S}} \to \RR_{\geq 0}$ is given by \eqref{eq:Proven_asymptotic_form}.  
\end{assumption}

\begin{assumption} \label{assum:Functions_varphi_j}
	The functions $\{\varphi_j \colon \NN \to \RR\}_{j=1}^k$ constitute an \emph{increasing asymptotic scale}, that is, $\varphi_j(n) = o(\varphi_{j+1}(n))$ for all $j \in \{1,\dots,k-1\}$, and $\varphi_1(n) = \Omega(1)$ as $n \to \infty$. 
\end{assumption}

\begin{assumption} \label{assum:Functions_u_j}
	Every function $u_j \colon \mathcal{S}_j \to \RR$, $j \in \{1,\dots,k\}$ and $\mathcal{S}_j \subseteq \RR$, is \emph{bijective} (thus \emph{invertible}) and also \emph{differentiable} (hence \emph{continuous}) with \emph{nonzero derivative} on its domain, i.e., $\frac{\mathrm{d} u_j}{\mathrm{d} x} (x') \neq 0$ for all $x' \in \mathcal{S}_j$.   
\end{assumption}

\begin{remark} \label{rem:Functions_u_j}
	Under Assumption~\ref{assum:Functions_u_j}, the inverse function of $u_j$ exists (since $u_j$ is bijective) and is denoted by $u_j^{-1} \colon \RR \to \mathcal{S}_j$. The inverse function of $\mathbf{u} \colon \mathcal{S} \to \RR^k$, where $\mathcal{S} \defeq \prod_{j=1}^k {\mathcal{S}_j} \subseteq \RR^k$, is denoted by $\mathbf{u}^{-1} \colon \RR^k \to \mathcal{S}$ and is given by
	\begin{equation*}
		\mathbf{u}^{-1}(\mathbf{y}) \defeq \transp{\left[ u_1^{-1}(y_1),\dots,u_k^{-1}(y_k) \right]}  \in \mathcal{S} ,
	\end{equation*}
	for all $\mathbf{y} \in \RR^k$. In addition, the differentiation rule of inverse function yields 
	\begin{equation} \label{eq:Derivative_of_inverse_u_j}
		\frac{\mathrm{d} u_j^{-1}}{\mathrm{d} y} (y') = \frac{1}{\frac{\mathrm{d} u_j}{\mathrm{d} x} \left( u_j^{-1}(y') \right)} ,
	\end{equation}
	for all $y' \in \RR$; note that $y = u_j(x)$ if and only if $x = u_j^{-1}(y)$. As a result, the inverse function $u_j^{-1}$ is \emph{differentiable} and thus \emph{continuous} on $\RR$, for all $j \in \{1,\dots,k\}$. The continuity of $u_j^{-1}$ is very important for \emph{numerical stability} of the inverse variable transformation $\mathbf{x} = \mathbf{u}^{-1}(\mathbf{y})$, which will be used later to compute the sequence of optimal solutions $\{\mathbf{x}_n^*\}_{n \in \NN}$.  
\end{remark}

\begin{remark}
	Assumption~\ref{assum:Functions_u_j} is satisfied when, for example, $u_j(x) \defeq x$ with $\mathcal{S}_j = \RR$, or $u_j(x) \defeq \log(x)$ with $\mathcal{S}_j = \RR_{>0}$, for all $j \in \{1,\dots,k\}$.   
\end{remark}

\subsection{Uniqueness of the Unknown Parameters}

Given the existence of $\boldsymbol{\alpha} \in \mathcal{S}$ and thus of $\boldsymbol{\gamma} \in \RR^k$, we can now prove their uniqueness. 

\begin{proposition}[Uniqueness of $\boldsymbol{\alpha}$ and $\boldsymbol{\gamma}$] 
	\label{prop:Uniqueness_of_alpha_and_gamma}
	Suppose that Assumptions~\labelcref{assum:Asymptotic_expansion_and_form,assum:Functions_varphi_j,assum:Functions_u_j} are satisfied. Then, the vectors $\boldsymbol{\alpha} = \transp{[\alpha_1,\dots,\alpha_k]} \in \mathcal{S}$ and $\boldsymbol{\gamma} = \transp{[\gamma_1,\dots,\gamma_k]} \in \RR^k$ are \emph{unique}.  
\end{proposition}

\begin{proof}
	Firstly, we will show the uniqueness of $\boldsymbol{\gamma} \in \RR^k$ by contradiction. Let us suppose that there are two such vectors $\boldsymbol{\gamma}, \boldsymbol{\gamma}' \in \RR^k$ with $\boldsymbol{\gamma} \neq \boldsymbol{\gamma}'$, i.e., there is at least one index $j \in \{1,\dots,k\}$ such that $\gamma_j \neq \gamma'_j$. Therefore, we can define the index  
	\begin{equation*}
		k^* \defeq \max\left\{ j \in \{1,\dots,k\} : \gamma_j \neq \gamma'_j \right\}  ,
	\end{equation*}
	and so $\gamma_j = \gamma'_j$ for all $j \in \{k^*+1,\dots,k\}$. By Assumption~\ref{assum:Asymptotic_expansion_and_form} and \eqref{eq:Proven_asymptotic_form_with_gamma}, we have $f(n) \sim \exp\left( \sum_{j=1}^k \varphi_j(n) \gamma_j \right)$ and $f(n) \sim \exp\left( \sum_{j=1}^k \varphi_j(n) \gamma'_j \right)$, hence 
	\begin{equation*}
		\exp\left( \sum_{j=1}^k \varphi_j(n) \gamma_j \right) \sim \exp\left( \sum_{j=1}^k \varphi_j(n) \gamma'_j \right)  \iff  \exp\left( \sum_{j=1}^k  \varphi_j(n) (\gamma_j - \gamma'_j) \right) \sim 1 
	\end{equation*}
	as $n \to \infty$. Since the logarithm function is continuous, we obtain 
	\begin{align} 
		\sum_{j=1}^k  \varphi_j(n) (\gamma_j - \gamma'_j) = o(1)  \implies  \sum_{j=1}^{k^*}  \varphi_j(n) (\gamma_j - \gamma'_j) = o(1)    \nonumber  \\ 
		\overset{\gamma_{k^*} \neq \gamma'_{k^*}}{\implies} \varphi_{k^*}(n) = \frac{1}{\gamma_{k^*} - \gamma'_{k^*}} \sum_{j=1}^{k^*-1}  \varphi_j(n) (\gamma'_j - \gamma_j) + o(1)  .  \label{eq:Varphi_k_star}
	\end{align}
	In the first implication we have used the fact that $\sum_{j=1}^k = \sum_{j=1}^{k^*} + \sum_{j=k^*+1}^k$ and the definition of $k^*$. Now, we consider two cases for the integer $k^* \geq 1$: a) $k^* = 1$ and b) $k^* \geq 2$. In the former case, \eqref{eq:Varphi_k_star} yields $\varphi_1(n) = o(1)$, which contradicts the fact that $\varphi_1(n) = \Omega(1)$ according to Assumption~\ref{assum:Functions_varphi_j}. In the latter case, \eqref{eq:Varphi_k_star} and Assumption~\ref{assum:Functions_varphi_j} give
	\begin{equation*}
		\varphi_{k^*}(n) = O(\varphi_{k^*-1}(n)) + O(1) = O(\varphi_{k^*-1}(n)) + O(\varphi_1(n)) = O(\varphi_{k^*-1}(n)) ,
	\end{equation*}
	since $o(1) = O(1)$ and $\varphi_1(n) = \Omega(1)$ is equivalent to $1 = O(\varphi_1(n))$. The estimate $\varphi_{k^*}(n) = O(\varphi_{k^*-1}(n))$ leads to a contradiction, because $\varphi_{k^*-1}(n) = o(\varphi_{k^*}(n))$ by Assumption~\ref{assum:Functions_varphi_j}.
	
	Secondly, the uniqueness of $\boldsymbol{\alpha} = \mathbf{u}^{-1}(\boldsymbol{\gamma}) \in \mathcal{S}$ follows from the uniqueness of $\boldsymbol{\gamma} \in \RR^k$ and the fact that $\mathbf{u}^{-1} \colon \RR^k \to \mathcal{S}$ exists due to Assumption~\ref{assum:Functions_u_j}; see also Remark~\ref{rem:Functions_u_j}. 
\end{proof}

\subsection{Asymptotic Invariants under Variable Transformations}

In ALT, as described in Section~\ref{sec:Asymptotics_Learning_Theory}, we construct a sequence of optimal solutions $\{\mathbf{x}_n^*\}_{n \in \NN}$ that ideally converges to the vector $\boldsymbol{\alpha}$. For the asymptotic  form~\eqref{eq:Proven_asymptotic_form}, it is more convenient to transform the original optimization problem~\eqref{eq:General_optimization_problem} into an equivalent problem (which is easier to solve) by applying a \emph{variable transformation}, that is, 
\begin{equation*}
	\mathbf{y} = \mathbf{u}(\mathbf{x}) \iff  y_j = u_j(x_j) , \;\;  \forall j \in \{1,\dots,k\} .  
\end{equation*}
Note that the above transformation is \emph{bijective}, and so \emph{invertible}, because of Assumption~\ref{assum:Functions_u_j}. First we compute a sequence of optimal solutions $\{\mathbf{y}_n^*\}_{n \in \NN}$ to the equivalent problems, and then apply the \emph{inverse variable transformation}, $\mathbf{x} = \mathbf{u}^{-1}(\mathbf{y})$, to obtain the sequence $\{\mathbf{x}_n^*\}_{n \in \NN}$.\footnote{Further details about this approach will be given in Sections~\ref{subsec:sLLSQ}~and~\ref{subsec:sT-LLSQ}.}

Due to the continuity of the functions $\{u_j,u_j^{-1}\}_{j=1}^k$, based on Assumption~\ref{assum:Functions_u_j} and Remark~\ref{rem:Functions_u_j}, it is easy to observe that $\lim_{n \to \infty} \mathbf{y}_n^* = \boldsymbol{\gamma}$ if and only if $\lim_{n \to \infty} \mathbf{x}_n^* = \boldsymbol{\alpha}$. Nevertheless, we should ensure not only the convergence to the desired vector, but also the convergence rate. Specifically, we introduce and prove three \emph{asymptotic invariants}, i.e., properties that remain unchanged under the variable transformation $\mathbf{y} = \mathbf{u}(\mathbf{x})$ as $n \to \infty$.

\begin{proposition}[Asymptotic invariants] \label{prop:Asymptotic_invariants}
	If Assumptions~\labelcref{assum:Asymptotic_expansion_and_form,assum:Functions_varphi_j,assum:Functions_u_j} hold, then we have the following invariants, as $n \to \infty$ and for all $j \in \{1,\dots,k\}$:  
	\begin{enumerate}[label={\arabic*.}]	
		\item \emph{Convergence-rate invariant:} the variable transformation and its inverse preserve the convergence rate, whenever convergence is achieved, that is, 
		\begin{align}
			x_{j,n}^* - \alpha_j = o(1)  \implies  y_{j,n}^* - \gamma_j = \Theta(x_{j,n}^* - \alpha_j) = o(1)  ,  \label{eq:Convergence-rate_invariant_x_to_y}   \\
			y_{j,n}^* - \gamma_j = o(1)  \implies  x_{j,n}^* - \alpha_j = \Theta(y_{j,n}^* - \gamma_j) = o(1)  . \label{eq:Convergence-rate_invariant_y_to_x}  
		\end{align}		
		As a result, it holds that   
		\begin{align} \label{eq:Convergence-rate_invariant_y_to_x_with_h}
			\begin{split}
				& { y_{j,n}^* - \gamma_j = O(h_j(n)) \;\; \mathrm{and} \;\; h_j(n) = o(1) }   \\ 
				\implies & { x_{j,n}^* - \alpha_j = \Theta(y_{j,n}^* - \gamma_j) = O(h_j(n)) = o(1) } . 
			\end{split}
		\end{align}
		
		\item \emph{Divergence invariant:} If $y_{j,n}^*$ does not converge to $\gamma_j$, then $x_{j,n}^*$ does not converge to $\alpha_j$, and vice versa. Symbolically, 		
		\begin{equation} \label{eq:Divergence_invariant}
			y_{j,n}^* - \gamma_j \neq o(1) \iff x_{j,n}^* - \alpha_j \neq o(1) .
		\end{equation}
		
		\item \emph{Boundedness invariant:} optimal solutions $y_{j,n}^*$ with bounded distance from $\gamma_j$ are mapped to optimal solutions $x_{j,n}^*$ with bounded distance from $\alpha_j$, i.e., 
		\begin{equation} \label{eq:Boundedness_invariant}
			y_{j,n}^* - \gamma_j = O(1) \implies x_{j,n}^* - \alpha_j = O(1) .
		\end{equation}
	\end{enumerate}
\end{proposition}

\begin{proof}
	First of all, the vectors $\boldsymbol{\alpha}$ and $\boldsymbol{\gamma}$ are unique by Proposition~\ref{prop:Uniqueness_of_alpha_and_gamma}. Now, based on Assumption~\ref{assum:Functions_u_j}, if $x_{j,n}^* - \alpha_j = o(1)$ and $x_{j,n}^* \neq \alpha_j$ for sufficiently large $n$, then 
	\begin{equation*}
		\lim_{n \to \infty} \frac{y_{j,n}^* - \gamma_j}{x_{j,n}^* - \alpha_j} = \lim_{n \to \infty} \frac{u_j(x_{j,n}^*) - u_j(\alpha_j)}{x_{j,n}^* - \alpha_j} = \lim_{x \to \alpha_j} \frac{u_j(x) - u_j(\alpha_j)}{x - \alpha_j} = \frac{\mathrm{d} u_j}{\mathrm{d} x} (\alpha_j) \neq 0   
	\end{equation*}
	by the definition of derivative. Since the absolute-value function is continuous, we also have $\lim_{n \to \infty} \abs{ \frac{y_{j,n}^* - \gamma_j}{x_{j,n}^* - \alpha_j} } = \abs{ \frac{\mathrm{d} u_j}{\mathrm{d} x} (\alpha_j) } > 0$. By the $(\varepsilon,N)$-definition of limit, we obtain $y_{j,n}^* - \gamma_j = \Theta(x_{j,n}^* - \alpha_j) = o(1)$ as $n \to \infty$. Observe that the $\Theta$-estimate holds even if $x_{j,n}^* = \alpha_j$ for sufficiently large $n$, because this implies that $y_{j,n}^* = u_j(x_{j,n}^*) = u_j(\alpha_j) = \gamma_j$ for sufficiently large $n$. Therefore, \eqref{eq:Convergence-rate_invariant_x_to_y} has been proved.
	
	Conversely, given Assumption~\ref{assum:Functions_u_j}, if $y_{j,n}^* - \gamma_j = o(1)$ and $y_{j,n}^* \neq \gamma_j$ for large enough $n$, then 
	\begin{align*}
		\lim_{n \to \infty} \frac{x_{j,n}^* - \alpha_j}{y_{j,n}^* - \gamma_j} & = \lim_{n \to \infty} \frac{u_j^{-1}(y_{j,n}^*) - u_j^{-1}(\gamma_j)}{y_{j,n}^* - \gamma_j} = \lim_{y \to \gamma_j} \frac{u_j^{-1}(y) - u_j^{-1}(\gamma_j)}{y - \gamma_j}   \\ 
		& = \frac{\mathrm{d} u_j^{-1}}{\mathrm{d} y} (\gamma_j) = \left( \frac{\mathrm{d} u_j}{\mathrm{d} x} (\alpha_j) \right)^{-1} \neq 0  
	\end{align*}
	by the definition of derivative and the differentiation rule of inverse function~\eqref{eq:Derivative_of_inverse_u_j}. Due to the continuity of absolute-value function, we also have $\lim_{n \to \infty} \abs{ \frac{x_{j,n}^* - \alpha_j}{y_{j,n}^* - \gamma_j} } = \abs{ \frac{\mathrm{d} u_j}{\mathrm{d} x} (\alpha_j) }^{-1} > 0$. By the $(\varepsilon,N)$-definition of limit, we get $x_{j,n}^* - \alpha_j = \Theta(y_{j,n}^* - \gamma_j) = o(1)$ as $n \to \infty$. The $\Theta$-estimate holds even if $y_{j,n}^* = \gamma_j$ for large enough $n$, since this implies that $x_{j,n}^* = u_j^{-1}(y_{j,n}^*) = u_j^{-1}(\gamma_j) = \alpha_j$ for large enough $n$. Hence, we have derived \eqref{eq:Convergence-rate_invariant_y_to_x}. Moreover, \eqref{eq:Convergence-rate_invariant_y_to_x_with_h} is easily obtained from~\eqref{eq:Convergence-rate_invariant_y_to_x}. 
	
	The divergence invariant~\eqref{eq:Divergence_invariant} is a direct consequence of the convergence-rate invariant. In particular, we have 
	\begin{equation*}
		y_{j,n}^* - \gamma_j \neq o(1)  \overset{\eqref{eq:Convergence-rate_invariant_x_to_y}}{\implies}  x_{j,n}^* - \alpha_j \neq o(1)  
	\end{equation*}
	and
	\begin{equation*}
		x_{j,n}^* - \alpha_j \neq o(1)  \overset{\eqref{eq:Convergence-rate_invariant_y_to_x}}{\implies}  y_{j,n}^* - \gamma_j \neq o(1)   .
	\end{equation*} 
	
	Finally, regarding the boundedness invariant, $y_{j,n}^* - \gamma_j = O(1)$ as $n \to \infty$ means that there exist a real constant $c > 0$ and an integer $n_0 \geq 1$ such that $\abs{y_{j,n}^* - \gamma_j} \leq c$, or equivalently $y_{j,n}^* \in [\gamma_j - c , \gamma_j + c] \subset \RR$, for all integers $n \geq n_0$. By the \emph{Weierstrass boundedness theorem}\footnote{Let $f$ be a real-valued continuous function on a closed (and bounded) interval $[a,b] \subset \RR$, with $-\infty < a \leq b < \infty$. Then, $f$ is bounded on this interval, i.e., there exists an $M \in \RR_{\geq 0}$ such that $\abs{f(x)} \leq M$, for all $x \in [a,b]$.} and the continuity of $u_j^{-1}$ on $\RR$ (see Remark~\ref{rem:Functions_u_j}), there exists a real constant $c' > 0$ such that $\abs{x_{j,n}^* - \alpha_j} = \abs{u_j^{-1}(y_{j,n}^*) - \alpha_j} \leq c'$ for all integers $n \geq n_0$. In other words, $x_{j,n}^* - \alpha_j = O(1)$ as $n \to \infty$, thus establishing \eqref{eq:Boundedness_invariant}. 
\end{proof}

\begin{remark}[Based on the proof of boundedness invariant \eqref{eq:Boundedness_invariant}]
	Despite the continuity of the function $u_j$ on its domain $\mathcal{S}_j$, the converse statement of \eqref{eq:Boundedness_invariant}, that is,
	\begin{equation} \label{eq:Boundedness_invariant_converse}
		x_{j,n}^* - \alpha_j = O(1) \implies y_{j,n}^* - \gamma_j = O(1) ,
	\end{equation}
	is \emph{not} necessarily true. This is due to the fact that the intersection $[\alpha_j - c, \alpha_j + c] \cap \mathcal{S}_j$ may not be closed (since $\mathcal{S}_j$ is \emph{not} necessarily closed), and hence the \emph{Weierstrass boundedness theorem} is generally not applicable. Nevertheless, if $\mathcal{S}_j$ is \emph{closed}, then \eqref{eq:Boundedness_invariant_converse} is true as well; recall that the intersection of closed sets is closed, so the Weierstrass boundedness theorem can be applied.
\end{remark}

\subsection{Sliding Linear Least Squares (sLLSQ)}
\label{subsec:sLLSQ}

Given a fixed $m \in \NN$ (the length of the sliding window), let us choose the loss function to be the logarithmic relative distance between $\widehat{\mathbf{f}}(n;\mathbf{x})$ and $\mathbf{f}(n)$, that is,  
\begin{align} \label{eq:Loss_function_F_n}
	\begin{split}
		F_n(\mathbf{x}) \defeq & \sum_{i=0}^{m-1} \left( \log\left( \frac{\widehat{f}(n+i;\mathbf{x})}{f(n+i)} \right) \right)^2  \\ 
		= & \sum_{i=0}^{m-1} \left( \log\left( \widehat{f}(n+i;\mathbf{x}) \right) - \log\left( f(n+i) \right) \right)^2  \\
		= & \sum_{i=0}^{m-1} \left( \sum_{j=1}^k {\varphi_j(n+i) u_j(x_j)} - \log\left( f(n+i) \right) \right)^2  ,
	\end{split}
\end{align}
where the last expression is due to \eqref{eq:Proven_asymptotic_form}. Observe that $F_n$ is well-defined as $n \to \infty$, because $f(n), \widehat{f}(n;\cdot) > 0$ for sufficiently large $n$.\footnote{An alternative choice of the loss function would be: $F_n(\mathbf{x}) \defeq \sum_{i=0}^{m-1} \left(  \frac{\widehat{f}(n+i;\mathbf{x})}{f(n+i)} - 1 \right)^2$. However, the latter is not suitable for the particular form of $\widehat{f}$ in \eqref{eq:Proven_asymptotic_form}.} 

In addition, the function $F_n$ can be written in matrix form as 
\begin{equation*}
	F_n(\mathbf{x}) = \twonorm{\mathbf{A} \mathbf{u}(\mathbf{x}) - \mathbf{b}}^2  ,
\end{equation*}
where 
\begin{equation} \label{eq:Matrix_A}
	\mathbf{A} = \mathbf{A}(n) \defeq [\boldsymbol{\varphi}_1,\dots,\boldsymbol{\varphi}_k]  =
	\begin{bmatrix}
		\varphi_1(n) & \cdots & \varphi_k(n) \\ 
		\vdots & \ddots & \vdots \\
		\varphi_1(n+m-1) & \cdots & \varphi_k(n+m-1)
	\end{bmatrix}
	\in \RR^{m \times k} ,
\end{equation}
with $\boldsymbol{\varphi}_j = \boldsymbol{\varphi}_j(n) \defeq \transp{[\varphi_j(n),\dots,\varphi_j(n+m-1)]} \in \RR^m$, for all $j \in \{1,\dots,k\}$, the vector function $\mathbf{u}$ is given by~\eqref{eq:Vector_function_u}, and  
\begin{equation} \label{eq:Vector_b}
	\mathbf{b} = \mathbf{b}(n) \defeq \transp{[\log(f(n)),\dots,\log(f(n+m-1))]}  \in \RR^m  .
\end{equation}
Based on \eqref{eq:General_optimization_problem}, the original (constrained) optimization problem is now defined by  
\begin{align} \label{eq:Original_optimization_problem}
	\begin{split}
		& \minimize_{\mathbf{x}} \;\; F_n(\mathbf{x}) = \twonorm{\mathbf{A} \mathbf{u}(\mathbf{x}) - \mathbf{b}}^2   \\
		& \subjectto \;\; \mathbf{x} \in \mathcal{S}  .  
	\end{split}
\end{align}
Note that the constraint $\mathbf{x} \in \mathcal{S}$ can be decomposed into separable constraints, i.e., $x_j \in {\mathcal{S}_j}$, for all $j \in \{1,\dots,k\}$, because $\mathcal{S} \defeq {\prod_{j=1}^k {\mathcal{S}_j}}$. A vector $\mathbf{x}_n^* \in \mathcal{S}$ is called a \emph{globally optimal solution} to (or a \emph{global minimizer} of) problem~\eqref{eq:Original_optimization_problem} if, and only if, $F_n(\mathbf{x}_n^*) \leq F_n(\mathbf{x})$ for all $\mathbf{x} \in \mathcal{S}$. When an optimal solution $\mathbf{x}_n^*$ exists, we denote the achieved objective value by $F_n^*$, i.e., $F_n^* \defeq F_n(\mathbf{x}_n^*)$, which is called the \emph{global optimum/minimum} of problem~\eqref{eq:Original_optimization_problem}.    

By applying the \emph{bijective variable transformation} $\mathbf{y} = \mathbf{u}(\mathbf{x}) \iff \mathbf{x} = \mathbf{u}^{-1}(\mathbf{y})$, problem~\eqref{eq:Original_optimization_problem} can be transformed into an unconstrained optimization problem, that is, 
\begin{equation} \label{eq:sLLSQ_problem} 
	\minimize_{\mathbf{y} \in \RR^k} \; F_n(\mathbf{u}^{-1}(\mathbf{y})) = \twonorm{\mathbf{A} \mathbf{y} - \mathbf{b}}^2  . 
\end{equation}
The above problem is a classical \emph{Linear Least Squares (LLSQ)} problem as shown in \eqref{eq:LLSQ_problem}. The corresponding method is called \emph{``sliding'' LLSQ (sLLSQ)} because, for every $n \in \NN$, problem~\eqref{eq:sLLSQ_problem} depends on the values $\{f(n),\dots,f(n+m-1)\}$. The global minimum of problem~\eqref{eq:sLLSQ_problem} is equal to $F_n^*$, provided that an optimal solution exists. Strictly speaking, problems \eqref{eq:Original_optimization_problem} and \eqref{eq:sLLSQ_problem} are \emph{equivalent} in the following sense.

\begin{proposition}[Equivalence of problems~\eqref{eq:Original_optimization_problem}~and~\eqref{eq:sLLSQ_problem}]
	\label{prop:Equivalence_of_optimization_problems}
	Suppose that Assumption~\ref{assum:Functions_u_j} holds. If $\mathbf{x}_n^*$ is an optimal solution to problem~\eqref{eq:Original_optimization_problem}, then $\mathbf{y}_n^* \defeq \mathbf{u}(\mathbf{x}_n^*)$ is an optimal solution to problem~\eqref{eq:sLLSQ_problem}. Conversely, if $\mathbf{y}_n^*$ is an optimal solution to problem~\eqref{eq:sLLSQ_problem}, then $\mathbf{x}_n^* \defeq \mathbf{u}^{-1}(\mathbf{y}_n^*)$ is an optimal solution to problem~\eqref{eq:Original_optimization_problem}. 
\end{proposition}

\begin{proof}
	Direct consequence of the fact that the variable transformation $\mathbf{u}$ is a \emph{bijection} between $\mathcal{S}$ and $\RR^k$. See also Remark~\ref{rem:Functions_u_j}. 
\end{proof}

As a result, problems~\eqref{eq:Original_optimization_problem}~and~\eqref{eq:sLLSQ_problem} have \emph{equinumerous sets of optimal solutions}, i.e., with the same cardinality. Subsequently, we will need an \emph{asymptotic rank assumption} to ensure that problem~\eqref{eq:sLLSQ_problem} admits a \emph{unique} (globally) optimal solution, for large enough $n$; the same applies to problem~\eqref{eq:Original_optimization_problem}.

\begin{assumption}[Asymptotic rank assumption]
	\label{assum:Asymptotic_rank_assumption}
	The matrix $\mathbf{A} = \mathbf{A}(n) \in \RR^{m \times k}$ given by \eqref{eq:Matrix_A} has full column rank, that is, $\rank{\mathbf{A}} = k$, for sufficiently large $n$ (thus $m \geq k$). 
\end{assumption}

\begin{assumption}[Equivalent asymptotic rank assumption]
	\label{assum:Equivalent_asymptotic_rank_assumption}
	There exists a function $h \colon \NN \to \RR_{>0}$ such that $\determ{\transp{\mathbf{A}} \mathbf{A}} = \Omega(h(n))$, as $n \to \infty$. 
\end{assumption}

\begin{proposition}
	\label{prop:Equivalence_of_asymptotic_rank_assumption}
	Assumption~\ref{assum:Asymptotic_rank_assumption} is equivalent to Assumption~\ref{assum:Equivalent_asymptotic_rank_assumption}. 
\end{proposition}

\begin{proof}
	From linear algebra, we know that the condition $\rank{\mathbf{A}} = k$ can be written as $\determ{\transp{\mathbf{A}} \mathbf{A}} \neq 0$, or equivalently as $\determ{\transp{\mathbf{A}} \mathbf{A}} > 0$ (because $\transp{\mathbf{A}} \mathbf{A}$ is symmetric and positive semi-definite and so $\determ{\transp{\mathbf{A}} \mathbf{A}} \geq 0$).
	
	If Assumption~\ref{assum:Asymptotic_rank_assumption} is true, then there exists an $n_0 \in \NN$ such that $\determ{\transp{\mathbf{A}} \mathbf{A}} > 0$, for all integers $n \geq n_0$. We can choose $h(n) \defeq \determ{\transp{\mathbf{A}} \mathbf{A}}$ for $n \geq n_0$, and $h(n) \defeq 1$ for $n \in \{1,\dots,n_0-1\}$. Therefore, Assumption~\ref{assum:Equivalent_asymptotic_rank_assumption} is true as well. 
	
	The converse statement follows directly from the definition of $\Omega$-notation. 	
\end{proof}

\begin{remark}
	If Assumption~\ref{assum:Asymptotic_rank_assumption}/\ref{assum:Equivalent_asymptotic_rank_assumption} holds for some $m \in \NN$, then it also holds for all integers $m' > m$. This is because the columns of matrix $\mathbf{A}' = \mathbf{A}'(n) \in \RR^{m' \times k}$ remain linearly independent, for large enough $n$.   
\end{remark}

The main theorem about the sLLSQ method is given below, including asymptotic estimates, sufficient conditions for convergence, and guarantees on the rate of convergence (whenever this is achieved).  

\begin{theorem} [The sLLSQ method]
	\label{thm:The_sLLSQ_method}
	Suppose that Assumptions~\labelcref{assum:Asymptotic_expansion_and_form,assum:Functions_varphi_j,assum:Functions_u_j} and \ref{assum:Asymptotic_rank_assumption}/\ref{assum:Equivalent_asymptotic_rank_assumption} are satisfied. Then, we have the following statements: 
	\begin{enumerate}[label={\arabic*.}]
		\item Problem~\eqref{eq:sLLSQ_problem} has a \emph{unique} (globally) optimal solution for sufficiently large~$n$, that is,    
		\begin{equation} \label{eq:y_n_star_sLLSQ}
			\mathbf{y}_n^* = \inv{(\transp{\mathbf{A}} \mathbf{A})} \transp{\mathbf{A}} \mathbf{b} = \pseudoinv{\mathbf{A}} \mathbf{b}    ,
		\end{equation}
		where the matrix $\mathbf{A}$ and vector $\mathbf{b}$ are given by~\eqref{eq:Matrix_A}~and~\eqref{eq:Vector_b}, respectively, and $\pseudoinv{\mathbf{A}} \defeq \inv{(\transp{\mathbf{A}} \mathbf{A})} \transp{\mathbf{A}} \in \RR^{k \times m}$ is the pseudoinverse of $\mathbf{A}$ (defined for large enough $n$). In addition, problem~\eqref{eq:Original_optimization_problem} admits a \emph{unique} optimal solution, for sufficiently large $n$, which is given by
		\begin{equation} \label{eq:x_n_star}
			\mathbf{x}_n^* = \mathbf{u}^{-1}(\mathbf{y}_n^*)  \iff  x_{j,n}^* = u_j^{-1}(y_{j,n}^*)  , \;\;  \forall j \in \{1,\dots,k\} . 
		\end{equation} 
		
		\item The global minimum $F_n^*$ of problem~\eqref{eq:sLLSQ_problem} vanishes asymptotically with decay rate $O( \norm{\mathbf{g}_1}^2 )$, that is, 
		\begin{equation} \label{eq:F_n_star_asymptotic_estimate_sLLSQ}
			F_n^* = O( \norm{\mathbf{g}_1}^2 ) = o(1) ,  \quad \mathrm{as}\ n \to \infty .
		\end{equation} 		
		Moreover, the optimal solution $\mathbf{y}_n^*$ of problem~\eqref{eq:sLLSQ_problem} satisfies the asymptotic estimate, as $n \to \infty$, 
		\begin{equation} \label{eq:y_n_star_asymptotic_estimate_sLLSQ}
			\norm{\mathbf{y}_n^* - \boldsymbol{\gamma}} = O( \eta(n) )  , 
		\end{equation}
		where 
		\begin{equation} \label{eq:Function_eta}
			\eta(n) \defeq \max\left( \norm{\pseudoinv{\mathbf{A}} \log\left( 1 + \sum_{l=1}^p {\beta_l \mathbf{g}_l} \right)}, \norm{\pseudoinv{\mathbf{A}}}  \norm{\mathbf{g}_{p+1}}  \right)   
		\end{equation}
		and 
		\begin{equation} \label{eq:Vector_g_l}
			\mathbf{g}_l = \mathbf{g}_l(n) \defeq \transp{[g_l(n),\dots,g_l(n+m-1)]} \in \RR^m , \;\; \forall l \in \{1,\dots,p+1\}    .
		\end{equation}
		Specifically, 
		\begin{equation} \label{eq:y_jn_star_asymptotic_estimate_sLLSQ}
			y_{j,n}^* - \gamma_j = O( \vartheta_j(n) )   , 
		\end{equation}
		where 
		\begin{equation} \label{eq:Function_vartheta_j}
			\vartheta_j(n) \defeq \max\left( \abs{\pseudoinv{\mathbf{A}}(j,:) \log\left( 1 + \sum_{l=1}^p {\beta_l \mathbf{g}_l} \right)}, \norm{\pseudoinv{\mathbf{A}}(j,:)} \norm{\mathbf{g}_{p+1}}  \right)  ,
		\end{equation} 
		for all $j \in \{1,\dots,k\}$. As a result, if $\eta(n) = o(1)$ (resp.~$\vartheta_j(n) = o(1)$), then $\lim_{n \to \infty} \mathbf{y}_n^* = \boldsymbol{\gamma}$ (resp.~$\lim_{n \to \infty} y_{j,n}^* = \gamma_j$) and the convergence rate is given by~\eqref{eq:y_n_star_asymptotic_estimate_sLLSQ} (resp.~\eqref{eq:y_jn_star_asymptotic_estimate_sLLSQ}).  
		
		\item The optimal solution $\mathbf{x}_n^*$ of problem~\eqref{eq:Original_optimization_problem}, given by~\eqref{eq:x_n_star}, satisfies the following implications, as $n \to \infty$ and for all $j \in \{1,\dots,k\}$,  
		\begin{equation} \label{eq:Implication_o_sLLSQ}
				\vartheta_j(n) = o(1)  \implies  x_{j,n}^* - \alpha_j = \Theta( y_{j,n}^* - \gamma_j ) = O( \vartheta_j(n) ) = o(1)   ,
		\end{equation} 
		\begin{equation} \label{eq:Implication_O_sLLSQ}
			\vartheta_j(n) = O(1)  \implies  x_{j,n}^* - \alpha_j = O(1)  .
		\end{equation}    
	\end{enumerate}
\end{theorem}

\begin{proof}
	See Appendix~\ref{subsec:Proof_of_the_sLLSQ_theorem}. 	
\end{proof}

\begin{remark}
	The implication's hypothesis in~\eqref{eq:Implication_o_sLLSQ}, i.e., $\vartheta_j(n) = o(1)$, is a \emph{sufficient condition for convergence} of $x_{j,n}^*$ to $\alpha_j$. When this condition is satisfied, the $O$-estimate in~\eqref{eq:Implication_o_sLLSQ} is essentially a convergence-rate guarantee. 
\end{remark}

\begin{definition} \label{def:Asymptotic_regularity}
	We say that a function $g \colon \NN \to \RR$ is \emph{asymptotically regular} if, and only if, 
	\begin{equation*}
		g(n+i) = O(g(n))  \quad \mathrm{as}\ n \to \infty , \quad \forall i \in \NN_0 .
	\end{equation*}
\end{definition}

\begin{remark} \label{rem:Examples_of_asymptotically_regular_functions}
	For example, the function $g(n) \defeq n^{-\sigma} = o(1)$ or $g(n) \defeq (-1)^n \, n^{-\sigma} = o(1)$, where $\sigma \in \RR_{>0}$, is asymptotically regular. In particular, for every $i \in \NN_0$, we have $\lim_{n \to \infty} \frac{\abs{g(n+i)}}{\abs{g(n)}} = \lim_{n \to \infty} \left( 1 + \frac{i}{n} \right)^{-\sigma} = 1$ and, by the limit definition, $g(n+i) = \Theta(g(n)) = O(g(n))$ as $n \to \infty$. 
	
	The same holds for $g(n) \defeq \tau^{-n} = o(1)$ or $g(n) \defeq (-1)^n \, \tau^{-n} = o(1)$, where $\tau \in \RR$ with $\tau > 1$. Specifically, $\lim_{n \to \infty} \frac{\abs{g(n+i)}}{\abs{g(n)}} = \tau^{-i} > 0$, which again gives $g(n+i) = \Theta(g(n)) = O(g(n))$ as $n \to \infty$, for all $i \in \NN_0$. 
\end{remark}

\begin{lemma} \label{lem:Asymptotic_regularity}
	If the function $g_l \colon \NN \to \RR$, $l \in \{1,\dots,p+1\}$, is asymptotically regular, then   
	\begin{equation*} 
		\norm{\mathbf{g}_l} = \Theta(g_l(n)) , \quad \mathrm{as}\ n \to \infty  ,
	\end{equation*}
	where $\mathbf{g}_l$ is given by~\eqref{eq:Vector_g_l} with fixed $m \in \NN$. 
\end{lemma}

\begin{proof}
	By the norm equivalence (Lemma~\ref{lem:Norm_equivalence}) and the asymptotic regularity of $g_l$, we obtain 
	\begin{align*}
		\norm{\mathbf{g}_l} & = O\left( \maxnorm{\mathbf{g}_l} \right) = O\left( \max\limits_{0 \leq i \leq m-1} {\abs{g_l(n+i)}} \right)   \\
		& = O\left( \max\limits_{0 \leq i \leq m-1} O(g_l(n)) \right) = O(g_l(n))  . 
	\end{align*} 
	In addition, it holds that $\norm{\mathbf{g}_l} = \Omega\left( \maxnorm{\mathbf{g}_l} \right) = \Omega(g_l(n))$ due to the norm equivalence (Lemma~\ref{lem:Norm_equivalence}) and the definition of maximum norm.
\end{proof}

\begin{remark} \label{rem:Asymptotic_regularity}
	If the function $g_{p+1}$ is asymptotically regular, then the term $\norm{\mathbf{g}_{p+1}}$ can be replaced by $\abs{g_{p+1}(n)}$ in~\eqref{eq:Function_eta}~and~\eqref{eq:Function_vartheta_j}, without affecting the asymptotic estimates due to Lemma~\ref{lem:Asymptotic_regularity}. To be more precise, in this case we have, as $n \to \infty$, 
		\begin{align*}
			\eta(n) & = \Theta\left( \max\left( \norm{\pseudoinv{\mathbf{A}} \log\left( 1 + \sum_{l=1}^p {\beta_l \mathbf{g}_l} \right)}, \norm{\pseudoinv{\mathbf{A}}} \abs{g_{p+1}(n)} \right) \right)   ,   \\
			\vartheta_j(n) & = \Theta\left( \max\left( \abs{\pseudoinv{\mathbf{A}}(j,:) \log\left( 1 + \sum_{l=1}^p {\beta_l \mathbf{g}_l} \right)}, \norm{\pseudoinv{\mathbf{A}}(j,:)} \abs{g_{p+1}(n)} \right) \right)   .
		\end{align*} 
		Similar conclusion applies to the term $\norm{\mathbf{g}_1}$ in~\eqref{eq:F_n_star_asymptotic_estimate_sLLSQ}, i.e., replacement with $\abs{g_1(n)}$.  	
\end{remark}

If we have information about higher-order asymptotics (i.e., larger $p$), then it may be possible to: 1) derive sharper asymptotic bounds, and therefore weaker sufficient conditions for convergence, 2) provide faster convergence rates, 3) or even prove the convergence of sLLSQ, especially when this is not possible with knowledge of lower-order asymptotics. 

\begin{proposition}[Impact of the asymptotics order]
	\label{prop:Impact_of_the_asymptotics_order}
	Let us define the functions 
	\begin{equation*} 
		\eta(n;p) \defeq \max\left( \norm{\pseudoinv{\mathbf{A}} \log\left( 1 + \sum_{l=1}^p {\beta_l \mathbf{g}_l} \right)}, \norm{\pseudoinv{\mathbf{A}}} \norm{\mathbf{g}_{p+1}} \right) ,  
	\end{equation*}
	\begin{equation*}
		\vartheta_j(n;p) \defeq \max\left( \abs{\pseudoinv{\mathbf{A}}(j,:) \log\left( 1 + \sum_{l=1}^p {\beta_l \mathbf{g}_l} \right)}, \norm{\pseudoinv{\mathbf{A}}(j,:)} \norm{\mathbf{g}_{p+1}}  \right)   .
	\end{equation*}
	Based on~\eqref{eq:Function_eta}~and~\eqref{eq:Function_vartheta_j}, we have just included the asymptotics order $p \in \NN_0$ as an extra argument in the functions $\eta$ and $\vartheta_j$. Then we have, as $n \to \infty$, 
	\begin{equation*}
		\eta(n;p) = O(\eta(n;p-1)) ,
	\end{equation*}
	\begin{equation*}
		\vartheta_j(n;p) = O(\vartheta_j(n;p-1)) ,
	\end{equation*}
	for all $p \in \NN$ and $j \in \{1,\dots,k\}$.  
\end{proposition}

\begin{proof}
	Since all norms are equivalent (Lemma~\ref{lem:Norm_equivalence}) and $g_{p+1}(n) = o(g_p(n)) = O(g_p(n))$, we deduce that 
	\begin{equation*}
		\norm{\mathbf{g}_{p+1}} = O(\maxnorm{\mathbf{g}_{p+1}}) = O(\maxnorm{\mathbf{g}_p}) = O(\norm{\mathbf{g}_p})  , 
	\end{equation*}
	\begin{equation*}
		\log\left( 1 + \sum_{l=1}^p {\beta_l \mathbf{g}_l} \right) = \log\left( 1 + \sum_{l=1}^{p-1} {\beta_l \mathbf{g}_l} + \beta_p \mathbf{g}_p \right) = \log\left( 1 + \sum_{l=1}^{p-1} {\beta_l \mathbf{g}_l} \right) + \mathbf{r}  , 
	\end{equation*}
	where the vector $\mathbf{r} \in \RR^m$ is given by
	\begin{align*}
		\mathbf{r} & = \transp{\left[ \log\left( 1 + \frac{\beta_p g_p(n)}{1 + \sum_{l=1}^{p-1} {\beta_l g_l(n)}} \right) , \dots, \log\left( 1 + \frac{\beta_p g_p(n+m-1)}{1 + \sum_{l=1}^{p-1} {\beta_l g_l(n+m-1)}} \right) \right]}    \\
		& = \transp{\left[ \log\left( 1 + \frac{O(g_p(n))}{\Omega(1)} \right) , \dots, \log\left( 1 + \frac{O(g_p(n+m-1))}{\Omega(1)} \right) \right]}  \\
		& = \transp{\left[ \log\left( 1 + O(g_p(n)) \right) , \dots, \log\left( 1 + O(g_p(n+m-1)) \right) \right]}  \\
		& = \transp{\left[ O(g_p(n)), \dots, O(g_p(n+m-1)) \right]}  .    
	\end{align*}
	In the last equality, we have used the asymptotic expansion $\log(1+x) = O(x)$ as $x \to 0$. Hence, $\norm{\mathbf{r}} = O(\maxnorm{\mathbf{r}}) = O(\maxnorm{\mathbf{g}_p}) = O(\norm{\mathbf{g}_p})$ and, by the triangle inequality, we obtain  
	\begin{align*}
		\norm{\pseudoinv{\mathbf{A}} \log\left( 1 + \sum_{l=1}^p {\beta_l \mathbf{g}_l} \right)} & = \norm{\pseudoinv{\mathbf{A}} \log\left( 1 + \sum_{l=1}^{p-1} {\beta_l \mathbf{g}_l} \right) + \pseudoinv{\mathbf{A}} \mathbf{r}}   \\
		& \leq \norm{\pseudoinv{\mathbf{A}} \log\left( 1 + \sum_{l=1}^{p-1} {\beta_l \mathbf{g}_l} \right)} + \norm{\pseudoinv{\mathbf{A}} \mathbf{r}}   \\
		& = \norm{\pseudoinv{\mathbf{A}} \log\left( 1 + \sum_{l=1}^{p-1} {\beta_l \mathbf{g}_l} \right)} + O\left( \norm{\pseudoinv{\mathbf{A}}} \norm{\mathbf{r}} \right)  \\
		& = \norm{\pseudoinv{\mathbf{A}} \log\left( 1 + \sum_{l=1}^{p-1} {\beta_l \mathbf{g}_l} \right)} + O\left( \norm{\pseudoinv{\mathbf{A}}} \norm{\mathbf{g}_p} \right)   \\
		& = O\left( \max\left( \norm{\pseudoinv{\mathbf{A}} \log\left( 1 + \sum_{l=1}^{p-1} {\beta_l \mathbf{g}_l} \right)}, \norm{\pseudoinv{\mathbf{A}}} \norm{\mathbf{g}_p} \right) \right)   .
	\end{align*}
	By combining the above, we conclude that 
	\begin{align*}
		& \max\left( \norm{\pseudoinv{\mathbf{A}} \log\left( 1 + \sum_{l=1}^p {\beta_l \mathbf{g}_l} \right)}, \norm{\pseudoinv{\mathbf{A}}} \norm{\mathbf{g}_{p+1}} \right)   \\ = & \, O\left( \max\left( \norm{\pseudoinv{\mathbf{A}} \log\left( 1 + \sum_{l=1}^{p-1} {\beta_l \mathbf{g}_l} \right)}, \norm{\pseudoinv{\mathbf{A}}} \norm{\mathbf{g}_p} \right) \right)  .
	\end{align*}
	In other words, $\eta(n;p) = O(\eta(n;p-1))$. Recall that the matrix $\mathbf{A}$, and so $\pseudoinv{\mathbf{A}}$, is independent of $p$ and the functions $\{g_l\}_{l=1}^{p+1}$. The per-component estimate involving the function $\vartheta_j$ can be derived analogously, i.e., just replace $\pseudoinv{\mathbf{A}}$ with $\pseudoinv{\mathbf{A}}(j,:)$. 
\end{proof}

\begin{corollary}
	Recursively, we obtain 
	\begin{equation*}
		\eta(n;p) = O(\eta(n;0)) = O\left( \norm{\pseudoinv{\mathbf{A}}} \norm{\mathbf{g}_1} \right) ,
	\end{equation*}
	\begin{equation*}
		\vartheta_j(n;p) = O(\vartheta_j(n;0)) = O\left( \norm{\pseudoinv{\mathbf{A}}(j,:)} \norm{\mathbf{g}_1} \right) ,
	\end{equation*}
	for all $p \in \NN$ and $j \in \{1,\dots,k\}$ (note that $\sum_{l=1}^0 {\beta_l \mathbf{g}_l} = \mathbf{0}$ by convention).
\end{corollary}

Next, we study how the length of the sliding window, $m \in \NN$, affects the convergence of the sLLSQ method. In particular, given that sLLSQ converges to $\boldsymbol{\gamma}$ for some $m \in \NN$, does it remain convergent (to $\boldsymbol{\gamma}$) when $m$ increases to an $m' > m$? According to the following proposition, the answer is \emph{affirmative} provided that the function $g_1$ is asymptotically regular and $\norm{\pseudoinv{\mathbf{A}}} g_1(n) = o(1)$ as $n \to \infty$.   

\begin{proposition}[Impact of the sliding window length]
	\label{prop:Impact_of_the_sliding_window_length}
	Let $m,m' \in \NN$ be fixed, with $m' > m$. Let $\mathbf{A}_{\mathrm{new}} = \mathbf{A}_{\mathrm{new}}(n) \in \RR^{m' \times k}$ be analogously defined as $\mathbf{A} = \mathbf{A}(n) \in \RR^{m \times k}$ in~\eqref{eq:Matrix_A}. Suppose that Assumption~\ref{assum:Asymptotic_rank_assumption}/\ref{assum:Equivalent_asymptotic_rank_assumption} is satisfied. Then, we have 
	\begin{equation} \label{eq:O-estimate_on_A_new}
		\norm{\pseudoinv{\mathbf{A}_{\mathrm{new}}}} = O\left(  \norm{\pseudoinv{\mathbf{A}}} \right) ,  \quad \mathrm{as}\ n \to \infty .
	\end{equation}
	Suppose also that Assumptions~\labelcref{assum:Asymptotic_expansion_and_form,assum:Functions_varphi_j,assum:Functions_u_j} are satisfied, and the function $g_1$ is asymptotically regular (see Definition~\ref{def:Asymptotic_regularity}). Then, the following implication holds 
	\begin{equation} \label{eq:Implication_on_y_n_new_star}
		\norm{\pseudoinv{\mathbf{A}}} g_1(n) = o(1)  \implies  
		\begin{cases}
			\norm{\mathbf{y}_n^* - \boldsymbol{\gamma}} = O\left( \norm{\pseudoinv{\mathbf{A}}} g_1(n) \right) = o(1) \;\; \mathrm{and}  \\
			\norm{\mathbf{y}_{n,\mathrm{new}}^* - \boldsymbol{\gamma}} = O\left( \norm{\pseudoinv{\mathbf{A}_{\mathrm{new}}}} g_1(n) \right) = o(1)  ,
		\end{cases} 
	\end{equation} 
	as $n \to \infty$, where $\mathbf{y}_n^*$ and $\mathbf{y}_{n,\mathrm{new}}^*$ are the respective optimal solutions of sLLSQ. 
\end{proposition}

\begin{proof}
	See Appendix~\ref{subsec:Proof_of_proposition_for_sliding_window_length}.
\end{proof}

Nevertheless, when $\norm{\pseudoinv{\mathbf{A}}} g_1(n) \neq o(1)$ as $n \to \infty$, it is natural to ask whether in general the following estimate holds 
\begin{equation*}
	\eta_{\mathrm{new}}(n)  \overset{?}{=}  O( \eta(n) ) , \quad \mathrm{as}\ n \to \infty ,
\end{equation*} 
where 
\begin{equation*}
	\eta_{\mathrm{new}}(n) \defeq \max\left( \norm{\pseudoinv{\mathbf{A}_{\mathrm{new}}} \log\left( 1 + \sum_{l=1}^p {\beta_l \mathbf{g}_{l,\mathrm{new}}} \right)}, \norm{\pseudoinv{\mathbf{A}_{\mathrm{new}}}} \norm{\mathbf{g}_{p+1,\mathrm{new}}} \right)  ,
\end{equation*}
$\mathbf{g}_{l,\mathrm{new}} = \mathbf{g}_{l,\mathrm{new}}(n) \defeq \transp{[g_l(n),\dots,g_l(n+m'-1)]} \in \RR^{m'}$, for all $l \in \{1,\dots,p+1\}$, and the function $\eta$ is defined by~\eqref{eq:Function_eta}. Counterintuitively, the above $O$-estimate is \emph{not} always true. The next remark provides an interesting \emph{counterexample}.

\begin{remark}
Let $k = 2$, $p = 1$, $\beta_1 \neq 0$, $[\varphi_1(n),\varphi_2(n)] \defeq [1,n]$, and $[g_1(n),g_2(n)] \defeq \left[ \frac{(-1)^n}{n},\frac{1}{n^2} \right]$. Observe that the functions $g_1$ and $g_2$ are asymptotically regular, according to Remark~\ref{rem:Examples_of_asymptotically_regular_functions}, and by Lemma~\ref{lem:Asymptotic_regularity} we obtain 
\begin{align*}
	\eta(n) & = \Theta\left( \max\left( \norm{\pseudoinv{\mathbf{A}} \log\left( 1 + {\beta_1 \mathbf{g}_1} \right)}, \norm{\pseudoinv{\mathbf{A}}} \abs{g_2(n)} \right)  \right)  ,  \\
	\eta_{\mathrm{new}}(n) & = \Theta\left( \max\left( \norm{\pseudoinv{\mathbf{A}_{\mathrm{new}}} \log\left( 1 + {\beta_1 \mathbf{g}_{1,\mathrm{new}}} \right)}, \norm{\pseudoinv{\mathbf{A}_{\mathrm{new}}}} \abs{g_2(n)} \right)  \right)   .
\end{align*} 
We examine two cases, namely, 
\begin{itemize}
\item $m = 3$: $\determ{\transp{\mathbf{A}} \mathbf{A}} = 6 = \Omega(1)$ (therefore Assumption~\ref{assum:Asymptotic_rank_assumption}/\ref{assum:Equivalent_asymptotic_rank_assumption} holds), $\norm{\pseudoinv{\mathbf{A}}} = \Theta(n)$ (thus $\norm{\pseudoinv{\mathbf{A}}} g_1(n) = \Theta(1) \neq o(1)$), and $\pseudoinv{\mathbf{A}} \log\left( 1 + {\beta_1 \mathbf{g}_1} \right) = {(-1)^n} \frac{\beta_1}{n} \transp{\left[ \frac{4}{3}, -\frac{1}{n} \right]} + \mathbf{l.o.t.} = \transp{\left[ \Theta\left( \frac{1}{n} \right), \Theta\left( \frac{1}{n^2} \right) \right]}$ $\implies$ $\norm{\pseudoinv{\mathbf{A}} \log\left( 1 + {\beta_1 \mathbf{g}_1} \right)} = \Theta\left( \frac{1}{n} \right) = o(1)$. Based on~\eqref{eq:y_n_star_asymptotic_estimate_sLLSQ}, the convergence of sLLSQ is \emph{guaranteed}, i.e., $\lim_{n \to \infty} \mathbf{y}_n^* = \boldsymbol{\gamma}$, since 
\begin{equation*}
\eta(n) = \Theta\left( \max\left( \norm{\pseudoinv{\mathbf{A}} \log\left( 1 + {\beta_1 \mathbf{g}_1} \right)}, \norm{\pseudoinv{\mathbf{A}}} \abs{g_2(n)} \right) \right)  = \Theta\left( \frac{1}{n} \right) = o(1) .
\end{equation*}

\item $m' = 4$: $\determ{\transp{\mathbf{A}_{\mathrm{new}}} \mathbf{A}_{\mathrm{new}}} = 20 = \Omega(1)$ (thus Assumption~\ref{assum:Asymptotic_rank_assumption}/\ref{assum:Equivalent_asymptotic_rank_assumption} is satisfied), $\norm{\pseudoinv{\mathbf{A}_{\mathrm{new}}}} = \Theta(n)$, and $\pseudoinv{\mathbf{A}_{\mathrm{new}}} \log\left( 1 + {\beta_1 \mathbf{g}_{1,\mathrm{new}}} \right) = {(-1)^n} \frac{2\beta_1}{5} \transp{\left[ 1, -\frac{1}{n} \right]} + \mathbf{l.o.t.} = \transp{\left[ \Theta(1), \Theta\left( \frac{1}{n} \right) \right]}$ $\implies$ $\norm{\pseudoinv{\mathbf{A}_{\mathrm{new}}} \log\left( 1 + {\beta_1 \mathbf{g}_{1,\mathrm{new}}} \right)} = \Theta(1) \neq o(1)$. The convergence of sLLSQ is \emph{not} ensured now, despite the increase of $m$ to $m'$, because 
\begin{equation*}
\eta_{\mathrm{new}}(n) = \Theta\left( \max\left( \norm{\pseudoinv{\mathbf{A}_{\mathrm{new}}} \log\left( 1 + {\beta_1 \mathbf{g}_{1,\mathrm{new}}} \right)}, \norm{\pseudoinv{\mathbf{A}_{\mathrm{new}}}} \abs{g_2(n)} \right) \right)  = \Theta(1) \neq o(1) .
\end{equation*}
\end{itemize}
As a result, we deduce that 
\begin{equation*}
\Theta(1) = \eta_{\mathrm{new}}(n)  \neq  O( \eta(n) ) = O\left( \frac{1}{n} \right) = o(1) .
\end{equation*}
\end{remark}

\subsection{Sliding Tikhonov Linear Least Squares (sT-LLSQ)}
\label{subsec:sT-LLSQ}

In order to avoid infinite-norm limit points of the sequence $\{\mathbf{y}_n^*\}_{n \in \NN}$, we resort to a well-known technique in optimization theory that is called \emph{``regularization''}. In particular, by adding a \emph{penalty term} (involving some norm) to the objective function, we can ensure that the norm of optimal solution  $\mathbf{y}_n^*$ remains bounded as $n \to \infty$. To this end, we need an additional hypothesis and an auxiliary lemma, which is used as a stepping stone to the main theorem. 

\begin{assumption} \label{assum:sT-LLSQ_when_k_equals_1}
	If $k = 1$, then $\lim_{n \to \infty} \abs{\varphi_1(n)} = \infty$. 
\end{assumption}

\begin{remark}
	The condition $\lim_{n \to \infty} \abs{\varphi_1(n)} = \infty$ implies that $\varphi_1(n) = \Omega(1)$ as $n \to \infty$. Hence, Assumption~\ref{assum:sT-LLSQ_when_k_equals_1} does not contradict Assumption~\ref{assum:Functions_varphi_j}. 
\end{remark}

\begin{lemma}
	\label{lem:Regularization_with_arbitrary_norm}
	Suppose that Assumptions~\labelcref{assum:Functions_varphi_j,assum:sT-LLSQ_when_k_equals_1} are true. Let $\widetilde{\boldsymbol{\gamma}} \in \RR^k$ be independent of $n$, and $\mathbf{w} = \mathbf{w}(n) \in \RR^m$ be such that $\norm{\mathbf{w}} = O(1)$ as $n \to \infty$. Consider the regularized optimization problem  
	\begin{equation} \label{eq:Regularized_optimization_problem}
		\minimize_{\mathbf{y} \in \RR^k} \; R_n(\mathbf{y}) \defeq \twonorm{\mathbf{A} \mathbf{y} - \widetilde{\mathbf{b}}}^2 + \mu \norm{\mathbf{y}}^{\nu}   , 
	\end{equation}
	where $\mathbf{A} = \mathbf{A}(n) \in \RR^{m \times k}$ is given by~\eqref{eq:Matrix_A}, $\widetilde{\mathbf{b}} = \widetilde{\mathbf{b}}(n) \defeq \mathbf{A} \widetilde{\boldsymbol{\gamma}} + \mathbf{w} \in \RR^m$, $\mu \in \RR_{>0}$ is the regularization parameter (independent of $n$), $\norm{\cdot}$ is an arbitrary vector norm, and $\nu \geq 1$ is a fixed real exponent. Then, we have the following statements: 
	\begin{enumerate}[label={\arabic*.}] 
		\item For every $n \in \NN$, \eqref{eq:Regularized_optimization_problem} is a \emph{convex optimization problem} and there exists at least one (globally) optimal solution, that is, a $\mathbf{y}_n^* \in \RR^k$ such that $R_n(\mathbf{y}_n^*) \leq R_n(\mathbf{y})$ for all $\mathbf{y} \in \RR^k$. Let $R_n^* \defeq R_n(\mathbf{y}_n^*)$ be the global minimum of problem~\eqref{eq:Regularized_optimization_problem}. Furthermore, as $n \to \infty$, 
		\begin{equation*}  
			R_n^* = O(1), \quad \norm{\mathbf{y}_n^*} = O(1),  
		\end{equation*}
		and 
		\begin{equation*}
			\norm{\mathbf{y}_n^* - \widetilde{\boldsymbol{\gamma}}} = O(1) ,  \quad  y_{j,n}^* - \widetilde{\gamma}_j = O(1) ,
		\end{equation*}
		for every sequence of optimal solutions $\{\mathbf{y}_n^*\}_{n \in \NN}$ and for all $j \in \{1,\dots,k\}$. In particular, for $j=k$, we obtain the \textbf{universal convergence rate}\footnote{We call it \emph{``universal''} because this convergence rate holds for \emph{every} parameter $\mu \in \RR_{>0}$, \emph{every} vector norm $\norm{\cdot}$, and \emph{every} real exponent $\nu \geq 1$ in the regularization term $\mu \norm{\mathbf{y}}^{\nu}$.}   
		\begin{equation} \label{eq:Universal_convergence_rate_for_index_k}
				y_{k,n}^* - \widetilde{\gamma}_k = O\left( \frac{\varphi_{k-1}(n)}{\varphi_k(n)} \right) = o(1)  , \quad \mathrm{as}\ n \to \infty  , 
		\end{equation}
		with the convention $\varphi_0(n) \defeq 1$ when $k=1$.
		
		\item Assume that the vector norm in the regularization term $\mu \norm{\mathbf{y}}^{\nu}$ is the Euclidean norm $\twonorm{\cdot}$ and the exponent $\nu = 2$ (i.e., Tikhonov regularization). Then, problem~\eqref{eq:Regularized_optimization_problem} has a unique optimal solution given by 
		\begin{equation*} 
			\mathbf{y}_n^* = \inv{(\transp{\mathbf{A}} \mathbf{A} + \mu \mathbf{I})} \transp{\mathbf{A}} \widetilde{\mathbf{b}} = \mathbf{C} \widetilde{\mathbf{b}} ,
		\end{equation*}
		where 
		\begin{equation} \label{eq:Matrix_C}
			\mathbf{C} = \mathbf{C}(n) \defeq \inv{(\transp{\mathbf{A}} \mathbf{A} + \mu \mathbf{I})} \transp{\mathbf{A}} = \mathbf{D} \transp{\mathbf{A}} \in \RR^{k \times m} ,
		\end{equation}
		\begin{equation} \label{eq:Matrix_D}
			\mathbf{D} = \mathbf{D}(n) \defeq \inv{(\transp{\mathbf{A}} \mathbf{A} + \mu \mathbf{I})} \in \RR^{k \times k}  .
		\end{equation} 
		In addition, the above matrices admit the asymptotic estimates, as $n \to \infty$, 
		\begin{equation*}
			\norm{\mathbf{C}} = O(1) , \quad \norm{\mathbf{D}} = O(1) ,
		\end{equation*}
		and 
		\begin{equation*}
			\norm{\mathbf{C}(j,:)} = O(1) , \quad \norm{\mathbf{D}(j,:)} = O(1) ,
		\end{equation*}
		for all $j \in \{1,\dots,k\}$. Specifically, for $j=k$, we have
		\begin{equation*}
			\norm{\mathbf{C}(k,:)} = O\left( \frac{\varphi_{k-1}(n)}{\varphi_k(n)} \right) = o(1)   \quad \mathrm{and} \quad \norm{\mathbf{D}(k,:)} = O\left( \frac{\varphi_{k-1}(n)}{\varphi_k(n)} \right) = o(1)  ,
		\end{equation*}
		under the same convention regarding $\varphi_0$.
	\end{enumerate}
\end{lemma}

\begin{proof}
	See Appendix~\ref{subsec:Proof_of_the_sT-LLSQ_lemma}. 
\end{proof}

Next, the \emph{sliding Tikhonov LLSQ (sT-LLSQ)} problem is formulated as follows  
\begin{equation} \label{eq:sT-LLSQ_problem}
	\minimize_{\mathbf{y} \in \RR^k} \; G_n(\mathbf{y}) \defeq \twonorm{\mathbf{A} \mathbf{y} - \mathbf{b}}^2 + \mu \twonorm{\mathbf{y}}^2  , 
\end{equation}
where the matrix $\mathbf{A}$ and vector $\mathbf{b}$ are given by \eqref{eq:Matrix_A} and \eqref{eq:Vector_b}, respectively. Moreover, $\mu \in \RR_{>0}$ is the \emph{regularization parameter}, which is independent of $n$. Note that $\mu = 0$ corresponds to sLLSQ, without regularization. Based on Lemma~\ref{lem:Regularization_with_arbitrary_norm}, we can now establish the main result for sT-LLSQ.

\begin{theorem} [The sT-LLSQ method]
	\label{thm:The_sT-LLSQ_method}
	Let Assumptions~\labelcref{assum:Asymptotic_expansion_and_form,assum:Functions_varphi_j,assum:Functions_u_j} and \ref{assum:sT-LLSQ_when_k_equals_1} be satisfied. Then, we have the following statements: 
	\begin{enumerate}[label={\arabic*.}]
		\item For every $n \in \NN$, problem~\eqref{eq:sT-LLSQ_problem} has a \emph{unique} (globally) optimal solution, which is given by     
		\begin{equation} \label{eq:y_n_star_sT-LLSQ}
			\mathbf{y}_n^* = \inv{(\transp{\mathbf{A}} \mathbf{A} + \mu \mathbf{I})} \transp{\mathbf{A}} \mathbf{b} = \mathbf{D} \transp{\mathbf{A}} \mathbf{b} = \mathbf{C} \mathbf{b} ,
		\end{equation}
		where the matrices $\mathbf{C}$ and $\mathbf{D}$ are defined by \eqref{eq:Matrix_C} and \eqref{eq:Matrix_D}, respectively. 
		
		\item The global minimum $G_n^* \defeq G_n(\mathbf{y}_n^*)$ of problem~\eqref{eq:sT-LLSQ_problem} remains bounded asymptotically, that is, 
		\begin{equation} \label{eq:G_n_star_asymptotic_estimate_sT-LLSQ}
			G_n^* = O(1)  , \quad \mathrm{as}\ n \to \infty .
		\end{equation}  
		In addition, the optimal solution $\mathbf{y}_n^*$ of problem~\eqref{eq:sT-LLSQ_problem} satisfies the asymptotic estimate, as $n \to \infty$, 
		\begin{equation} \label{eq:y_n_star_asymptotic_estimate_sT-LLSQ}
			\norm{\mathbf{y}_n^* - \boldsymbol{\gamma}} = O(\chi(n)) = O(1)   , 
		\end{equation}
		where 
		\begin{equation} \label{eq:Function_chi}
			\chi(n) \defeq \max\left( \norm{\mathbf{C} \log\left( 1 + \sum_{l=1}^p {\beta_l \mathbf{g}_l} \right)}, \norm{\mathbf{C}} \norm{\mathbf{g}_{p+1}}  , \norm{\mathbf{D}}  \right) = O(1) 
		\end{equation}
		and the vectors $\{\mathbf{g}_l = \mathbf{g}_l(n)\}_{l=1}^{p+1}$ are defined by~\eqref{eq:Vector_g_l}.  
		In particular, 
		\begin{equation} \label{eq:y_jn_star_asymptotic_estimate_sT-LLSQ}
			y_{j,n}^* - \gamma_j = O(\psi_j(n)) = O(1)  ,     
		\end{equation}
		where 
		\begin{equation} \label{eq:Function_psi_j}
			\psi_j(n) \defeq \max\left( \abs{\mathbf{C}(j,:) \log\left( 1 + \sum_{l=1}^p {\beta_l \mathbf{g}_l} \right)}, \norm{\mathbf{C}(j,:)} \norm{\mathbf{g}_{p+1}} , \norm{\mathbf{D}(j,:)}  \right) = O(1) ,
		\end{equation} 
		for all $j \in \{1,\dots,k\}$. More specifically, for $j=k$, it holds that $\psi_k(n) = O\left( \frac{\varphi_{k-1}(n)}{\varphi_k(n)} \right) = o(1)$, thus yielding the convergence rate  
		\begin{equation} \label{eq:y_kn_star_asymptotic_estimate_sT-LLSQ}
			y_{k,n}^* - \gamma_k = O(\psi_k(n)) = O\left( \frac{\varphi_{k-1}(n)}{\varphi_k(n)} \right) = o(1)  , 
		\end{equation}
		with the convention $\varphi_0(n) \defeq 1$ when $k=1$. 
		
		\item Let $\mathbf{x}_n^*$ be defined by~\eqref{eq:x_n_star}, where $\mathbf{y}_n^*$ is given by~\eqref{eq:y_n_star_sT-LLSQ}. Then, as $n \to \infty$ and for all $j \in \{1,\dots,k\}$, we have    
		\begin{equation} \label{eq:x_jn_star_asymptotic_estimate_sT-LLSQ}
			x_{j,n}^* - \alpha_j = O(1)  ,
		\end{equation} 
		\begin{equation} \label{eq:Implication_o_sT-LLSQ}
			\psi_j(n) = o(1)  \implies  x_{j,n}^* - \alpha_j = \Theta( y_{j,n}^* - \gamma_j ) = O( \psi_j(n) ) = o(1)  ,
		\end{equation} 
		and, specifically for $j = k$, 
		\begin{equation} \label{eq:x_kn_star_asymptotic_estimate_sT-LLSQ}
			x_{k,n}^* - \alpha_k = \Theta( y_{k,n}^* - \gamma_k ) = O( \psi_k(n) ) = O\left( \frac{\varphi_{k-1}(n)}{\varphi_k(n)} \right) = o(1)   , 
		\end{equation}  
		under the same convention about $\varphi_0$. 
	\end{enumerate}
\end{theorem}

\begin{proof}
	See Appendix~\ref{subsec:Proof_of_the_sT-LLSQ_theorem}. 
\end{proof}

\begin{remark} \label{rem:No_improvement_in_convergence_rate}
	If $\norm{\mathbf{D}(k,:)} = \Omega\left( \frac{\varphi_{k-1}(n)}{\varphi_k(n)} \right)$,\footnote{This is equivalent to $\norm{\mathbf{D}(k,:)} = \Theta\left( \frac{\varphi_{k-1}(n)}{\varphi_k(n)} \right)$, since $\norm{\mathbf{D}(k,:)} = O\left( \frac{\varphi_{k-1}(n)}{\varphi_k(n)} \right)$ according to Lemma~\ref{lem:Regularization_with_arbitrary_norm}.} then $\psi_k(n) = \Theta\left( \frac{\varphi_{k-1}(n)}{\varphi_k(n)} \right)$. The proof is simple: we already know that $\psi_k(n) = O\left( \frac{\varphi_{k-1}(n)}{\varphi_k(n)} \right)$ and, based on~\eqref{eq:Function_psi_j}, we also have $\psi_k(n) \geq \norm{\mathbf{D}(k,:)} = \Omega\left( \frac{\varphi_{k-1}(n)}{\varphi_k(n)} \right)$. In this case, $\psi_k(n)$ does not provide any improvement in convergence rate compared to the ratio $\frac{\varphi_{k-1}(n)}{\varphi_k(n)}$, regardless of $p$. 
\end{remark}
 
\begin{remark}
	Nevertheless, $\psi_k(n)$ can sometimes give a \emph{faster} rate of convergence than the ratio $\frac{\varphi_{k-1}(n)}{\varphi_k(n)}$. For example, let us consider the following case: $m = k = 3$, $\mu = 1$, and $[\varphi_1(n), \varphi_2(n), \varphi_3(n)] = [n, n^2, n^3]$. As a result, we have $\frac{\varphi_{k-1}(n)}{\varphi_k(n)} = \frac{\varphi_2(n)}{\varphi_3(n)} = \frac{1}{n}$, $\norm{\mathbf{D}(k,:)} = \norm{\mathbf{D}(3,:)} = \Theta\left( \frac{1}{n^2} \right) = o\left( \frac{1}{n} \right) \neq \Omega\left( \frac{1}{n} \right)$, and $\norm{\mathbf{C}(k,:)} = \norm{\mathbf{C}(3,:)} = \Theta\left( \frac{1}{n^2} \right) = o\left( \frac{1}{n} \right)$. Therefore, we obtain 
	\begin{align*}
		\psi_3(n) & = O\left( \max\left( \norm{\mathbf{C}(3,:)} , \norm{\mathbf{D}(3,:)} \right) \right) = O\left( \frac{1}{n^2} \right)  = o\left( \frac{1}{n} \right) .
	\end{align*}
\end{remark}

\begin{remark} \label{rem:sT-LLSQ_for_single_parameter_and_unit_varphi_1}
	Suppose that $k=1$ and $\varphi_1(n) \defeq 1 = O(1)$, thus Assumption~\ref{assum:sT-LLSQ_when_k_equals_1} is false. Then, for every $m \in \NN$, we have $\mathbf{D} = (m + \mu)^{-1}$ and $\norm{\mathbf{C}(1,:)} = \Theta(1) = O(1)$. From~\eqref{eq:y_n_star_sT-LLSQ_rearranged}, we obtain
	\begin{align*}
		y_{1,n}^* & = \gamma_1 + \mathbf{C}(1,:) \log\left( 1 + \sum_{l=1}^p {\beta_l \mathbf{g}_l} \right) + \mathbf{C}(1,:) \mathbf{z} - \mu \mathbf{D}(1,:) \gamma_1  \\
		& = \frac{m}{m + \mu} \gamma_1 + \mathbf{C}(1,:) \left( \log\left( 1 + \sum_{l=1}^p {\beta_l \mathbf{g}_l} \right) + \mathbf{z} \right) ,
	\end{align*}
	thus giving  
	\begin{equation*}
		y_{1,n}^* - \frac{m}{m + \mu} \gamma_1 = O(\norm{\mathbf{C}(1,:)} \norm{\mathbf{g}_1}) = O(\norm{\mathbf{g}_1}) = o(1)   ,
	\end{equation*}  
	because $\norm{\log\left( 1 + \sum_{l=1}^p {\beta_l \mathbf{g}_l} \right) + \mathbf{z}} = O(\norm{\mathbf{g}_1})$. Alternatively, we can write  
	\begin{equation*}
		\widetilde{y}_{1,n}^* - \gamma_1 = O(\norm{\mathbf{g}_1}) = o(1) ,
	\end{equation*} 
	where $\widetilde{y}_{1,n}^* \defeq \frac{m + \mu}{m} y_{1,n}^*$. 
\end{remark}

Subsequently, we give a result that is analogous to Proposition~\ref{prop:Impact_of_the_asymptotics_order}. For the sake of brevity, the proof is omitted since it follows very similar steps.  

\begin{proposition}
Let us define the functions 
\begin{equation*} 
	\chi(n;p) \defeq \max\left( \norm{\mathbf{C} \log\left( 1 + \sum_{l=1}^p {\beta_l \mathbf{g}_l} \right)}, \norm{\mathbf{C}} \norm{\mathbf{g}_{p+1}}  , \norm{\mathbf{D}}  \right) ,  
\end{equation*}
\begin{equation*}
	\psi_j(n;p) \defeq \max\left( \abs{\mathbf{C}(j,:) \log\left( 1 + \sum_{l=1}^p {\beta_l \mathbf{g}_l} \right)}, \norm{\mathbf{C}(j,:)} \norm{\mathbf{g}_{p+1}} , \norm{\mathbf{D}(j,:)}  \right)   .
\end{equation*}
Based on~\eqref{eq:Function_chi}~and~\eqref{eq:Function_psi_j}, we have just included the asymptotics order $p \in \NN_0$ as an extra argument in the functions $\chi$ and $\psi_j$. Then we have, as $n \to \infty$,
\begin{equation*}
	\chi(n;p) = O(\chi(n;p-1)) ,
\end{equation*}
\begin{equation*}
	\psi_j(n;p) = O(\psi_j(n;p-1)) ,
\end{equation*}
for all $p \in \NN$ and $j \in \{1,\dots,k\}$.  
\end{proposition}

\begin{corollary}
	Recursively, we obtain 
	\begin{equation*}
		\chi(n;p) = O(\chi(n;0)) = O\left( \max\left( \norm{\mathbf{C}}  \norm{\mathbf{g}_1} , \norm{\mathbf{D}} \right) \right)  ,
	\end{equation*}
	\begin{equation*}
		\psi_j(n;p) = O(\psi_j(n;0)) = O\left( \max\left( \norm{\mathbf{C}(j,:)} \norm{\mathbf{g}_1} , \norm{\mathbf{D}(j,:)} \right) \right)  ,
	\end{equation*}
	for all $p \in \NN$ and $j \in \{1,\dots,k\}$ (note that $\sum_{l=1}^0 {\beta_l \mathbf{g}_l} = \mathbf{0}$ by convention).
\end{corollary}

\begin{remark}
The \emph{Least Absolute Shrinkage and Selection Operator (LASSO)} regularization method \cite{Santosa-Symes_1986,Tibshirani_1996} is similar to the Tikhonov regularization, but its regularization term is based on the $1$-norm instead of the $2$-norm. In particular, the \emph{sliding LASSO-LLSQ} problem (in Lagrangian form) is given by 
\begin{equation} \label{eq:sliding_LASSO-LLSQ_problem} 
	\minimize_{\mathbf{y} \in \RR^k} \; L_n(\mathbf{y}) \defeq \twonorm{\mathbf{A} \mathbf{y} - \mathbf{b}}^2 + \mu \onenorm{\mathbf{y}}  . 
\end{equation}
Based on the uniqueness of $\boldsymbol{\gamma}$ by Proposition~\ref{prop:Uniqueness_of_alpha_and_gamma}, we first choose $\widetilde{\boldsymbol{\gamma}} \defeq \boldsymbol{\gamma}$ and $\mathbf{w} \defeq \mathbf{b} - \mathbf{A} \boldsymbol{\gamma}$ (satisfying $\norm{\mathbf{w}} = o(1) = O(1)$ as $n \to \infty$, according to~\eqref{eq:Vector_b_asymptotics}), and then $\widetilde{\mathbf{b}} \defeq \mathbf{A} \widetilde{\boldsymbol{\gamma}} + \mathbf{w} = \mathbf{b}$. In this way, problem~\eqref{eq:sliding_LASSO-LLSQ_problem} is a special case of the regularized problem~\eqref{eq:Regularized_optimization_problem}, where the general vector norm $\norm{\cdot}$ is replaced by the 1-norm $\onenorm{\cdot}$, and $\nu = 1$. Therefore, Lemma~\ref{lem:Regularization_with_arbitrary_norm} is applicable. Specifically, problem~\eqref{eq:sliding_LASSO-LLSQ_problem} is convex (although its objective function is not differentiable) and has an optimal solution $\mathbf{y}_n^*$, with its global minimum $L_n^* \defeq L_n(\mathbf{y}_n^*)$ satisfying $L_n^* = O(1)$ as $n \to \infty$. Interestingly, the universal convergence rate~\eqref{eq:Universal_convergence_rate_for_index_k} still holds, despite the fact that $\mathbf{y}_n^*$ does \emph{not} have an explicit/closed-form expression (in contrast to sT-LLSQ). Consequently, we can use the sliding LASSO-LLSQ method to compute $\gamma_k$, as an alternative to sT-LLSQ. Note that every convex optimization problem, and thus~\eqref{eq:sliding_LASSO-LLSQ_problem}, can be globally solved in polynomial time using iterative techniques, such as interior-point methods~\cite{Boyd-Vandenberghe_2004}.    
\end{remark}

\subsection{Overview of Theoretical Results}

Table~\ref{tbl:Summary_of_theoretical_results_for_sLLSQ_and_sT-LLSQ} provides a bird's-eye view of the theoretical findings. Clearly, sLLSQ and sT-LLSQ are complementary approaches.

Based on the discussion of Lemmas~\ref{lem:Sufficient_condition_for_unique_solution_LLSQ}~and~\ref{lem:Sufficient_condition_for_unique_solution_T-LLSQ}, the iteration complexities of sLLSQ and sT-LLSQ are $O(k^2 m)$ and $O(k^2 \max(m,k))$, respectively, assuming a uniform/unit-cost model of computation. It is important to note that we can use \emph{symbolic matrix operations} (e.g., addition, multiplication, and inversion) to compute and simplify the matrices $\pseudoinv{\mathbf{A}} = \pseudoinv{\mathbf{A}}(n)$ and $\mathbf{C} = \mathbf{C}(n)$ \emph{only once}. In each iteration we calculate their final symbolic expressions for a particular value of $n$, and then multiply with the data vector $\mathbf{b} = \mathbf{b}(n)$ to get the optimal solution $\mathbf{y}_n^*$. In this way, we can reduce the iteration complexity of sLLSQ and sT-LLSQ in terms of time, at the expense of space complexity.

\begin{remark}
	On the one hand, according to Theorem~\ref{thm:The_sLLSQ_method} for sLLSQ, if the functions $\{g_l\}_{l=1}^{p+1}$ tend to zero \emph{sufficiently fast} so that $\vartheta_j(n) = o(1)$ as $n \to \infty$, then the convergence of $x_{j,n}^*$ to $\alpha_j$ is guaranteed; cf.~\eqref{eq:Function_vartheta_j} and \eqref{eq:Implication_o_sLLSQ}. On the other hand, based on Theorem~\ref{thm:The_sT-LLSQ_method} for sT-LLSQ, the sequence $x_{k,n}^*$ converges to $\alpha_k$ \emph{regardless} of how fast the functions $\{g_l\}_{l=1}^{p+1}$ vanish asymptotically, although the convergence rate depends on these functions through $\psi_k$; cf.~\eqref{eq:Function_psi_j}, for $j=k$, and \eqref{eq:x_kn_star_asymptotic_estimate_sT-LLSQ}. As a result, when sLLSQ cannot provide a convergence guarantee for $x_{k,n}^*$, we can compute the \emph{dominant critical parameter} $\alpha_k$ using the sT-LLSQ technique. 
\end{remark}

\begin{remark}
	A \emph{byproduct} of theoretical analysis is that the proposed methods can be useful even if the constants in an asymptotic expansion are \emph{explicitly} known, e.g., by using the saddle-point method. In such a case, if $x_j^*$ provably converges to $\alpha_j$ and the latter is a number-theoretic constant (e.g., $\pi$ and $e$), then we have automatically a numerical method to compute it.  
\end{remark}

\renewcommand{\arraystretch}{1.5}
\begin{table}[t!]
	\centering
	\caption{Summary of theoretical results.}
	\label{tbl:Summary_of_theoretical_results_for_sLLSQ_and_sT-LLSQ} 
	
	\resizebox{\textwidth}{!}{%
		\begin{tabular}{|l|l|}
		\hline
		
		\makecell[l]{\textbf{General hypotheses}} & \makecell[l]{Assumptions~\labelcref{assum:Asymptotic_expansion_and_form,assum:Functions_varphi_j,assum:Functions_u_j}}   \\
		\hline 
		
		\makecell[l]{\textbf{Uniqueness of the unknown parameters} \\ \textbf{(i.e., the vectors $\boldsymbol{\alpha} \in \mathcal{S}$ and $\boldsymbol{\gamma} \in \RR^k$)}} & \makecell[l]{Proposition~\ref{prop:Uniqueness_of_alpha_and_gamma}}  \\
		\hline
		
		\makecell[l]{\textbf{Asymptotic invariants under the} \\ \textbf{variable transformation $\mathbf{y} = \mathbf{u}(\mathbf{x})$}} & \makecell[l]{Proposition~\ref{prop:Asymptotic_invariants}}   \\
		\hline 	
		
		\makecell[l]{\textbf{Additional hypotheses}} & \makecell[l]{sLLSQ: Assumption~\ref{assum:Asymptotic_rank_assumption}/\ref{assum:Equivalent_asymptotic_rank_assumption} (see also Proposition~\ref{prop:Equivalence_of_asymptotic_rank_assumption})  \vspace{1mm} \\  sT-LLSQ: Assumption~\ref{assum:sT-LLSQ_when_k_equals_1} (otherwise, see Remark~\ref{rem:sT-LLSQ_for_single_parameter_and_unit_varphi_1})}   \\
		\hline 
		
		\makecell[l]{\textbf{Sequence of optimization problems}} & \makecell[l]{sLLSQ: $\underset{\mathbf{y} \in \RR^k}{\minimize} \; \twonorm{\mathbf{A} \mathbf{y} - \mathbf{b}}^2$ \vspace{1mm} \\ 
		sT-LLSQ: $\underset{\mathbf{y} \in \RR^k}{\minimize} \; \twonorm{\mathbf{A} \mathbf{y} - \mathbf{b}}^2 + \mu \twonorm{\mathbf{y}}^2$}   \\
		\hline 
		
		\makecell[l]{\textbf{Sequence of (unique) optimal solutions}} & \makecell[l]{sLLSQ: $\mathbf{y}_n^* = \inv{(\transp{\mathbf{A}} \mathbf{A})} \transp{\mathbf{A}} \mathbf{b} = \pseudoinv{\mathbf{A}} \mathbf{b}$ \vspace{1mm} \\ 
			sT-LLSQ: $\mathbf{y}_n^* = \inv{(\transp{\mathbf{A}} \mathbf{A} + \mu \mathbf{I})} \transp{\mathbf{A}} \mathbf{b} = \mathbf{C} \mathbf{b}$  \vspace{1mm} \\  sLLSQ and sT-LLSQ: $\mathbf{x}_n^* = \mathbf{u}^{-1}(\mathbf{y}_n^*)$}   \\
		\hline
		
		\makecell[l]{\textbf{Sequence of global minima}} & \makecell[l]{sLLSQ: $F_n^* \defeq \twonorm{\mathbf{A} \mathbf{y}_n^* - \mathbf{b}}^2$ \vspace{1mm} \\ 
		sT-LLSQ: $G_n^* \defeq \twonorm{\mathbf{A} \mathbf{y}_n^* - \mathbf{b}}^2 + \mu \twonorm{\mathbf{y}_n^*}^2$}   \\
		\hline 
		
		\makecell[l]{\textbf{Main theorems}} & \makecell[l]{sLLSQ: Theorem~\ref{thm:The_sLLSQ_method}  \vspace{1mm} \\ 
		sT-LLSQ: Theorem~\ref{thm:The_sT-LLSQ_method}}    \\
		\hline	
	\end{tabular}%
	}
\end{table}

\section{Fundamental Applications in Analytic Combinatorics}
\label{sec:Fundamental_applications_in_analytic_combinatorics}

Next, we examine some important applications of sLLSQ and sT-LLSQ to asymptotic forms that are frequently found in analytic combinatorics and are given in Table~\ref{tbl:Common_asymptotic_forms}.\footnote{Those are just indicative examples, and therefore constitute a non-exhaustive list.} The first asymptotic form (AF-1) is inspired by the asymptotics of the factorial function, $n!$ (see Section~\ref{subsec:Factorial_function}). Also, the quantity $\alpha_j^{n^{\xi}}$ is a \emph{stretched-exponential term}, where $\xi \in \RR_{>0}$ with $\xi \neq 1$; typically, $\xi \in \left\{\frac{2}{3}, \frac{1}{2}, \frac{1}{3} \right\}$ as in AF-3--AF-5 and AF-7--AF-9.\footnote{For example, stretched-exponential terms appear in several types of integer partitions \cite[\S{VIII.6}, Figure~VIII.8]{Flajolet-Sedgewick_2009}, where all constants can be explicitly determined by the saddle-point method. See also the paper of Price,~Fang,~and~Wallner~\cite{Price-Fang-Wallner_2021} as well as the work of Wallner~\cite{Wallner_2026}.}

\renewcommand{\arraystretch}{1.5}
\begin{table}[t!]
	\centering
	\caption{Common asymptotic forms that satisfy \eqref{eq:Proven_asymptotic_form_with_gamma}, i.e., $\widehat{f}(n;\boldsymbol{\alpha}) = \exp\left( \sum_{j=1}^k \varphi_j(n) u_j(\alpha_j) \right) = \exp\left( \sum_{j=1}^k \varphi_j(n) \gamma_j \right)$.} 
	\label{tbl:Common_asymptotic_forms}
	
	\resizebox{\textwidth}{!}{%
		\begin{tabular}{|l|l|l|l|l|l|} 
			\hline 
			\makecell[c]{\textbf{Acronym of} \\ \textbf{asymptotic} \\ \textbf{form}} & \makecell[c]{$k$} & \makecell[c]{$\widehat{f}(n;\boldsymbol{\alpha})$} & \makecell[c]{\textbf{Conditions on} \\ $\boldsymbol{\alpha} = \transp{[\alpha_1,\dots,\alpha_k]}$ \\ \textbf{($\alpha_j \in \mathcal{S}_j ,\ \forall j \in \{1,\dots,k\}$)}} & \makecell[c]{$\boldsymbol{\gamma} = \transp{[\gamma_1,\dots,\gamma_k]}$} & \makecell[c]{$\boldsymbol{\varphi}(n) = \transp{[\varphi_1(n),\dots,\varphi_k(n)]}$} \\  
			\hline
			
			AF-1 & 4 & $\alpha_1 n^{\alpha_2} \alpha_3^n n^{\alpha_4 n}$ & $\alpha_1, \alpha_3 \in \RR_{>0}$ and $\alpha_2, \alpha_4 \in \RR$ & $\transp{[\log(\alpha_1), \alpha_2, \log(\alpha_3), \alpha_4]}$ & $\transp{[1,\log(n),n,n\log(n)]}$   \\
			\hline
			
			AF-2 & 3 & $\alpha_1 n^{\alpha_2} \alpha_3^n$ & \multirow{4}*{$\alpha_1, \alpha_3 \in \RR_{>0}$ and $\alpha_2 \in \RR$} & \multirow{4}*{$\transp{[\log(\alpha_1), \alpha_2, \log(\alpha_3)]}$} & $\transp{[1,\log(n),n]}$   \\
			\cline{1-3}\cline{6-6}
			
			AF-3 & 3 & $\alpha_1 n^{\alpha_2} \alpha_3^{n^{2/3}}$ &  &  & $\transp{[1,\log(n), n^{2/3}]}$   \\
			\cline{1-3}\cline{6-6}
			
			AF-4 & 3 & $\alpha_1 n^{\alpha_2} \alpha_3^{n^{1/2}}$ &  &  & $\transp{[1,\log(n), n^{1/2}]}$   \\
			\cline{1-3}\cline{6-6}
			
			AF-5 & 3 & $\alpha_1 n^{\alpha_2} \alpha_3^{n^{1/3}}$ &  &  & $\transp{[1,\log(n), n^{1/3}]}$   \\
			\hline
			
			AF-6 & 2 & $\alpha_1 \alpha_2^n$ & \multirow{4}*{$\alpha_1,\alpha_2 \in \RR_{>0}$} & \multirow{4}*{$\transp{[\log(\alpha_1), \log(\alpha_2)]}$} & $\transp{[1,n]}$   \\
			\cline{1-3}\cline{6-6}
			
			AF-7 & 2 & $\alpha_1 \alpha_2^{n^{2/3}}$ &  &  & $\transp{[1,n^{2/3}]}$   \\
			\cline{1-3}\cline{6-6}
			
			AF-8 & 2 & $\alpha_1 \alpha_2^{n^{1/2}}$ &  &  & $\transp{[1,n^{1/2}]}$   \\
			\cline{1-3}\cline{6-6}
			
			AF-9 & 2 & $\alpha_1 \alpha_2^{n^{1/3}}$ &  &  & $\transp{[1,n^{1/3}]}$   \\
			\hline
			
			AF-10 & 2 & $\alpha_1 n^{\alpha_2}$ & $\alpha_1 \in \RR_{>0}$ and $\alpha_2 \in \RR$ & $\transp{[\log(\alpha_1), \alpha_2]}$ & $\transp{[1,\log(n)]}$   \\
			\hline
			
			AF-11 & 1 & $\alpha_1$ & $\alpha_1 \in \RR_{>0}$ & $\transp{[\log(\alpha_1)]}$ & $\transp{[1]}$   \\
			\hline			
		\end{tabular}%
	}
\end{table}

Furthermore, the asymptotic estimates achieved by sLLSQ are shown in Tables~\ref{tbl:Fundamental_results_for_common_asymptotic_forms_sLLSQ}~and~\ref{tbl:Fundamental_results_for_common_asymptotic_forms_sLLSQ_continued}, while those by sT-LLSQ are listed in Table~\ref{tbl:Fundamental_results_for_common_asymptotic_forms_sT-LLSQ}. Recall that the convergence-rate invariant~\eqref{eq:Convergence-rate_invariant_y_to_x_with_h} ensures that: if $y_{j,n}^* - \gamma_j = O(h_j(n))$ and $h_j(n) = o(1)$, then $x_{j,n}^* - \alpha_j = \Theta(y_{j,n}^* - \gamma_j) = O(h_j(n)) = o(1)$. Note that the results have been obtained using a computer algebra system that can compute the \emph{Puiseux series} of symbolic expressions of $n$, as $n \to \infty$; the real coefficients $\{\beta_l\}_{l=1}^p$ and the regularization parameter $\mu \in \RR_{>0}$ are treated as \emph{symbolic constants}.\footnote{The Puiseux series is a generalization of Laurent series, where the exponents of the variable are allowed to be rational numbers with a common denominator (instead of only being integers).} In this way, we can find simple $\Theta$-estimates about the expressions contained in the $\max$-operation of the sLLSQ and sT-LLSQ asymptotics (see Theorems~\ref{thm:The_sLLSQ_method}~and~\ref{thm:The_sT-LLSQ_method}). In particular, if $y_{j,n}^* - \gamma_j = O\left( \max_i {\abs{\theta_i(n)}} \right)$ and $\theta_i(n) = \Theta\left( h_i(n) \right)$ for all $i$, then $y_{j,n}^* - \gamma_j = O\left( \max_i {\abs{h_i(n)}} \right)$, where each function $h_i$ is usually simpler than $\theta_i$.

\renewcommand{\arraystretch}{1.5}
\begin{table}[t!]
	\centering
	\caption{Fundamental results of sLLSQ for common asymptotic forms. See also Tables~\ref{tbl:Asymptotics_learning_theory_General_problem},~\ref{tbl:Summary_of_theoretical_results_for_sLLSQ_and_sT-LLSQ}~and~\ref{tbl:Common_asymptotic_forms} (the asymptotic forms AF-1 to AF-11 are given in Table~\ref{tbl:Common_asymptotic_forms}). For simplicity, we assume that $g_{p+1}$ is asymptotically regular (see Definition~\ref{def:Asymptotic_regularity} and Remark~\ref{rem:Asymptotic_regularity}).}
	\label{tbl:Fundamental_results_for_common_asymptotic_forms_sLLSQ} 

\resizebox{\textwidth}{!}{%
\begin{tabular}{|l|l|l|l|}
\hline 
\makecell[c]{\multirow{3}{*}{\makecell[c]{\textbf{Asymptotic} \\ \textbf{form}}}} & \multicolumn{3}{c|}{\makecell[c]{\textbf{sLLSQ with $m=k$} \\ \textbf{(asymptotic estimates as $n \to \infty$)}}}  \\
\cline{2-4}
& \makecell[c]{$\determ{\transp{\mathbf{A}} \mathbf{A}}$} & \makecell[c]{$p = 0$, $g_1(n) = o(1)$.} & \makecell[c]{$p = 1$, $g_1(n) = \frac{1}{n}$, $g_2(n) = o\left( \frac{1}{n} \right)$.}   \\
\hline

\makecell[l]{AF-1} & \makecell[l]{$\Theta\left( \frac{1}{n^8} \right)$} & \makecell[l]{$\norm{\mathbf{y}_n^* - \boldsymbol{\gamma}} = O\left( n^3 \log(n) g_1(n) \right)$, \vspace{1mm} \\ $y_{1,n}^* - \gamma_1 = O\left( n^3 \log(n) g_1(n) \right)$, \\ $y_{2,n}^* - \gamma_2 = O\left( n^3 g_1(n) \right)$, \\ $y_{3,n}^* - \gamma_3 = O\left( n^2 \log(n) g_1(n) \right)$, \\ 
$y_{4,n}^* - \gamma_4 = O\left( n^2 g_1(n) \right)$.} & \makecell[l]{$\norm{\mathbf{y}_n^* - \boldsymbol{\gamma}} = O\left( \log(n) \max\left( \tfrac{1}{n}, n^3 \abs{g_2(n)} \right) \right)$, \vspace{1mm} \\ $y_{1,n}^* - \gamma_1 = O\left( \log(n) \max\left( \tfrac{1}{n}, n^3 \abs{g_2(n)} \right) \right)$, \\ $y_{2,n}^* - \gamma_2 = O\left( \max\left( \tfrac{1}{n}, n^3 \abs{g_2(n)} \right) \right)$, \\ $y_{3,n}^* - \gamma_3 = O\left( \log(n) \max\left( \tfrac{1}{n^2}, n^2 \abs{g_2(n)} \right) \right)$,  \\ 	$y_{4,n}^* - \gamma_4 = O\left( \max\left( \tfrac{1}{n^2}, n^2 \abs{g_2(n)} \right) \right)$.}  \\
\hline 

\makecell[l]{AF-2} & \makecell[l]{$\Theta\left( \frac{1}{n^4} \right)$} & \makecell[l]{$\norm{\mathbf{y}_n^* - \boldsymbol{\gamma}} = O\left( n^2 \log(n) g_1(n) \right)$, \vspace{1mm} \\ $y_{1,n}^* - \gamma_1 = O\left( n^2 \log(n) g_1(n) \right)$, \\ $y_{2,n}^* - \gamma_2 = O\left( n^2 g_1(n) \right)$, \\ $y_{3,n}^* - \gamma_3 = O\left( n g_1(n) \right)$.} & \makecell[l]{$\norm{\mathbf{y}_n^* - \boldsymbol{\gamma}}  = O\left( \log(n) \max\left( \tfrac{1}{n}, n^2 \abs{g_2(n)} \right) \right)$, \vspace{1mm} \\ $y_{1,n}^* - \gamma_1 = O\left( \log(n) \max\left( \frac{1}{n}, n^2 \abs{g_2(n)} \right) \right)$, \\ $y_{2,n}^* - \gamma_2  = O\left( \max\left( \tfrac{1}{n}, n^2 \abs{g_2(n)} \right) \right)$, \\ $y_{3,n}^* - \gamma_3 = O\left( \max\left( \tfrac{1}{n^2}, n \abs{g_2(n)} \right) \right) = o(1)$.}  \\
\hline 

\makecell[l]{AF-3} & \makecell[l]{$\Theta\left( \frac{1}{n^{14/3}} \right)$} & \makecell[l]{$\norm{\mathbf{y}_n^* - \boldsymbol{\gamma}} = O\left( n^2 \log(n) g_1(n) \right)$, \vspace{1mm} \\ $y_{1,n}^* - \gamma_1 = O\left( n^2 \log(n) g_1(n) \right)$, \\ $y_{2,n}^* - \gamma_2 = O\left( n^2 g_1(n) \right)$, \\ $y_{3,n}^* - \gamma_3 = O\left( n^{4/3} g_1(n) \right)$.} & \makecell[l]{$\norm{\mathbf{y}_n^* - \boldsymbol{\gamma}} = O\left( \log(n) \max\left( \tfrac{1}{n}, n^2 \abs{g_2(n)} \right) \right)$, \vspace{1mm} \\ $y_{1,n}^* - \gamma_1 = O\left( \log(n) \max\left( \tfrac{1}{n}, n^2 \abs{g_2(n)} \right) \right)$, \\ $y_{2,n}^* - \gamma_2 = O\left( \max\left( \tfrac{1}{n}, n^2 \abs{g_2(n)} \right) \right)$, \\ $y_{3,n}^* - \gamma_3 = O\left( \max\left( \tfrac{1}{n^{5/3}}, n^{4/3} \abs{g_2(n)} \right) \right)$.}  \\
\hline

\makecell[l]{AF-4} & \makecell[l]{$\Theta\left( \frac{1}{n^5} \right)$} & \makecell[l]{$\norm{\mathbf{y}_n^* - \boldsymbol{\gamma}} = O\left( n^2 \log(n) g_1(n) \right)$, \vspace{1mm} \\ $y_{1,n}^* - \gamma_1 = O\left( n^2 \log(n) g_1(n) \right)$, \\ $y_{2,n}^* - \gamma_2 = O\left( n^2 g_1(n) \right)$, \\ $y_{3,n}^* - \gamma_3 = O\left( n^{3/2} g_1(n) \right)$.} & \makecell[l]{$\norm{\mathbf{y}_n^* - \boldsymbol{\gamma}} = O\left( \log(n) \max\left( \tfrac{1}{n}, n^2 \abs{g_2(n)} \right) \right)$, \vspace{1mm} \\ $y_{1,n}^* - \gamma_1 = O\left( \log(n) \max\left( \tfrac{1}{n}, n^2 \abs{g_2(n)} \right) \right)$, \\ $y_{2,n}^* - \gamma_2 = O\left( \max\left( \tfrac{1}{n}, n^2 \abs{g_2(n)} \right) \right)$, \\ $y_{3,n}^* - \gamma_3 = O\left( \max\left( \tfrac{1}{n^{3/2}}, n^{3/2} \abs{g_2(n)} \right) \right)$.}  \\
\hline

\makecell[l]{AF-5} & \makecell[l]{$\Theta\left( \frac{1}{n^{16/3}} \right)$} & \makecell[l]{$\norm{\mathbf{y}_n^* - \boldsymbol{\gamma}} = O\left( n^2 \log(n) g_1(n) \right)$, \vspace{1mm} \\ $y_{1,n}^* - \gamma_1 = O\left( n^2 \log(n) g_1(n) \right)$, \\ $y_{2,n}^* - \gamma_2 = O\left( n^2 g_1(n) \right)$, \\ $y_{3,n}^* - \gamma_3 = O\left( n^{5/3} g_1(n) \right)$.} & \makecell[l]{$\norm{\mathbf{y}_n^* - \boldsymbol{\gamma}} = O\left( \log(n) \max\left( \tfrac{1}{n}, n^2 \abs{g_2(n)} \right) \right)$, \vspace{1mm} \\ $y_{1,n}^* - \gamma_1 = O\left( \log(n) \max\left( \tfrac{1}{n}, n^2 \abs{g_2(n)} \right) \right)$, \\ $y_{2,n}^* - \gamma_2 = O\left( \max\left( \tfrac{1}{n}, n^2 \abs{g_2(n)} \right) \right)$, \\ $y_{3,n}^* - \gamma_3 = O\left( \max\left( \tfrac{1}{n^{4/3}}, n^{5/3} \abs{g_2(n)} \right) \right)$.}  \\
\hline 

\makecell[l]{AF-6} & \makecell[l]{$1 = \Theta(1)$} & \makecell[l]{$\norm{\mathbf{y}_n^* - \boldsymbol{\gamma}} = O\left( n g_1(n) \right)$, \vspace{1mm} \\ $y_{1,n}^* - \gamma_1 = O\left( n g_1(n) \right)$, \\ $y_{2,n}^* - \gamma_2 = O\left( g_1(n) \right) = o(1)$.} & \makecell[l]{$\norm{\mathbf{y}_n^* - \boldsymbol{\gamma}} = O\left( \max\left(\tfrac{1}{n}, n \abs{g_2(n)} \right) \right) = o(1)$, \vspace{1mm} \\ $y_{1,n}^* - \gamma_1 = O\left( \max\left(\tfrac{1}{n}, n \abs{g_2(n)} \right) \right) = o(1)$, \\ $y_{2,n}^* - \gamma_2 =  O\left( \max\left( \tfrac{1}{n^2}, \abs{g_2(n)} \right) \right) = o(1)$.}  \\
\hline 

\makecell[l]{AF-7} & \makecell[l]{$\Theta\left( \frac{1}{n^{2/3}} \right)$} & \makecell[l]{$\norm{\mathbf{y}_n^* - \boldsymbol{\gamma}} = O\left( n g_1(n) \right)$, \vspace{1mm} \\ $y_{1,n}^* - \gamma_1 = O\left( n g_1(n) \right)$, \\ $y_{2,n}^* - \gamma_2 = O\left( n^{1/3} g_1(n) \right)$.}  & \makecell[l]{$\norm{\mathbf{y}_n^* - \boldsymbol{\gamma}} = O\left( \max\left(\tfrac{1}{n}, n \abs{g_2(n)} \right) \right) = o(1)$, \vspace{1mm} \\ $y_{1,n}^* - \gamma_1 = O\left( \max\left(\tfrac{1}{n}, n \abs{g_2(n)} \right) \right) = o(1)$, \\ $y_{2,n}^* - \gamma_2 =  O\left( \max\left( \tfrac{1}{n^{5/3}}, n^{1/3} \abs{g_2(n)} \right) \right) = o(1)$.} \\
\hline 

\makecell[l]{AF-8} & \makecell[l]{$\Theta\left( \frac{1}{n} \right)$} & \makecell[l]{$\norm{\mathbf{y}_n^* - \boldsymbol{\gamma}} = O\left( n g_1(n) \right)$, \vspace{1mm} \\ $y_{1,n}^* - \gamma_1 = O\left( n g_1(n) \right)$, \\ $y_{2,n}^* - \gamma_2 = O\left( n^{1/2} g_1(n) \right)$.}  & \makecell[l]{$\norm{\mathbf{y}_n^* - \boldsymbol{\gamma}} = O\left( \max\left(\tfrac{1}{n}, n \abs{g_2(n)} \right) \right) = o(1)$, \vspace{1mm} \\ $y_{1,n}^* - \gamma_1 = O\left( \max\left(\tfrac{1}{n}, n \abs{g_2(n)} \right) \right) = o(1)$, \\ $y_{2,n}^* - \gamma_2 =  O\left( \max\left( \tfrac{1}{n^{3/2}}, n^{1/2} \abs{g_2(n)} \right) \right) = o(1)$.} \\
\hline 

\makecell[l]{AF-9} & \makecell[l]{$\Theta\left( \frac{1}{n^{4/3}} \right)$} & \makecell[l]{$\norm{\mathbf{y}_n^* - \boldsymbol{\gamma}} = O\left( n g_1(n) \right)$, \vspace{1mm} \\ $y_{1,n}^* - \gamma_1 = O\left( n g_1(n) \right)$, \\ $y_{2,n}^* - \gamma_2 = O\left( n^{2/3} g_1(n) \right)$.}  & \makecell[l]{$\norm{\mathbf{y}_n^* - \boldsymbol{\gamma}} = O\left( \max\left(\tfrac{1}{n}, n \abs{g_2(n)} \right) \right) = o(1)$, \vspace{1mm} \\ $y_{1,n}^* - \gamma_1 = O\left( \max\left(\tfrac{1}{n}, n \abs{g_2(n)} \right) \right) = o(1)$, \\ $y_{2,n}^* - \gamma_2 = O\left( \max\left( \tfrac{1}{n^{4/3}}, n^{2/3} \abs{g_2(n)} \right) \right) = o(1)$.} \\
\hline  

\makecell[l]{AF-10} & \makecell[l]{$\Theta\left( \frac{1}{n^2} \right)$} & \makecell[l]{$\norm{\mathbf{y}_n^* - \boldsymbol{\gamma}} = O\left( n \log(n) g_1(n) \right)$, \vspace{1mm} \\ $y_{1,n}^* - \gamma_1 = O\left( n \log(n) g_1(n) \right)$, \\ $y_{2,n}^* - \gamma_2 = O\left( n g_1(n) \right)$.} & \makecell[l]{$\norm{\mathbf{y}_n^* - \boldsymbol{\gamma}} = O\left( \log(n) \max\left(\tfrac{1}{n}, n \abs{g_2(n)} \right) \right)$, \vspace{1mm} \\ $y_{1,n}^* - \gamma_1 = O\left( \log(n) \max\left(\tfrac{1}{n}, n \abs{g_2(n)} \right) \right)$, \\ $y_{2,n}^* - \gamma_2 =  O\left( \max\left( \tfrac{1}{n}, n \abs{g_2(n)} \right) \right) = o(1)$.}  \\
\hline 

\makecell[l]{AF-11} & \makecell[l]{$1 = \Theta(1)$} & \makecell[l]{$y_{1,n}^* - \gamma_1 = O(g_1(n)) = o(1)$}  & \makecell[l]{$y_{1,n}^* - \gamma_1 = O\left( \max\left( \frac{1}{n}, \abs{g_2(n)} \right) \right) = O\left( \frac{1}{n} \right) = o(1)$} \\
\hline  
\end{tabular}%
}
\end{table}

\renewcommand{\arraystretch}{1.5}
\begin{table}[t!]
	\centering
	\caption{Table~\ref{tbl:Fundamental_results_for_common_asymptotic_forms_sLLSQ} (continued).} 
	\label{tbl:Fundamental_results_for_common_asymptotic_forms_sLLSQ_continued}
		
	\resizebox{\textwidth}{!}{%
		\begin{tabular}{|l|l|l|}
		\hline 
		
		\multirow{2}{*}{\makecell[c]{\textbf{Asymptotic} \\ \textbf{form}}} & \multicolumn{2}{c|}{\makecell[c]{\textbf{sLLSQ with $m=k$} \\ \textbf{(asymptotic estimates as $n \to \infty$)}}}  \\
		\cline{2-3}
		& \makecell[c]{$p = 2$, $g_1(n) = \frac{1}{n}$, $g_2(n) = \frac{1}{n^2}$, $g_3(n) = o\left( \frac{1}{n^2} \right)$.} & \makecell[c]{$p = 3$, $g_1(n) = \frac{1}{n}$, $g_2(n) = \frac{1}{n^2}$, $g_3(n) = \frac{1}{n^3}$, $g_4(n) = o\left( \frac{1}{n^3} \right)$.}   \\
		\hline
		
		\makecell[l]{AF-1} & \makecell[l]{$\norm{\mathbf{y}_n^* - \boldsymbol{\gamma}} = O\left( \log(n) \max\left( \tfrac{1}{n}, n^3 \abs{g_3(n)} \right) \right)$, \vspace{1mm} \\ $y_{1,n}^* - \gamma_1 = O\left( \log(n) \max\left( \tfrac{1}{n}, n^3 \abs{g_3(n)} \right) \right)$, \\ $y_{2,n}^* - \gamma_2 = O\left( \max\left( \tfrac{1}{n}, n^3 \abs{g_3(n)} \right) \right)$, \\ $y_{3,n}^* - \gamma_3 = O\left( \log(n) \max\left( \tfrac{1}{n^2}, n^2 \abs{g_3(n)} \right) \right)$,  \\ 	$y_{4,n}^* - \gamma_4 = O\left( \max\left( \tfrac{1}{n^2}, n^2 \abs{g_3(n)} \right) \right) = o(1)$.} & \makecell[l]{$\norm{\mathbf{y}_n^* - \boldsymbol{\gamma}} = O\left( \log(n) \max\left( \tfrac{1}{n}, n^3 \abs{g_4(n)} \right) \right)$, \vspace{1mm} \\ $y_{1,n}^* - \gamma_1 = O\left( \log(n) \max\left( \tfrac{1}{n}, n^3 \abs{g_4(n)} \right) \right)$, \\ $y_{2,n}^* - \gamma_2 = O\left( \max\left( \tfrac{1}{n}, n^3 \abs{g_4(n)} \right) \right) = o(1)$, \\ $y_{3,n}^* - \gamma_3 = O\left( \log(n) \max\left( \tfrac{1}{n^2}, n^2 \abs{g_4(n)} \right) \right) = o(1)$,  \\ 	$y_{4,n}^* - \gamma_4 = O\left( \max\left( \tfrac{1}{n^2}, n^2 \abs{g_4(n)} \right) \right) = o(1)$.}  \\
		\hline 
		
		\makecell[l]{AF-2} & \makecell[l]{$				\norm{\mathbf{y}_n^* - \boldsymbol{\gamma}} = O\left( \log(n) \max\left( \tfrac{1}{n}, n^2 \abs{g_3(n)} \right) \right)$, \vspace{1mm} \\ $y_{1,n}^* - \gamma_1 = O\left( \log(n) \max\left( \frac{1}{n}, n^2 \abs{g_3(n)} \right) \right)$, \\ $y_{2,n}^* - \gamma_2  = O\left( \max\left( \tfrac{1}{n}, n^2 \abs{g_3(n)} \right) \right) = o(1)$, \\ $y_{3,n}^* - \gamma_3 = O\left( \max\left( \tfrac{1}{n^2}, n \abs{g_3(n)} \right) \right) = o(1)$.} & \makecell[l]{$\norm{\mathbf{y}_n^* - \boldsymbol{\gamma}} = O\left( \log(n) \max\left( \tfrac{1}{n}, n^2 \abs{g_4(n)} \right) \right) = O\left( \tfrac{\log(n)}{n} \right) = o(1)$, \vspace{1mm} \\ $y_{1,n}^* - \gamma_1 = O\left( \log(n) \max\left( \frac{1}{n}, n^2 \abs{g_4(n)} \right) \right) = O\left( \tfrac{\log(n)}{n} \right) = o(1)$, \\ $y_{2,n}^* - \gamma_2  = O\left( \max\left( \tfrac{1}{n}, n^2 \abs{g_4(n)} \right) \right) = O\left( \tfrac{1}{n} \right) = o(1)$, \\ $y_{3,n}^* - \gamma_3 = O\left( \max\left( \tfrac{1}{n^2}, n \abs{g_4(n)} \right) \right) = O\left( \tfrac{1}{n^2} \right) = o(1)$.}  \\
		\hline 
		
		\makecell[l]{AF-3} & \makecell[l]{$				\norm{\mathbf{y}_n^* - \boldsymbol{\gamma}} = O\left( \log(n) \max\left( \tfrac{1}{n}, n^2 \abs{g_3(n)} \right) \right)$,  \vspace{1mm} \\ $y_{1,n}^* - \gamma_1 = O\left( \log(n) \max\left( \tfrac{1}{n}, n^2 \abs{g_3(n)} \right) \right)$, \\ $y_{2,n}^* - \gamma_2 = O\left( \max\left( \tfrac{1}{n}, n^2 \abs{g_3(n)} \right) \right) =  o(1)$, \\ $y_{3,n}^* - \gamma_3 = O\left( \max\left( \tfrac{1}{n^{5/3}}, n^{4/3} \abs{g_3(n)} \right) \right) =  o(1)$.} & \makecell[l]{$\norm{\mathbf{y}_n^* - \boldsymbol{\gamma}} = O\left( \log(n) \max\left( \tfrac{1}{n}, n^2 \abs{g_4(n)} \right) \right) = O\left( \tfrac{\log(n)}{n} \right) =  o(1)$, \vspace{1mm} \\ $y_{1,n}^* - \gamma_1 = O\left( \log(n) \max\left( \tfrac{1}{n}, n^2 \abs{g_4(n)} \right) \right) = O\left( \tfrac{\log(n)}{n} \right) =  o(1)$, \\ $y_{2,n}^* - \gamma_2 = O\left( \max\left( \tfrac{1}{n}, n^2 \abs{g_4(n)} \right) \right) = O\left( \tfrac{1}{n} \right) =  o(1)$, \\ $y_{3,n}^* - \gamma_3 = O\left( \max\left( \tfrac{1}{n^{5/3}}, n^{4/3} \abs{g_4(n)} \right) \right) = O\left( \tfrac{1}{n^{5/3}} \right) =  o(1)$.}  \\
		\hline
		
		\makecell[l]{AF-4} & \makecell[l]{$				\norm{\mathbf{y}_n^* - \boldsymbol{\gamma}} = O\left( \log(n) \max\left( \tfrac{1}{n}, n^2 \abs{g_3(n)} \right) \right)$, \vspace{1mm} \\ $y_{1,n}^* - \gamma_1 = O\left( \log(n) \max\left( \tfrac{1}{n}, n^2 \abs{g_3(n)} \right) \right)$, \\ $y_{2,n}^* - \gamma_2 = O\left( \max\left( \tfrac{1}{n}, n^2 \abs{g_3(n)} \right) \right) = o(1)$, \\ $y_{3,n}^* - \gamma_3 = O\left( \max\left( \tfrac{1}{n^{3/2}}, n^{3/2} \abs{g_3(n)} \right) \right) = o(1)$.} & \makecell[l]{$\norm{\mathbf{y}_n^* - \boldsymbol{\gamma}} = O\left( \log(n) \max\left( \tfrac{1}{n}, n^2 \abs{g_4(n)} \right) \right) = O\left( \tfrac{\log(n)}{n} \right) = o(1)$, \vspace{1mm} \\ $y_{1,n}^* - \gamma_1 = O\left( \log(n) \max\left( \tfrac{1}{n}, n^2 \abs{g_4(n)} \right) \right) = O\left( \tfrac{\log(n)}{n} \right) = o(1)$, \\ $y_{2,n}^* - \gamma_2 = O\left( \max\left( \tfrac{1}{n}, n^2 \abs{g_4(n)} \right) \right) = O\left( \tfrac{1}{n} \right) = o(1)$, \\ $y_{3,n}^* - \gamma_3 = O\left( \max\left( \tfrac{1}{n^{3/2}}, n^{3/2} \abs{g_4(n)} \right) \right) = O\left( \tfrac{1}{n^{3/2}} \right) = o(1)$.}  \\
		\hline
		
		\makecell[l]{AF-5} & \makecell[l]{$				\norm{\mathbf{y}_n^* - \boldsymbol{\gamma}} = O\left( \log(n) \max\left( \tfrac{1}{n}, n^2 \abs{g_3(n)} \right) \right)$, \vspace{1mm} \\ $y_{1,n}^* - \gamma_1 = O\left( \log(n) \max\left( \tfrac{1}{n}, n^2 \abs{g_3(n)} \right) \right)$, \\ $y_{2,n}^* - \gamma_2 = O\left( \max\left( \tfrac{1}{n}, n^2 \abs{g_3(n)} \right) \right) = o(1)$, \\ $y_{3,n}^* - \gamma_3 = O\left( \max\left( \tfrac{1}{n^{4/3}}, n^{5/3} \abs{g_3(n)} \right) \right) = o(1)$.} & \makecell[l]{$\norm{\mathbf{y}_n^* - \boldsymbol{\gamma}} = O\left( \log(n) \max\left( \tfrac{1}{n}, n^2 \abs{g_4(n)} \right) \right) = O\left( \tfrac{\log(n)}{n} \right) = o(1)$, \vspace{1mm} \\ $y_{1,n}^* - \gamma_1 = O\left( \log(n) \max\left( \tfrac{1}{n}, n^2 \abs{g_4(n)} \right) \right) = O\left( \tfrac{\log(n)}{n} \right) = o(1)$, \\ $y_{2,n}^* - \gamma_2 = O\left( \max\left( \tfrac{1}{n}, n^2 \abs{g_4(n)} \right) \right) = O\left( \tfrac{1}{n} \right) = o(1)$, \\ $y_{3,n}^* - \gamma_3 = O\left( \max\left( \tfrac{1}{n^{4/3}}, n^{5/3} \abs{g_4(n)} \right) \right) = O\left( \tfrac{1}{n^{4/3}} \right) = o(1)$.}  \\
		\hline 
		
		\makecell[l]{AF-6} & \makecell[l]{$\norm{\mathbf{y}_n^* - \boldsymbol{\gamma}} = O\left( \max\left(\tfrac{1}{n}, n \abs{g_3(n)} \right) \right) = O\left( \tfrac{1}{n} \right) = o(1)$, \vspace{1mm} \\ $y_{1,n}^* - \gamma_1 = O\left( \max\left(\tfrac{1}{n}, n \abs{g_3(n)} \right) \right) = O\left( \tfrac{1}{n} \right) = o(1)$, \\ $y_{2,n}^* - \gamma_2 =  O\left( \max\left( \tfrac{1}{n^2}, \abs{g_3(n)} \right) \right) = O\left( \tfrac{1}{n^2} \right) = o(1)$.} & \makecell[l]{$\norm{\mathbf{y}_n^* - \boldsymbol{\gamma}} = O\left( \max\left(\tfrac{1}{n}, n \abs{g_4(n)} \right) \right) = O\left( \tfrac{1}{n} \right) = o(1)$, \vspace{1mm} \\ $y_{1,n}^* - \gamma_1 = O\left( \max\left(\tfrac{1}{n}, n \abs{g_4(n)} \right) \right) = O\left( \tfrac{1}{n} \right) = o(1)$, \\ $y_{2,n}^* - \gamma_2 =  O\left( \max\left( \tfrac{1}{n^2}, \abs{g_4(n)} \right) \right) = O\left( \tfrac{1}{n^2} \right) = o(1)$.}  \\
		\hline 
		
		\makecell[l]{AF-7} & \makecell[l]{$\norm{\mathbf{y}_n^* - \boldsymbol{\gamma}} = O\left( \max\left(\tfrac{1}{n}, n \abs{g_3(n)} \right) \right) = O\left( \tfrac{1}{n} \right) = o(1)$, \vspace{1mm} \\ $y_{1,n}^* - \gamma_1 = O\left( \max\left(\tfrac{1}{n}, n \abs{g_3(n)} \right) \right) = O\left( \tfrac{1}{n} \right) = o(1)$, \\ $y_{2,n}^* - \gamma_2 =  O\left( \max\left( \tfrac{1}{n^{5/3}}, n^{1/3} \abs{g_3(n)} \right) \right) = O\left( \tfrac{1}{n^{5/3}} \right) = o(1)$.}  & \makecell[l]{$\norm{\mathbf{y}_n^* - \boldsymbol{\gamma}} = O\left( \max\left(\tfrac{1}{n}, n \abs{g_4(n)} \right) \right) = O\left( \tfrac{1}{n} \right) = o(1)$, \vspace{1mm} \\ $y_{1,n}^* - \gamma_1 = O\left( \max\left(\tfrac{1}{n}, n \abs{g_4(n)} \right) \right) = O\left( \tfrac{1}{n} \right) = o(1)$, \\ $y_{2,n}^* - \gamma_2 =  O\left( \max\left( \tfrac{1}{n^{5/3}}, n^{1/3} \abs{g_4(n)} \right) \right) = O\left( \tfrac{1}{n^{5/3}} \right) = o(1)$.} \\
		\hline 
		
		\makecell[l]{AF-8} & \makecell[l]{$\norm{\mathbf{y}_n^* - \boldsymbol{\gamma}} = O\left( \max\left(\tfrac{1}{n}, n \abs{g_3(n)} \right) \right) = O\left( \tfrac{1}{n} \right) = o(1)$, \vspace{1mm} \\ $y_{1,n}^* - \gamma_1 = O\left( \max\left(\tfrac{1}{n}, n \abs{g_3(n)} \right) \right) = O\left( \tfrac{1}{n} \right) = o(1)$, \\ $y_{2,n}^* - \gamma_2 =  O\left( \max\left( \tfrac{1}{n^{3/2}}, n^{1/2} \abs{g_3(n)} \right) \right) = O\left( \tfrac{1}{n^{3/2}} \right) = o(1)$.}  & \makecell[l]{$\norm{\mathbf{y}_n^* - \boldsymbol{\gamma}} = O\left( \max\left(\tfrac{1}{n}, n \abs{g_4(n)} \right) \right) = O\left( \tfrac{1}{n} \right) = o(1)$, \vspace{1mm} \\ $y_{1,n}^* - \gamma_1 = O\left( \max\left(\tfrac{1}{n}, n \abs{g_4(n)} \right) \right) = O\left( \tfrac{1}{n} \right) = o(1)$, \\ $y_{2,n}^* - \gamma_2 =  O\left( \max\left( \tfrac{1}{n^{3/2}}, n^{1/2} \abs{g_4(n)} \right) \right) = O\left( \tfrac{1}{n^{3/2}} \right) = o(1)$.} \\
		\hline 
		
		\makecell[l]{AF-9} & \makecell[l]{$\norm{\mathbf{y}_n^* - \boldsymbol{\gamma}} = O\left( \max\left(\tfrac{1}{n}, n \abs{g_3(n)} \right) \right) = O\left( \tfrac{1}{n} \right) = o(1)$, \vspace{1mm} \\ $y_{1,n}^* - \gamma_1 = O\left( \max\left(\tfrac{1}{n}, n \abs{g_3(n)} \right) \right) = O\left( \tfrac{1}{n} \right) = o(1)$, \\ $y_{2,n}^* - \gamma_2 = O\left( \max\left( \tfrac{1}{n^{4/3}}, n^{2/3} \abs{g_3(n)} \right) \right) = O\left( \tfrac{1}{n^{4/3}} \right) = o(1)$.}  & \makecell[l]{$\norm{\mathbf{y}_n^* - \boldsymbol{\gamma}} = O\left( \max\left(\tfrac{1}{n}, n \abs{g_4(n)} \right) \right) = O\left( \tfrac{1}{n} \right) = o(1)$, \vspace{1mm} \\ $y_{1,n}^* - \gamma_1 = O\left( \max\left(\tfrac{1}{n}, n \abs{g_4(n)} \right) \right) = O\left( \tfrac{1}{n} \right) = o(1)$, \\ $y_{2,n}^* - \gamma_2 = O\left( \max\left( \tfrac{1}{n^{4/3}}, n^{2/3} \abs{g_4(n)} \right) \right) = O\left( \tfrac{1}{n^{4/3}} \right) = o(1)$.} \\
		\hline  
		
		\makecell[l]{AF-10} & \makecell[l]{$\norm{\mathbf{y}_n^* - \boldsymbol{\gamma}} = O\left( \log(n) \max\left(\tfrac{1}{n}, n \abs{g_3(n)} \right) \right) = O\left( \tfrac{\log(n)}{n} \right) = o(1)$, \vspace{1mm} \\ $y_{1,n}^* - \gamma_1 = O\left( \log(n) \max\left(\tfrac{1}{n}, n \abs{g_3(n)} \right) \right) = O\left( \tfrac{\log(n)}{n} \right) = o(1)$, \\ $y_{2,n}^* - \gamma_2 =  O\left( \max\left( \tfrac{1}{n}, n \abs{g_3(n)} \right) \right) = O\left( \tfrac{1}{n} \right) = o(1)$.} & \makecell[l]{$\norm{\mathbf{y}_n^* - \boldsymbol{\gamma}} = O\left( \log(n) \max\left(\tfrac{1}{n}, n \abs{g_4(n)} \right) \right) = O\left( \tfrac{\log(n)}{n} \right) = o(1)$, \vspace{1mm} \\ $y_{1,n}^* - \gamma_1 = O\left( \log(n) \max\left(\tfrac{1}{n}, n \abs{g_4(n)} \right) \right) = O\left( \tfrac{\log(n)}{n} \right) = o(1)$, \\ $y_{2,n}^* - \gamma_2 =  O\left( \max\left( \tfrac{1}{n}, n \abs{g_4(n)} \right) \right) = O\left( \tfrac{1}{n} \right) = o(1)$.}  \\
		\hline 
		
		\makecell[l]{AF-11} & \makecell[l]{$y_{1,n}^* - \gamma_1 = O\left( \max\left( \frac{1}{n}, \abs{g_3(n)} \right) \right) = O\left( \frac{1}{n} \right) = o(1)$}  & \makecell[l]{$y_{1,n}^* - \gamma_1 = O\left( \max\left( \frac{1}{n}, \abs{g_4(n)} \right) \right) = O\left( \frac{1}{n} \right) = o(1)$} \\
		\hline   
	\end{tabular}%
} 
\end{table}

\renewcommand{\arraystretch}{1.5}
\begin{table}[t!]
	\centering
	\caption{Fundamental results of sT-LLSQ for common asymptotic forms. See also Tables~\ref{tbl:Asymptotics_learning_theory_General_problem},~\ref{tbl:Summary_of_theoretical_results_for_sLLSQ_and_sT-LLSQ}~and~\ref{tbl:Common_asymptotic_forms} (the asymptotic forms AF-1 to AF-11 are given in Table~\ref{tbl:Common_asymptotic_forms}) as well as Remark~\ref{rem:No_improvement_in_convergence_rate}. In all the considered cases (except for AF-11), we also have $y_{j,n}^* - \gamma_j = O(1)$, as $n \to \infty$, for all $j \in \{1,\dots,k-1\}$. Specifically, $y_{j,n}^* = o(1)$ because $y_{j,n}^* = \gamma_j - \mu \mathbf{D}(j,:) \boldsymbol{\gamma} + \mathbf{C}(j,:) \left( \log\left( 1 + \sum_{l=1}^p {\beta_l \mathbf{g}_l} \right) + \mathbf{z} \right) = \gamma_j - \mu \mathbf{D}(j,:) \boldsymbol{\gamma} + O\left( \norm{\mathbf{C}(j,:)} \right)$ with $\gamma_j - \mu \mathbf{D}(j,:) \boldsymbol{\gamma} = o(1)$ and $\norm{\mathbf{C}(j,:)} = o(1)$, for all $j \in \{1,\dots,k-1\}$.}
	\label{tbl:Fundamental_results_for_common_asymptotic_forms_sT-LLSQ} 
	
	\footnotesize
		\begin{tabular}{|l|l|l|}
			\hline 
			\multirow{3}{*}{\makecell[c]{\textbf{Asymptotic} \\ \textbf{form}}} & \multicolumn{2}{c|}{\makecell[c]{\textbf{sT-LLSQ with $m=k$ and fixed $\mu \in \RR_{>0}$} \\ \textbf{(asymptotic estimate of $\norm{\mathbf{D}(k,:)}$} \\ \textbf{and convergence rate of $y_{k,n}^*$ as $n \to \infty$)}}} \\  
			\cline{2-3}
			& \makecell[c]{$\norm{\mathbf{D}(k,:)}$} & \makecell[c]{$p \in \NN_0$, $g_1(n) = o(1)$ and \\ $g_{l+1}(n) = o(g_l(n))$, $\forall l \in \{1,\dots,p\}$.}   \\
			\hline
			
			\makecell[l]{AF-1} & \makecell[l]{$\Theta\left( \frac{1}{\log(n)} \right)$} & \makecell[l]{$y_{4,n}^* - \gamma_4 = O\left( \frac{1}{\log(n)} \right) = o(1)$}    \\
			\hline
			
			\makecell[l]{AF-2} & \makecell[l]{$\Theta\left( \frac{\log(n)}{n} \right)$} & \makecell[l]{$y_{3,n}^* - \gamma_3 = O\left( \frac{\log(n)}{n} \right) = o(1)$}    \\
			\hline 
			
			\makecell[l]{AF-3} & \makecell[l]{$\Theta\left( \frac{\log(n)}{n^{2/3}} \right)$} & \makecell[l]{$y_{3,n}^* - \gamma_3 = O\left( \frac{\log(n)}{n^{2/3}} \right) = o(1)$}    \\
			\hline 
			
			\makecell[l]{AF-4} & \makecell[l]{$\Theta\left( \frac{\log(n)}{n^{1/2}} \right)$} & \makecell[l]{$y_{3,n}^* - \gamma_3 = O\left( \frac{\log(n)}{n^{1/2}} \right) = o(1)$}    \\
			\hline 
			
			\makecell[l]{AF-5} & \makecell[l]{$\Theta\left( \frac{\log(n)}{n^{1/3}} \right)$} & \makecell[l]{$y_{3,n}^* - \gamma_3 = O\left( \frac{\log(n)}{n^{1/3}} \right) = o(1)$}    \\
			\hline 
			
			\makecell[l]{AF-6} & \makecell[l]{$\Theta\left( \frac{1}{n} \right)$} & \makecell[l]{$y_{2,n}^* - \gamma_2 = O\left( \frac{1}{n} \right) = o(1)$}    \\
			\hline
			
			\makecell[l]{AF-7} & \makecell[l]{$\Theta\left( \frac{1}{n^{2/3}} \right)$} & \makecell[l]{$y_{2,n}^* - \gamma_2 = O\left( \frac{1}{n^{2/3}} \right) = o(1)$}    \\
			\hline
			
			\makecell[l]{AF-8} & \makecell[l]{$\Theta\left( \frac{1}{n^{1/2}} \right)$} & \makecell[l]{$y_{2,n}^* - \gamma_2 = O\left( \frac{1}{n^{1/2}} \right) = o(1)$}    \\
			\hline
			
			\makecell[l]{AF-9} & \makecell[l]{$\Theta\left( \frac{1}{n^{1/3}} \right)$} & \makecell[l]{$y_{2,n}^* - \gamma_2 = O\left( \frac{1}{n^{1/3}} \right) = o(1)$}   \\
			\hline
			
			\makecell[l]{AF-10} & \makecell[l]{$\Theta\left( \frac{1}{\log(n)} \right)$} & \makecell[l]{$y_{2,n}^* - \gamma_2 = O\left( \frac{1}{\log(n)} \right) = o(1)$}   \\
			\hline
			
			\makecell[l]{AF-11} & \makecell[l]{$\frac{1}{1 + \mu} = \Theta(1)$} & \makecell[l]{Not directly applicable, since \\ Assumption~\ref{assum:sT-LLSQ_when_k_equals_1} does not hold \\ ($k=1$ and $\varphi_1(n) \defeq 1 = O(1)$). \\ However, Remark~\ref{rem:sT-LLSQ_for_single_parameter_and_unit_varphi_1} gives \\ $\widetilde{y}_{1,n}^* - \gamma_1 = O(\norm{\mathbf{g}_1}) = o(1)$, \\ where $\widetilde{y}_{1,n}^* \defeq (1 + \mu) y_{1,n}^*$.}   \\
			\hline  
		\end{tabular}%
\end{table}

An \emph{important observation} is that information about higher-order asymptotics is \emph{crucial} in order to guarantee the convergence of sLLSQ. In other words, when sLLSQ is actually convergent, we are able to rigorously prove its convergence only if the order of asymptotic expansion is \emph{sufficiently high}. This is clearly explained in the following remarks.

\begin{remark} \label{rem:AF-1_sLLSQ_convergence}
	For the asymptotic form AF-1 in Table~\ref{tbl:Common_asymptotic_forms} with $p=3$ and $g_l(n) = \frac{1}{n^l}$ for all $l \in \{1,2,3,4\}$, i.e., when 
	\begin{equation*}
		f(n) = \alpha_1 n^{\alpha_2} \alpha_3^n n^{\alpha_4 n} \left( 1 + \sum_{l=1}^3 \frac{\beta_l}{n^l} + O\left( \frac{1}{n^4} \right) \right)   , 
	\end{equation*} 
	we obtain the following convergence rates of sLLSQ with $m=k=4$ (see Table~\ref{tbl:Fundamental_results_for_common_asymptotic_forms_sLLSQ_continued}): 
	\begin{align*}
		\norm{\mathbf{y}_n^* - \boldsymbol{\gamma}} & = O\left( \log(n) \max\left( \tfrac{1}{n}, n^3 \abs{g_4(n)} \right) \right) = O\left( \frac{\log(n)}{n} \right) = o(1)  ,  \\
		y_{1,n}^* - \gamma_1 & = O\left( \log(n) \max\left( \tfrac{1}{n}, n^3 \abs{g_4(n)} \right) \right) = O\left( \frac{\log(n)}{n} \right) = o(1)  , \\
		y_{2,n}^* - \gamma_2 & = O\left( \max\left( \tfrac{1}{n}, n^3 \abs{g_4(n)} \right) \right) = O\left( \frac{1}{n} \right) = o(1) , \\
		y_{3,n}^* - \gamma_3 & = O\left( \log(n) \max\left( \tfrac{1}{n^2}, n^2 \abs{g_4(n)} \right) \right) = O\left( \frac{\log(n)}{n^2} \right) = o(1)  ,  \\
		y_{4,n}^* - \gamma_4 & = O\left( \max\left( \tfrac{1}{n^2}, n^2 \abs{g_4(n)} \right) \right) = O\left( \frac{1}{n^2} \right) = o(1) .
	\end{align*}
	In other words, the sLLSQ method is \emph{provably convergent} in such a case, with given third-order (or, higher-order) asymptotics. Table~\ref{tbl:Fundamental_results_for_common_asymptotic_forms_sT-LLSQ} gives the convergence rate $y_{4,n}^* - \gamma_4 = O\left( \frac{1}{\log(n)} \right) = o(1)$ for sT-LLSQ, which is much lower compared to sLLSQ because $\log(n) = o\left( n^2 \right)$.
	
	However, if only the zeroth/first/second-order asymptotics is given (i.e., $p=0,1,2$), then the sLLSQ convergence is \emph{not} provable. According to Table~\ref{tbl:Fundamental_results_for_common_asymptotic_forms_sLLSQ_continued}, for AF-1 with $p = 2$ and $g_l(n) = \frac{1}{n^l}$ for all $l \in \{1,2,3\}$, we obtain
	\begin{align*}
		\norm{\mathbf{y}_n^* - \boldsymbol{\gamma}} & = O\left( \log(n) \max\left( \tfrac{1}{n}, n^3 \abs{g_3(n)} \right) \right) =  O\left( \log(n) \right)  ,  \\
		y_{1,n}^* - \gamma_1 & = O\left( \log(n) \max\left( \tfrac{1}{n}, n^3 \abs{g_3(n)} \right) \right) =  O\left( \log(n) \right)  , \\
		y_{2,n}^* - \gamma_2 & = O\left( \max\left( \tfrac{1}{n}, n^3 \abs{g_3(n)} \right) \right) = O\left( 1 \right) , \\
		y_{3,n}^* - \gamma_3 & = O\left( \log(n) \max\left( \tfrac{1}{n^2}, n^2 \abs{g_3(n)} \right) \right) = O\left( \frac{\log(n)}{n} \right) = o(1)  ,  \\
		y_{4,n}^* - \gamma_4 & = O\left( \max\left( \tfrac{1}{n^2}, n^2 \abs{g_3(n)} \right) \right) = O\left( \frac{1}{n} \right) = o(1) .
	\end{align*}
	Furthermore, for AF-1 with $p = 1$ and $g_l(n) = \frac{1}{n^l}$ for all $l \in \{1,2\}$, Table~\ref{tbl:Fundamental_results_for_common_asymptotic_forms_sLLSQ} yields
	\begin{align*}
		\norm{\mathbf{y}_n^* - \boldsymbol{\gamma}} & = O\left( \log(n) \max\left( \tfrac{1}{n}, n^3 \abs{g_2(n)} \right) \right) = O\left( n \log(n) \right) ,  \\
		y_{1,n}^* - \gamma_1 & = O\left( \log(n) \max\left( \tfrac{1}{n}, n^3 \abs{g_2(n)} \right) \right) = O\left( n \log(n) \right)  , \\
		y_{2,n}^* - \gamma_2 & = O\left( \max\left( \tfrac{1}{n}, n^3 \abs{g_2(n)} \right) \right) = O\left( n \right) , \\
		y_{3,n}^* - \gamma_3 & = O\left( \log(n) \max\left( \tfrac{1}{n^2}, n^2 \abs{g_2(n)} \right) \right) = O\left( \log(n) \right)  ,  \\
		y_{4,n}^* - \gamma_4 & = O\left( \max\left( \tfrac{1}{n^2}, n^2 \abs{g_2(n)} \right) \right) = O\left( 1 \right) .
	\end{align*}
	Finally, for AF-1 with $p = 0$ and $g_1(n) = \frac{1}{n}$, we have from Table~\ref{tbl:Fundamental_results_for_common_asymptotic_forms_sLLSQ}
	\begin{align*}
		\norm{\mathbf{y}_n^* - \boldsymbol{\gamma}} & = O\left( n^3 \log(n) g_1(n) \right) = O\left( n^2 \log(n) \right) ,  \\
		y_{1,n}^* - \gamma_1 & = O\left( n^3 \log(n) g_1(n) \right) = O\left( n^2 \log(n) \right)  , \\
		y_{2,n}^* - \gamma_2 & = O\left( n^3 g_1(n) \right) = O\left( n^2 \right) , \\
		y_{3,n}^* - \gamma_3 & = O\left( n^2 \log(n) g_1(n) \right) = O\left( n \log(n) \right)  ,  \\
		y_{4,n}^* - \gamma_4 & = O\left( n^2 g_1(n) \right) = O\left( n \right) .
	\end{align*}
\end{remark}

\begin{remark} \label{rem:AF-2_sLLSQ_convergence}
	For the asymptotic form AF-2 in Table~\ref{tbl:Common_asymptotic_forms} with $p = 2$ and $g_l(n) = \frac{1}{n^l}$ for all $l \in \{1,2,3\}$, i.e., when 
	\begin{equation*}
		f(n) = \alpha_1 n^{\alpha_2} \alpha_3^n \left( 1 + \sum_{l=1}^2 \frac{\beta_l}{n^l} + O\left( \frac{1}{n^3} \right) \right)   , 
	\end{equation*}
	we obtain the following convergence rates of sLLSQ with $m=k=3$ (see Table~\ref{tbl:Fundamental_results_for_common_asymptotic_forms_sLLSQ_continued}):
	\begin{align*}
		\norm{\mathbf{y}_n^* - \boldsymbol{\gamma}} & = O\left( \log(n) \max\left( \frac{1}{n}, n^2 \abs{g_3(n)} \right) \right) = O\left( \frac{\log(n)}{n} \right) = o(1) , \\
		y_{1,n}^* - \gamma_1 & = O\left( \log(n) \max\left( \frac{1}{n}, n^2 \abs{g_3(n)} \right) \right) = O\left( \frac{\log(n)}{n} \right) = o(1) ,  \\
		y_{2,n}^* - \gamma_2  & = O\left( \max\left( \tfrac{1}{n}, n^2 \abs{g_3(n)} \right) \right) = O\left( \frac{1}{n} \right) = o(1)  , \\
		y_{3,n}^* - \gamma_3 & = O\left( \max\left( \tfrac{1}{n^2}, n \abs{g_3(n)} \right) \right) = O\left( \frac{1}{n^2} \right) = o(1) .
	\end{align*}
	In other words, the sLLSQ method is \emph{provably convergent} in such a case, with given second-order (or, higher-order) asymptotics. Table~\ref{tbl:Fundamental_results_for_common_asymptotic_forms_sT-LLSQ} yields the convergence rate $y_{3,n}^* - \gamma_3 = O\left( \frac{\log(n)}{n} \right) = o(1)$ for sT-LLSQ, which is much lower than sLLSQ since $\frac{n}{\log(n)} = o\left( n^2 \right)$. 
	
	On the other hand, if only the zeroth/first-order asymptotics is given (i.e., $p=0,1$), then the sLLSQ convergence is \emph{not} provable. According to Table~\ref{tbl:Fundamental_results_for_common_asymptotic_forms_sLLSQ}, for AF-2 with $p = 1$, $g_1(n) = \frac{1}{n}$ and $g_2(n) = \frac{1}{n^2}$, we obtain  
	\begin{align*}
		\norm{\mathbf{y}_n^* - \boldsymbol{\gamma}} & = O\left( \log(n) \max\left( \frac{1}{n}, n^2 \abs{g_2(n)} \right) \right) =  O\left( \log(n) \right) , \\
		y_{1,n}^* - \gamma_1 & = O\left( \log(n) \max\left( \frac{1}{n}, n^2 \abs{g_2(n)} \right) \right) =  O\left( \log(n) \right) ,  \\
		y_{2,n}^* - \gamma_2 & = O\left( \max\left( \tfrac{1}{n}, n^2 \abs{g_2(n)} \right) \right) = O\left( 1 \right)  , \\
		y_{3,n}^* - \gamma_3 & = O\left( \max\left( \tfrac{1}{n^2}, n \abs{g_2(n)} \right) \right) = O\left( \frac{1}{n} \right) = o(1) .
	\end{align*}
	Moreover, for AF-2 with $p = 0$ and $g_1(n) = \frac{1}{n}$, we get from Table~\ref{tbl:Fundamental_results_for_common_asymptotic_forms_sLLSQ}
	\begin{align*}
		\norm{\mathbf{y}_n^* - \boldsymbol{\gamma}} & = O\left( n^2 \log(n) g_1(n) \right) = O\left( n \log(n) \right) , \\
		y_{1,n}^* - \gamma_1 & = O\left( n^2 \log(n) g_1(n) \right) = O\left( n \log(n) \right) ,  \\
		y_{2,n}^* - \gamma_2 & = O\left( n^2 g_1(n) \right) = O\left( n \right)  , \\
		y_{3,n}^* - \gamma_3 & = O\left( n g_1(n) \right) = O\left( 1 \right) .
	\end{align*}
\end{remark}

\begin{remark} 
	For the asymptotic form AF-6 in Table~\ref{tbl:Common_asymptotic_forms} with $p = 1$ and $g_l(n) = \frac{1}{n^l}$ for all $l \in \{1,2\}$, i.e., when 
	\begin{equation*}
		f(n) = \alpha_1 \alpha_2^n \left( 1 + \frac{\beta_1}{n} + O\left( \frac{1}{n^2} \right) \right)   , 
	\end{equation*}
	we obtain the following convergence rates of sLLSQ with $m=k=2$ (see Table~\ref{tbl:Fundamental_results_for_common_asymptotic_forms_sLLSQ}):
	\begin{align*}
		\norm{\mathbf{y}_n^* - \boldsymbol{\gamma}} & = O\left( \max\left(\frac{1}{n}, n \abs{g_2(n)} \right) \right) = O\left( \frac{1}{n} \right) = o(1) , \\
		y_{1,n}^* - \gamma_1 & = O\left( \max\left(\frac{1}{n}, n \abs{g_2(n)} \right) \right) = O\left( \frac{1}{n} \right) = o(1) ,  \\
		y_{2,n}^* - \gamma_2  & = O\left( \max\left(\frac{1}{n^2}, \abs{g_2(n)} \right) \right) = O\left( \frac{1}{n^2} \right) = o(1)  .
	\end{align*}
	In other words, the sLLSQ method is \emph{provably convergent} in such a case, with given first-order (or, higher-order) asymptotics. Table~\ref{tbl:Fundamental_results_for_common_asymptotic_forms_sT-LLSQ} yields the convergence rate $y_{2,n}^* - \gamma_2 = O\left( \frac{1}{n} \right) = o(1)$ for sT-LLSQ, which is much lower than sLLSQ since $n = o\left( n^2 \right)$.
	
	Nevertheless, if only the zeroth-order asymptotics is given (i.e., $p=0$), then the sLLSQ convergence is \emph{not} provable. According to Table~\ref{tbl:Fundamental_results_for_common_asymptotic_forms_sLLSQ}, for AF-6 with $p = 0$ and $g_1(n) = \frac{1}{n}$, we obtain
	\begin{align*}
		\norm{\mathbf{y}_n^* - \boldsymbol{\gamma}} & = O\left( n g_1(n) \right) = O\left( 1 \right) , \\
		y_{1,n}^* - \gamma_1 & = O\left( n g_1(n) \right) = O\left( 1 \right) ,  \\
		y_{2,n}^* - \gamma_2  & = O\left( g_1(n) \right) = O\left( \frac{1}{n} \right) = o(1)  .
	\end{align*}	
\end{remark}

\section{Comparison with the Ratio Method and its Variants}
\label{sec:Ratio_method}

For simplicity of presentation and ease of comparison, we will thoroughly examine the \emph{Ratio Method (RM)} by considering only $0$-order asymptotics (i.e., $p = 0$).\footnote{The general case with $p$-order asymptotics, $p \in \NN$, can be treated in a similar way, although cumbersome since it requires appropriate truncation of the complete asymptotic series: $(1+x)^{\theta} \simeq \sum_{\nu = 0}^{\infty} {\binom{\theta}{\nu} x^{\nu}}$ as $x \to 0$. Specifically, it holds that $(1+x)^{\theta} = \sum_{\nu = 0}^m {\binom{\theta}{\nu} x^{\nu}} + O(x^{m+1})$, as $x \to 0$, for all $\theta \in \RR$ and $m \in \NN_0$, where $\binom{\theta}{\nu} \defeq \frac{\prod_{i=0}^{\nu-1} (\theta-i)}{\nu !}$ is the generalized binomial coefficient defined for all $\theta \in \RR$ and $\nu \in \NN_0$ with $\binom{\theta}{0} \defeq 1$.} The following result summarizes its asymptotic properties. 

\begin{proposition}[The ratio method] 
	\label{prop:Ratio_method_asymptotic_estimates}
	Suppose that 
	\begin{equation*}
		f(n) = {\alpha_1} n^{\alpha_2} {\alpha_3^n} \left( 1 + O(g_1(n)) \right) , \quad \mathrm{as}\ n \to \infty ,
	\end{equation*}
	where $\alpha_1,\alpha_3 \in \RR_{>0}$, $\alpha_2 \in \RR$, and $g_1(n) = o(1)$ with $g_1(n \pm 1) = O(g_1(n))$ as $n \to \infty$.\footnote{The uniqueness of parameters $\alpha_1$, $\alpha_2$, and $\alpha_3$ is guaranteed by Proposition~\ref{prop:Uniqueness_of_alpha_and_gamma}.} Let us consider the following sequences 
	\begin{align*}
		r_n & \defeq \frac{f(n+1)}{f(n)} , \\
		\zeta_n & \defeq n r_n - (n-1) r_{n-1} , \\
		\kappa_n & \defeq n^2 \left( 1 - \frac{r_n}{r_{n-1}} \right) , \\
		\zeta'_n & \defeq \frac{n r_n}{n + \alpha_2} ,  \\
		\kappa'_n & \defeq n \left( \frac{r_n}{\alpha_3} - 1 \right) .
	\end{align*}
	Note that $\zeta'_n$ (resp. $\kappa'_n$) can be defined only if the parameter $\alpha_2$ (resp. $\alpha_3$) is explicitly known. Then, we have the asymptotic estimates, as $n \to \infty$,
	\begin{align*}
		r_n - \alpha_3 & = O\left( \max\left( \frac{1}{n}, \abs{g_1(n)} \right) \right) = o(1) , \\
		\zeta_n - \alpha_3 & = O\left( \max\left( \frac{1}{n^2}, n \abs{g_1(n)} \right) \right) , \\
		\kappa_n - \alpha_2 & = O\left( \max\left( \frac{1}{n^2}, n^2 \abs{g_1(n)} \right) \right) , \\
		\zeta'_n - \alpha_3 & = O\left( \max\left( \frac{1}{n^2}, \abs{g_1(n)} \right) \right) = o(1) , \\
		\kappa'_n - \alpha_2 & = O\left( \max\left( \frac{1}{n}, n \abs{g_1(n)} \right) \right) . 
	\end{align*}
\end{proposition}

\begin{proof}
	The ratio $r_n$ yields
	\begin{align*}
		r_n \defeq \frac{f(n+1)}{f(n)} & = {\alpha_3} \left( 1 + \frac{1}{n} \right)^{\alpha_2} \frac{1 + O(g_1(n+1))}{1 + O(g_1(n))}   \\
		& = {\alpha_3} \left( 1 + \frac{1}{n} \right)^{\alpha_2} \left( 1 + O(g_1(n)) \right) \left( 1 + O(g_1(n)) \right)    \\
		& = {\alpha_3} \left( 1 + \frac{1}{n} \right)^{\alpha_2} \left( 1 + O(g_1(n)) \right)     \\
		& = {\alpha_3} \left( 1 + O\left( \frac{1}{n} \right) \right) \left( 1 + O(g_1(n)) \right)  \\
		& = {\alpha_3} \left( 1 + O\left( \frac{1}{n} \right) + O(g_1(n)) \right)     \\
		& = {\alpha_3} \left( 1 + O\left( \max\left( \frac{1}{n}, \abs{g_1(n)} \right) \right) \right)   \\
		& = {\alpha_3} + O\left( \max\left( \frac{1}{n}, \abs{g_1(n)} \right) \right)   ,
	\end{align*}
	since $(1+x)^{-1} = 1 + O(x)$, as $x \to 0$. 
	
	Moreover, if we use the second-order asymptotic expansion of $\left( 1 + \frac{1}{n} \right)^{\alpha_2} = 1 + \frac{\alpha_2}{n} + \frac{\lambda}{n^2} + O\left( \frac{1}{n^3} \right)$, where $\lambda \defeq \binom{\alpha_2}{2} = \frac{\alpha_2 (\alpha_2-1)}{2}$, then $r_n$ can also be expressed as  
	\begin{align*}
		r_n & = {\alpha_3} \left( 1 + \frac{\alpha_2}{n} + \frac{\lambda}{n^2} + O\left( \frac{1}{n^3} \right) \right) \left( 1 + O(g_1(n)) \right)  \\
		& = {\alpha_3} \left( 1 + \frac{\alpha_2}{n} + \frac{\lambda}{n^2} + O\left( \max\left( \frac{1}{n^3}, \abs{g_1(n)} \right) \right) \right)   .
	\end{align*}
	We can now construct the sequence $\{\zeta_n\}_{n \in \NN}$ via linear extrapolation, i.e.,
	\begin{align*}
		\zeta_n \defeq n r_n - (n-1) r_{n-1} & = {\alpha_3} \left( 1 - \frac{\lambda}{n(n-1)} + O\left( \max\left( \frac{1}{n^2}, n \abs{g_1(n)} \right) \right) \right.   \\
		& \qquad \quad + \left. O\left( \max\left( \frac{1}{(n-1)^2}, (n-1) \abs{g_1(n-1)} \right) \right) \right)   \\
		& = {\alpha_3} \left( 1 + O\left( \frac{1}{n^2} \right) + O\left( \max\left( \frac{1}{n^2}, n \abs{g_1(n)} \right) \right) \right)     \\
		& = {\alpha_3} +  O\left( \max\left( \frac{1}{n^2}, n \abs{g_1(n)} \right) \right)    . 
	\end{align*}
	
	In addition, by using the asymptotic expansion $r_n = {\alpha_3} \left( \frac{n+1}{n} \right)^{\alpha_2} \left( 1 + O(g_1(n)) \right)$, we can define a new sequence 
	\begin{align*}
		\rho_n \defeq \frac{r_n}{r_{n-1}} & = \left( \frac{n+1}{n} \frac{n-1}{n} \right)^{\alpha_2} \frac{1 + O(g_1(n))}{1 + O(g_1(n-1))} \\
		& = \left( 1 - \frac{1}{n^2} \right)^{\alpha_2} \left( 1 + O(g_1(n)) \right)    \\
		& = \left( 1 - \frac{\alpha_2}{n^2} + O\left( \frac{1}{n^4} \right) \right) \left( 1 + O(g_1(n)) \right)   \\
		& = 1 - \frac{\alpha_2}{n^2} + O\left( \frac{1}{n^4} \right) + O(g_1(n))   \\
		& = 1 - \frac{\alpha_2}{n^2} + O\left( \max\left( \frac{1}{n^4}, \abs{g_1(n)} \right) \right)   . 
	\end{align*}
	Therefore, we obtain  
	\begin{equation*}
		\kappa_n \defeq n^2 (1 - \rho_n) = \alpha_2 + O\left( \max\left( \frac{1}{n^2}, n^2 \abs{g_1(n)} \right) \right)  . 
	\end{equation*}
	
	Furthermore, if $\alpha_2$ is explicitly known, then by using the first-order asymptotic expansion of $\left( 1 + \frac{1}{n} \right)^{\alpha_2} = 1 + \frac{\alpha_2}{n} + O\left( \frac{1}{n^2} \right)$ we have  
	\begin{align*}
		\zeta'_n \defeq \frac{n r_n}{n + \alpha_2} & = \frac{n}{n + \alpha_2} {\alpha_3} \left( 1 + \frac{\alpha_2}{n} + O\left( \max\left( \frac{1}{n^2}, \abs{g_1(n)} \right) \right) \right)   \\
		& = {\alpha_3} + \frac{n}{n + \alpha_2} O\left( \max\left( \frac{1}{n^2}, \abs{g_1(n)} \right) \right)   \\
		& = {\alpha_3} + O\left( \max\left( \frac{1}{n^2}, \abs{g_1(n)} \right) \right)    ,
	\end{align*}
	where the last equality follows from the fact that $\frac{n}{n + \alpha_2} = \Theta(1) = O(1)$. 
		
	Finally, if $\alpha_3$ is precisely known, then 
	\begin{align*}
		\kappa'_n \defeq n \left( \frac{r_n}{\alpha_3} - 1 \right) & = n \left(\frac{\alpha_2}{n} + O\left( \max\left( \frac{1}{n^2}, \abs{g_1(n)} \right) \right) \right)   \\
		& = {\alpha_2} + O\left( \max\left( \frac{1}{n}, n \abs{g_1(n)} \right) \right)  .
	\end{align*}
\end{proof}

\begin{remark}
	The function $g_1$ in the asymptotic expansion is almost always \emph{overlooked} in the existing literature. Nevertheless, this term is \emph{crucial} for establishing the convergence rate of numerical methods. In addition, it may conceal important information about $f$ (e.g., \emph{oscillating behavior}) that is responsible for the \emph{divergence} of a particular technique, thus hindering the computation of unknown parameters. See the \emph{counterintuitive} numerical example in Section~\ref{subsec:Counterintuitive_example}. 
\end{remark}

Table~\ref{tbl:Comparison_between_computational_methods} compares the asymptotic estimates achieved by RM, sLLSQ and sT-LLSQ. Firstly, for AF-2, we observe that $x_{3,n}^*$ (resp. $x_{2,n}^*$) of sLLSQ attains an asymptotic estimate which is at least as good as that of $\zeta_n$ (resp. $\kappa_n$), since $n \abs{g_1(n)} \leq \max\left( \frac{1}{n^2}, n \abs{g_1(n)} \right)$ (resp. $n^2 \abs{g_1(n)} \leq \max\left( \frac{1}{n^2}, n^2 \abs{g_1(n)} \right)$). The convergence rate of sT-LLSQ is complementary to that of $r_n$: if $g_1(n) = O\left( \frac{\log(n)}{n} \right)$ then $\max\left( \frac{1}{n}, \abs{g_1(n)} \right) = O\left( \max\left( \frac{1}{n}, \frac{\log(n)}{n} \right) \right) = O\left( \frac{\log(n)}{n} \right)$, whereas if $g_1(n) = \Omega\left( \frac{\log(n)}{n} \right)$ then $\max\left( \frac{1}{n}, \abs{g_1(n)} \right) \geq \abs{g_1(n)} = \Omega\left( \frac{\log(n)}{n} \right)$. Secondly, for AF-2 with known $\alpha_2 \in \RR$, we notice that $\zeta'_n$ provides an improvement compared to both $r_n$ and $\zeta_n$. Again, sLLSQ gives a better convergence rate than RM in terms of $\alpha_3$, because $\abs{g_1(n)} \leq \max\left( \frac{1}{n^2}, \abs{g_1(n)} \right)$, while sT-LLSQ complements RM: if $g_1(n) = O\left( \frac{1}{n} \right)$ then $\max\left( \frac{1}{n^2}, \abs{g_1(n)} \right) = O\left( \max\left( \frac{1}{n^2}, \frac{1}{n} \right) \right) = O\left( \frac{1}{n} \right)$, and if $g_1(n) = \Omega\left( \frac{1}{n} \right)$ then $\max\left( \frac{1}{n^2}, \abs{g_1(n)} \right) \geq \abs{g_1(n)} = \Omega\left( \frac{1}{n} \right)$. Thirdly, for AF-2 with known $\alpha_3 \in \RR_{>0}$, we observe that $\kappa'_n$ is not necessarily an improvement over $\kappa_n$. Now, regarding the unknown parameter $\alpha_2$, sLLSQ gives a better asymptotic estimate since $n \abs{g_1(n)} \leq \max\left( \frac{1}{n}, n \abs{g_1(n)} \right)$. The complementarity of sT-LLSQ and RM is also obvious: if $g_1(n) = O\left( \frac{1}{n \log(n)} \right)$ then $\max\left( \frac{1}{n}, n \abs{g_1(n)} \right) = O\left( \max\left( \frac{1}{n}, \frac{1}{\log(n)} \right) \right) = O\left( \frac{1}{\log(n)} \right)$, while if $g_1(n) = \Omega\left( \frac{1}{n \log(n)} \right)$ then $\max\left( \frac{1}{n}, n \abs{g_1(n)} \right) \geq n \abs{g_1(n)} = \Omega\left( \frac{1}{\log(n)} \right)$. Note that, in contrast to $x_{2,n}^*$ of sT-LLSQ, $\kappa'_n$ is not always guaranteed to converge to $\alpha_2$, e.g., when $g_1(n) = \frac{1}{\sqrt{n}}$ we just have $\kappa'_n - \alpha_2 = O\left( \max\left( \frac{1}{n}, n \abs{g_1(n)} \right) \right) = O(\sqrt{n})$.

\renewcommand{\arraystretch}{1.5}
\begin{table}[t!]
	\centering
	\caption{Comparison between the ratio method and the proposed techniques (sLLSQ, sT-LLSQ). For simplicity, we assume that \linebreak $p=0$ and $g_1(n) = o(1)$, with $g_1(n+i) = O(g_1(n))$ for all integers $i \geq -1$ (for the ratio method: $g_1(n \pm 1) = O(g_1(n))$, whereas for sLLSQ: $g_1(n+i) = O(g_1(n))$ for all $i \in \NN_0$; see Definition~\ref{def:Asymptotic_regularity} and Remark~\ref{rem:Asymptotic_regularity}). The asymptotic forms AF-2, AF-6 and AF-10 are given in Table~\ref{tbl:Common_asymptotic_forms}. The sequences $r_n$, $\zeta_n$, $\kappa_n$, $\zeta'_n$ and $\kappa'_n$ for the ratio method are defined in Proposition~\ref{prop:Ratio_method_asymptotic_estimates}. For sLLSQ, each asymptotic estimate $x_{j,n}^* - \alpha_j = O(\vartheta_j(n))$ is valid provided that $\vartheta_j(n) = o(1)$, according to~\eqref{eq:Implication_o_sLLSQ}.}
	\label{tbl:Comparison_between_computational_methods} 

\resizebox{\textwidth}{!}{%
\begin{tabular}{|l|l|l|l|}
\hline 
\makecell[c]{\textbf{Asymptotic form}} & \makecell[c]{\textbf{Ratio method}} & \makecell[c]{\textbf{sLLSQ ($m=k$)}} & \makecell[c]{\textbf{sT-LLSQ ($m=k$, $\mu \in \RR_{>0}$)}}  \\
\hline

\makecell[l]{AF-2} & \makecell[l]{$r_n - \alpha_3 = O\left( \max\left( \frac{1}{n}, \abs{g_1(n)} \right) \right) = o(1)$, \\ $\zeta_n - \alpha_3 = O\left( \max\left( \frac{1}{n^2}, n \abs{g_1(n)} \right) \right)$, \\ $\kappa_n - \alpha_2 = O\left( \max\left( \frac{1}{n^2}, n^2 \abs{g_1(n)} \right) \right)$.} & \makecell[l]{$x_{3,n}^* - \alpha_3 = O\left( n g_1(n) \right)$, \\ $x_{2,n}^* - \alpha_2 = O\left( n^2 g_1(n) \right)$, \\ $x_{1,n}^* - \alpha_1 = O\left( n^2 \log(n) g_1(n) \right)$.} & \makecell[l]{$x_{3,n}^* - \alpha_3 = O\left( \frac{\log(n)}{n} \right) = o(1)$}    \\
\hline  

\makecell[l]{AF-2 with \\ known $\alpha_2 \in \RR$} & \makecell[l]{$\zeta'_n - \alpha_3 = O\left( \max\left( \frac{1}{n^2}, \abs{g_1(n)} \right) \right) = o(1)$} & \makecell[l]{$x_{3,n}^* - \alpha_3 = O\left( g_1(n) \right) = o(1)$, \\ $x_{1,n}^* - \alpha_1 = O\left( n g_1(n) \right)$. \vspace{1mm} \\ (similar to AF-6 \\ with $f(n) \mapsto f(n)/{n^{\alpha_2}}$ \\ and $\widehat{f}(n;\boldsymbol{\alpha}) = \alpha_1 \alpha_3^n$)} & \makecell[l]{$x_{3,n}^* - \alpha_3 = O\left( \frac{1}{n} \right) = o(1)$ \vspace{1mm} \\ (similar to AF-6 \\ with $f(n) \mapsto f(n)/{n^{\alpha_2}}$ \\ and $\widehat{f}(n;\boldsymbol{\alpha}) = \alpha_1 \alpha_3^n$)}    \\
\hline

\makecell[l]{AF-2 with \\ known $\alpha_3 \in \RR_{>0}$} & \makecell[l]{$\kappa'_n - \alpha_2 = O\left( \max\left( \frac{1}{n}, n \abs{g_1(n)} \right) \right)$} & \makecell[l]{$x_{2,n}^* - \alpha_2 = O\left( n g_1(n) \right)$, \\ $x_{1,n}^* - \alpha_1 = O\left( n \log(n) g_1(n) \right)$. \vspace{1mm} \\ (similar to AF-10 \\ with $f(n) \mapsto f(n)/{\alpha_3^n}$ \\ and $\widehat{f}(n;\boldsymbol{\alpha}) = \alpha_1 n^{\alpha_2}$)} & \makecell[l]{$x_{2,n}^* - \alpha_2 = O\left( \frac{1}{\log(n)} \right) = o(1)$ \vspace{1mm} \\ (similar to AF-10 \\ with $f(n) \mapsto f(n)/{\alpha_3^n}$ \\ and $\widehat{f}(n;\boldsymbol{\alpha}) = \alpha_1 n^{\alpha_2}$)}  \\
\hline  
\end{tabular}%
} 
\end{table}

\begin{remark}
	A key advantage of sLLSQ, compared to RM, is the \emph{simultaneous computation} of all the unknown parameters (combined into a \emph{single vector}), without the need of constructing a separate sequence for each parameter; this is due to linear algebra. Another advantage of the proposed methods (both sLLSQ and sT-LLSQ) is their \emph{flexibility} to handle a broader class of asymptotic forms (besides AF-2, AF-6, and AF-10) together with a \emph{simplified and unified analysis of convergence}. Therefore, sLLSQ and sT-LLSQ are clearly \emph{complementary} to existing methods. 
\end{remark}

\begin{remark}
	Further comparison with other approaches (as mentioned in Section~\ref{subsec:Related_work}) is not followed herein, because their convergence analysis is missing from the existing literature and, in general, it is not straightforward. 	 
\end{remark}

\section{Numerical Examples}

In order to verify the theoretical results we consider enumeration sequences with \emph{proven} asymptotics and \emph{explicit}  parameters ($\boldsymbol{\gamma}$ or $\boldsymbol{\alpha}$), so that we are able to calculate the actual distance between $\mathbf{y}^*$ and $\boldsymbol{\gamma}$, or $\mathbf{x}^*$ and $\boldsymbol{\alpha}$ (i.e., through a norm of their difference). Nevertheless, remember that in ALT only the enumeration sequence and its asymptotic expansion are given, while the parameters in its asymptotic form are \emph{unknown} to us (see Table~\ref{tbl:Asymptotics_learning_theory_General_problem}).   

For comparison purposes we also include numerical results of RM, whenever applicable. In addition, for the sake of simplicity and legibility of figures, we just use $\mu = 1$ for the sT-LLSQ method. Different (fixed) values of $\mu \in \RR_{>0}$ do not explicitly affect the universal convergence rate $y_{k,n}^* - \gamma_k = O\left( \frac{\varphi_{k-1}(n)}{\varphi_k(n)} \right) = o(1)$, cf.~\eqref{eq:Universal_convergence_rate_for_index_k},~\eqref{eq:y_kn_star_asymptotic_estimate_sT-LLSQ}~and~\eqref{eq:x_kn_star_asymptotic_estimate_sT-LLSQ}, but only the hidden constant in the $O$-notation (i.e., the universal convergence rate is preserved up to a multiplicative constant).

\subsection{The Fibonacci Numbers}

The Fibonacci sequence is defined by the recurrence relation 
\begin{equation*}
	f(n) = f(n-1) + f(n-2) ,
\end{equation*}
for all integers $n \geq 2$ with initial conditions $f(0) = 0$ and $f(1) = 1$. This is a homogeneous second-order linear recurrence with constant coefficients, and therefore it admits a closed-form expression given by
\begin{equation*}
	f(n) = \frac{1}{\sqrt{5}} \left( \phi^n - \psi^n \right) = \frac{1}{\sqrt{5}} \phi^n \left( 1 - \left( \frac{\psi}{\phi} \right)^n \right) = \frac{1}{\sqrt{5}} \phi^n \left( 1 - \frac{(-1)^n}{\phi^{2n}} \right)  ,
\end{equation*}
for all $n \in \NN_0$, where $\phi = \frac{1+\sqrt{5}}{2} \approx 1.618$ is the golden ratio, and $\psi = \frac{1-\sqrt{5}}{2} \approx -0.618$. In particular, the constants $\phi$ and $\psi$ are the roots of its characteristic polynomial $x^2 - x - 1$, thus satisfying $\psi = 1-\phi = -\phi^{-1}$. Note that $\frac{(-1)^n}{\phi^{2n}} = \Theta\left( \frac{1}{\phi^{2n}} \right) = O\left( \frac{1}{\tau^n} \right) = o(1)$, as $n \to \infty$, for every $\tau \in \RR$ such that $1 < \tau \leq \phi^2$.

In ALT we want to compute the (unknown) parameters $\boldsymbol{\alpha} = \transp{[\alpha_1, \alpha_2]} = \transp{\left[ \frac{1}{\sqrt{5}}, \phi \right]}$, given the values $\{f(n)\}_{n \in \NN}$ and the asymptotic expansion   
\begin{equation*}
	f(n) = {\alpha_1} {\alpha_2^n} \left(1 + O(g_1(n)) \right)  , 
\end{equation*}
 where $g_1(n) \defeq \frac{1}{\tau^n} = o(1)$ for some (given) real constant $\tau \in (1,\phi^2]$. According to Tables~\ref{tbl:Fundamental_results_for_common_asymptotic_forms_sLLSQ}~and~\ref{tbl:Fundamental_results_for_common_asymptotic_forms_sT-LLSQ} for AF-6 with $p=0$, and the convergence-rate invariant~\eqref{eq:Convergence-rate_invariant_y_to_x_with_h}, the sLLSQ method achieves \emph{exponentially fast convergence}, i.e., 
\begin{align*}
	x_{1,n}^* - \alpha_1 & = O\left( n g_1(n) \right) = O\left( \frac{n}{\tau^n} \right) = o(1) , \\ 
	x_{2,n}^* - \alpha_2 & = O\left( g_1(n) \right) = O\left( \frac{1}{\tau^n} \right) = o(1)  ,
\end{align*}
whereas the sT-LLSQ technique exhibits a much slower (linear) convergence of $x_{2,n}^*$ with $x_{1,n}^*$ remaining bounded, i.e.,
\begin{equation*}
	x_{1,n}^* - \alpha_1 = O(1)  \quad \mathrm{and} \quad  x_{2,n}^* - \alpha_2 = O\left( \frac{1}{n} \right) = o(1) .
\end{equation*}
Based on Proposition~\ref{prop:Impact_of_the_sliding_window_length} and Theorem~\ref{thm:The_sT-LLSQ_method}, the above estimates hold for every $m \geq k = 2$. Moreover, RM gives the asymptotic estimate (due to the absence of a term $n^{\alpha}$ in the asymptotic form)
\begin{equation*}
	r_n - \alpha_2 = O\left( g_1(n) \right) = O\left( \frac{1}{\tau^n} \right) = o(1) ,
\end{equation*}
which coincides with that of sLLSQ.  

In Figure~\ref{fig:Fibonacci}, we present the numerical results for the Fibonacci numbers. Specifically, Figure~\ref{fig:Fibonacci_Error_x_star_sLLSQ_and_RM} shows the absolute errors of sLLSQ and RM that decay exponentially to zero; a few iterations are enough to achieve a very small error. In addition, Figure~\ref{fig:Fibonacci_Error_x_star_sTikhonov_LLSQ} is in agreement with the theoretical behavior of sT-LLSQ. Finally, Figure~\ref{fig:Fibonacci_Objective_values} verifies that the global minima of sLLSQ vanish asymptotically, i.e., $F_n^* = o(1)$, and those of sT-LLSQ remain bounded asymptotically, i.e., $G_n^* = O(1)$; see Theorems~\ref{thm:The_sLLSQ_method}~and~\ref{thm:The_sT-LLSQ_method}.

\begin{figure}[t!]
	\centering
	\begin{subfigure}[t]{0.5\textwidth}
		\centering
		\includegraphics[width=\textwidth]{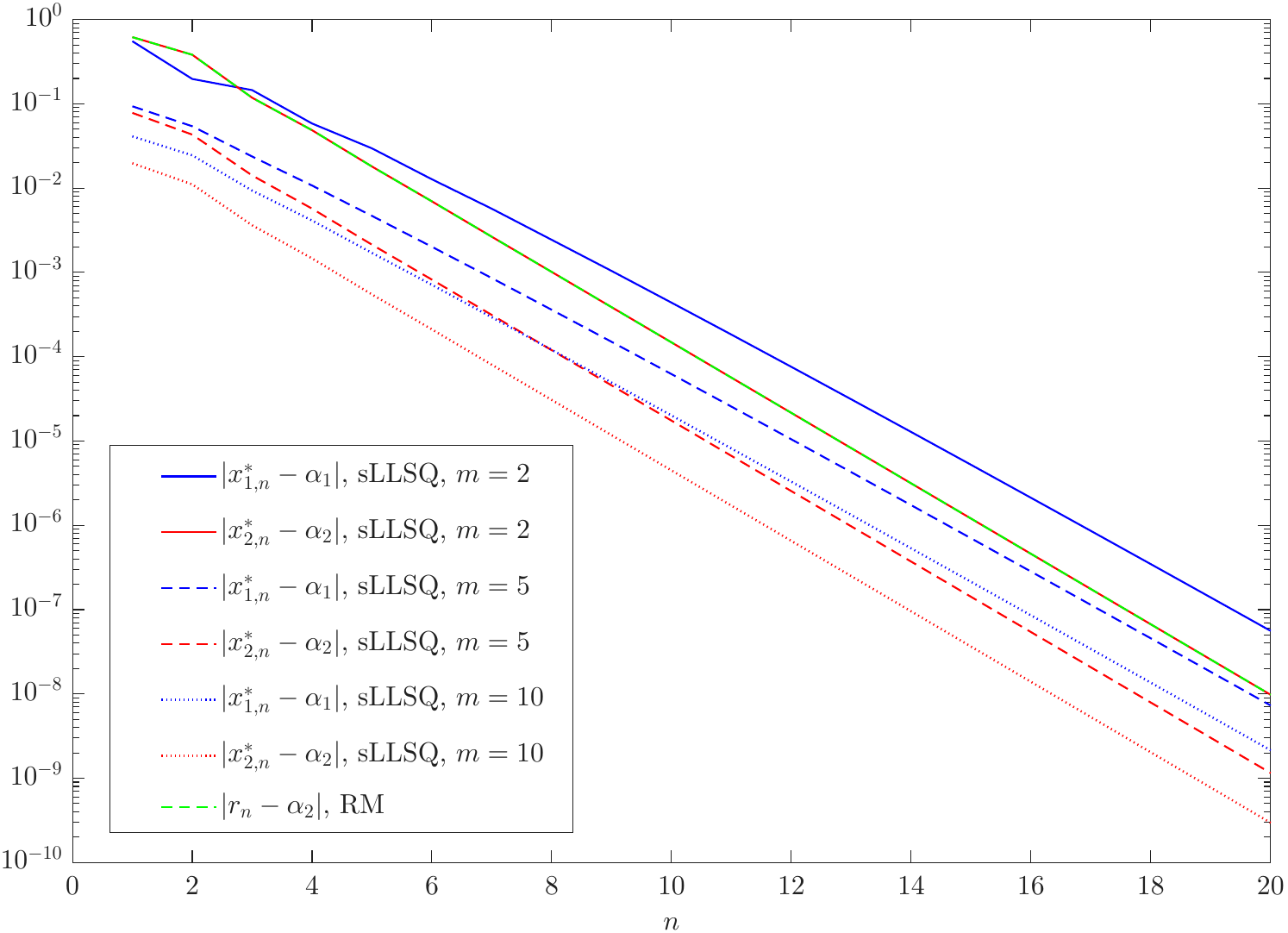}
		\caption{Absolute errors of sLLSQ and RM.}
		\label{fig:Fibonacci_Error_x_star_sLLSQ_and_RM}
	\end{subfigure}
	~
	\begin{subfigure}[t]{0.5\textwidth}
		\centering
		\includegraphics[width=\textwidth]{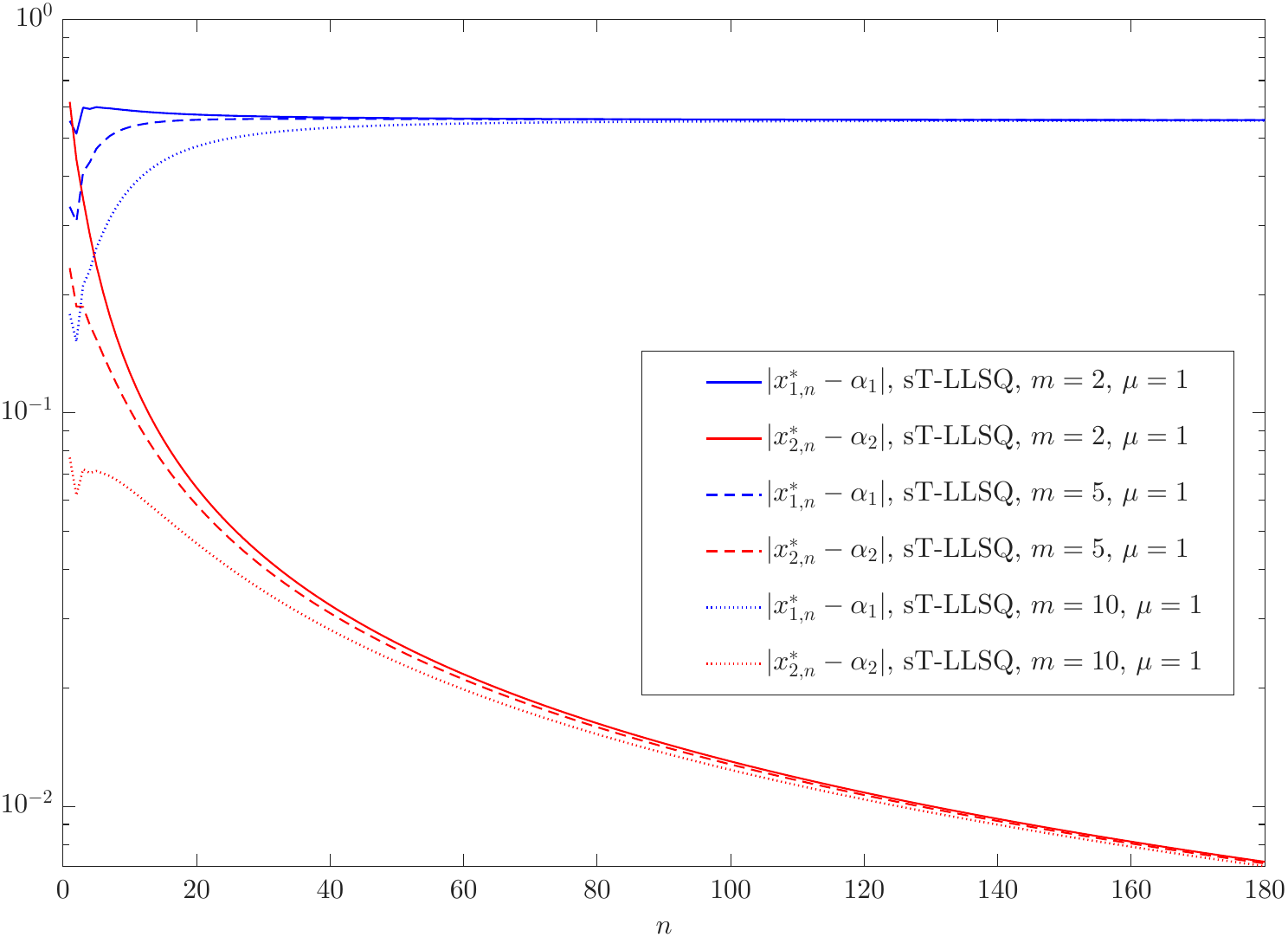}
		\caption{Absolute errors of sT-LLSQ.}
		\label{fig:Fibonacci_Error_x_star_sTikhonov_LLSQ}
	\end{subfigure}
	\\
	\begin{subfigure}[t]{0.5\textwidth}
		\centering
		\includegraphics[width=\textwidth]{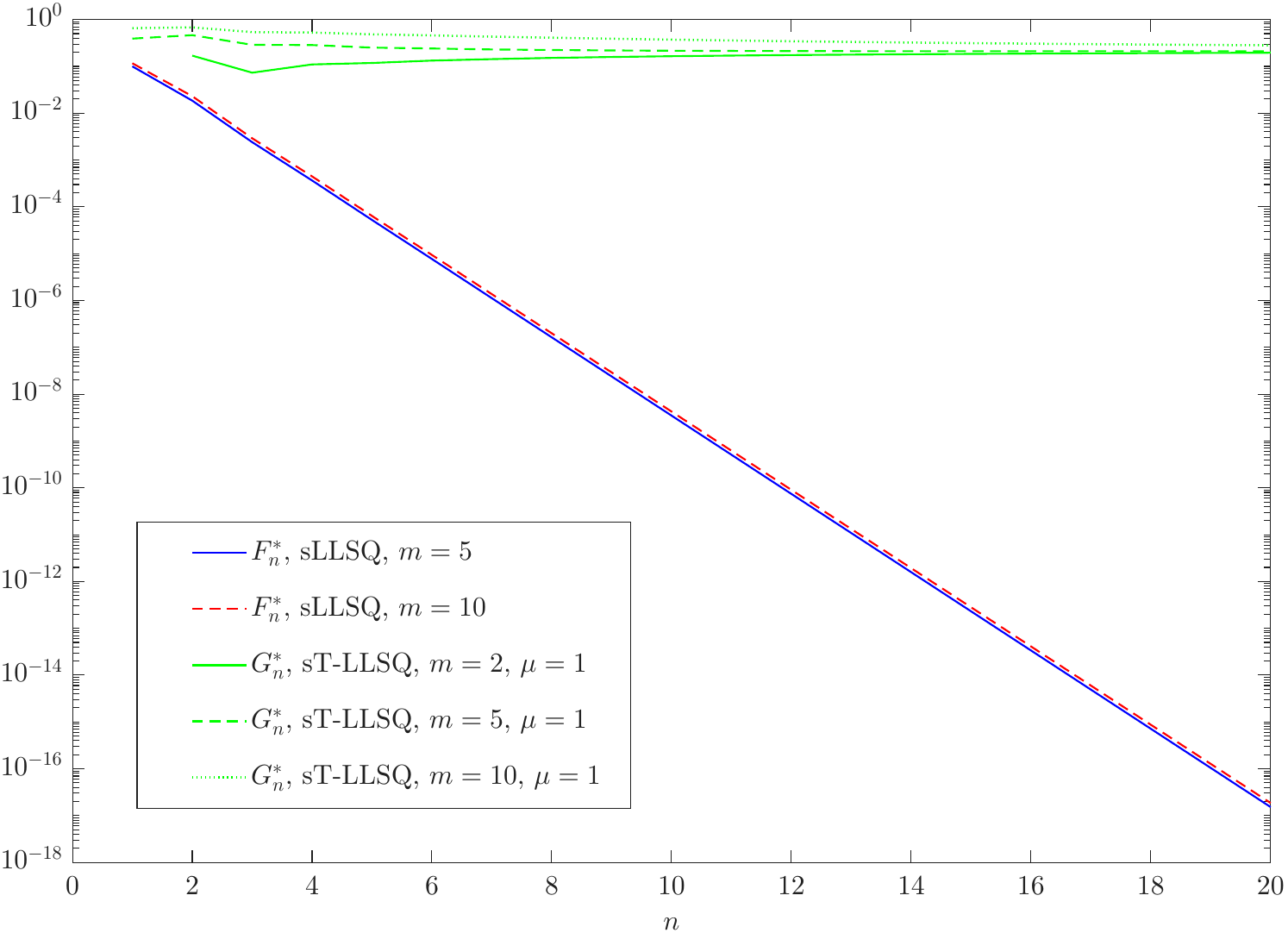}
		\caption{Global minima achieved by sLLSQ and sT-LLSQ. For sLLSQ with $m=k=2$, we have $F_n^* = 0$ for sufficiently large $n$ (this is not visible due to the logarithmic scale of the vertical axis).}
		\label{fig:Fibonacci_Objective_values}
	\end{subfigure}
	
	\caption{Application to the Fibonacci numbers ($k=2$).}
	\label{fig:Fibonacci}
\end{figure}

\subsection{The Catalan Numbers}

The $n$-th Catalan number is defined recursively by 
\begin{equation*}
	C_n = \frac{2(2n-1)}{n+1} C_{n-1} ,
\end{equation*}
for all $n \in \NN$ with initial condition $C_0 = 1$, or equivalently by the explicit formula
\begin{equation} \label{eq:Catalan_formula}
	C_n = \frac{1}{n+1} \binom{2n}{n} = \frac{(2n)!}{(n+1)! \, n!}  .
\end{equation}
By applying the de Moivre-Stirling approximation of the second order, 
\begin{equation*}
	n! = \sqrt{2 \pi n} \left( \frac{n}{e} \right)^n \left( 1 + \frac{1}{12n} + \frac{1}{288 n^2} + O\left( \frac{1}{n^3} \right) \right) ,
\end{equation*}
we find an asymptotic expansion for the central binomial coefficient as $n \to \infty$, i.e., 
\begin{align*}
	\binom{2n}{n} & = \frac{(2n)!}{(n!)^2} =
	\frac{2 \sqrt{\pi n} \left( \frac{2n}{e} \right)^{2n} \left( 1 + \frac{1}{24n} + \frac{1}{1152 n^2} + O\left( \frac{1}{n^3} \right) \right)}{2 \pi n \left( \frac{n}{e} \right)^{2n} \left( 1 + \frac{1}{12n} + \frac{1}{288 n^2} + O\left( \frac{1}{n^3} \right) \right)^2}  \\
	& = \frac{4^n}{\sqrt{\pi n}}  \left( 1 + \frac{1}{24n} + \frac{1}{1152 n^2} + O\left( \frac{1}{n^3} \right) \right) \cdot   \\
	& \quad\, \left( 1 - \frac{1}{6n} - \frac{1}{144 n^2} + O\left( \frac{1}{n^3} \right) + 3 \left( \frac{1}{12 n} + O\left( \frac{1}{n^2} \right) \right)^2  + O\left( \left( O\left( \frac{1}{n} \right) \right)^3 \right)  \right)  \\
	& = \frac{4^n}{\sqrt{\pi n}}  \left( 1 + \frac{1}{24n} + \frac{1}{1152 n^2} + O\left( \frac{1}{n^3} \right) \right)  \left( 1 - \frac{1}{6n} + \frac{1}{72 n^2} + O\left( \frac{1}{n^3} \right)   \right)   \\
	& = \frac{4^n}{\sqrt{\pi n}}  \left( 1 - \frac{1}{8n} + \frac{1}{128 n^2} + O\left( \frac{1}{n^3} \right) \right)   .
\end{align*}
Here, we have used the asymptotic expansion $(1+x)^{\theta} = 1 + \theta x + \binom{\theta}{2} x^2 + O(x^3)$ as $x \to 0$, which holds for all $\theta \in \RR$ (in particular, for $\theta = -2$). Hence, the $n$-th Catalan number has the following asymptotic expansion, as $n \to \infty$,
\begin{align*}
	C_n & = \frac{1}{n} \left( 1 + \frac{1}{n} \right)^{-1} \binom{2n}{n}   \\
	& = \frac{1}{n} \left( 1 - \frac{1}{n} + \frac{1}{n^2} + O\left( \frac{1}{n^3} \right) \right) \frac{4^n}{\sqrt{\pi n}} \left( 1 - \frac{1}{8n} + \frac{1}{128 n^2} + O\left( \frac{1}{n^3} \right) \right)   \\
	& = \frac{1}{\sqrt{\pi}} n^{-3/2} 4^n \left( 1 - \frac{9}{8n} + \frac{145}{128 n^2} + O\left( \frac{1}{n^3} \right) \right) .
\end{align*}
Observe that $C_n$ admits the asymptotic form AF-2 in Table~\ref{tbl:Common_asymptotic_forms} with $k=3$, $\boldsymbol{\alpha} = \transp{\left[ \frac{1}{\sqrt{\pi}}, -\frac{3}{2}, 4 \right]}$, $p = 2$ and $g_l(n) = \frac{1}{n^l}$, for all $l \in \{1,2,3\}$. As a result, based on Remark~\ref{rem:AF-2_sLLSQ_convergence}, the sLLSQ technique (with $m=k=3$) is provably convergent with 
\begin{align*}
	x_{1,n}^* - \alpha_1 & = O\left( \frac{\log(n)}{n} \right) = o(1) , \\
	x_{2,n}^* - \alpha_2  & = O\left( \frac{1}{n} \right) = o(1)  , \\
	x_{3,n}^* - \alpha_3 & = O\left( \frac{1}{n^2} \right) = o(1) ,
\end{align*}
where we have used the convergence-rate invariant~\eqref{eq:Convergence-rate_invariant_y_to_x_with_h}. Notice that $x_{1,n}^*$, $x_{2,n}^*$ and $x_{3,n}^*$ attain sub-linear, linear and quadratic convergence rates, respectively. 

According to Table~\ref{tbl:Fundamental_results_for_common_asymptotic_forms_sT-LLSQ} and the convergence-rate invariant~\eqref{eq:Convergence-rate_invariant_y_to_x_with_h}, the sT-LLSQ method provides the following estimates  
\begin{align*}
	x_{1,n}^* - \alpha_1 & = O\left( 1 \right)  , \\
	x_{2,n}^* - \alpha_2  & = O\left( 1 \right)  , \\
	x_{3,n}^* - \alpha_3 & = O\left( \frac{\log(n)}{n} \right) = o(1) .
\end{align*}
The convergence of $x_{3,n}^*$ is sub-linear, and therefore much slower than its sLLSQ counterpart. 

Furthermore, regarding RM, we can derive $\Theta$-asymptotics, as $n \to \infty$, by exploiting the exact formula \eqref{eq:Catalan_formula} of the Catalan numbers. Specifically,
\begin{align*}
	& r_n \defeq \frac{f(n+1)}{f(n)} = \frac{C_{n+1}}{C_n} = \frac{2(2n+1)}{n+2} = 4 - \frac{6}{n} + O\left( \frac{1}{n^2} \right)  \\
	\implies & r_n - \alpha_3 = r_n - 4 = \Theta\left( \frac{1}{n} \right) = o(1) ,  \\
	& \zeta_n \defeq n r_n - (n-1) r_{n-1} = 4 - \frac{12}{n^2} + O\left( \frac{1}{n^3} \right)  \\
	\implies & \zeta_n - \alpha_3 = \zeta_n - 4 = \Theta\left( \frac{1}{n^2} \right) = o(1)  ,  \\ 
	& \kappa_n \defeq n^2 \left( 1 - \frac{r_n}{r_{n-1}} \right) = -\frac{3}{2} + \frac{9}{4n} + O\left( \frac{1}{n^2} \right)  \\
	\implies & \kappa_n - \alpha_2 = \kappa_n - \left( -\frac{3}{2} \right) = \Theta\left( \frac{1}{n} \right) = o(1) .
\end{align*}
As a consequence, $\kappa_n$ and $\zeta_n$ have similar convergence rates with $x_{2,n}^*$ and $x_{3,n}^*$ of sLLSQ, respectively, whereas the convergence of $r_n$ is much slower compared to $\zeta_n$ and $x_{3,n}^*$ of sLLSQ.

Numerical results concerning the application of ALT to Catalan numbers are shown in Figure~\ref{fig:Catalan}. We can observe that numerical experiments confirm all the theoretical findings.

\begin{figure}[t!]
	\centering
	\begin{subfigure}[t]{0.5\textwidth}
		\centering
		\includegraphics[width=\textwidth]{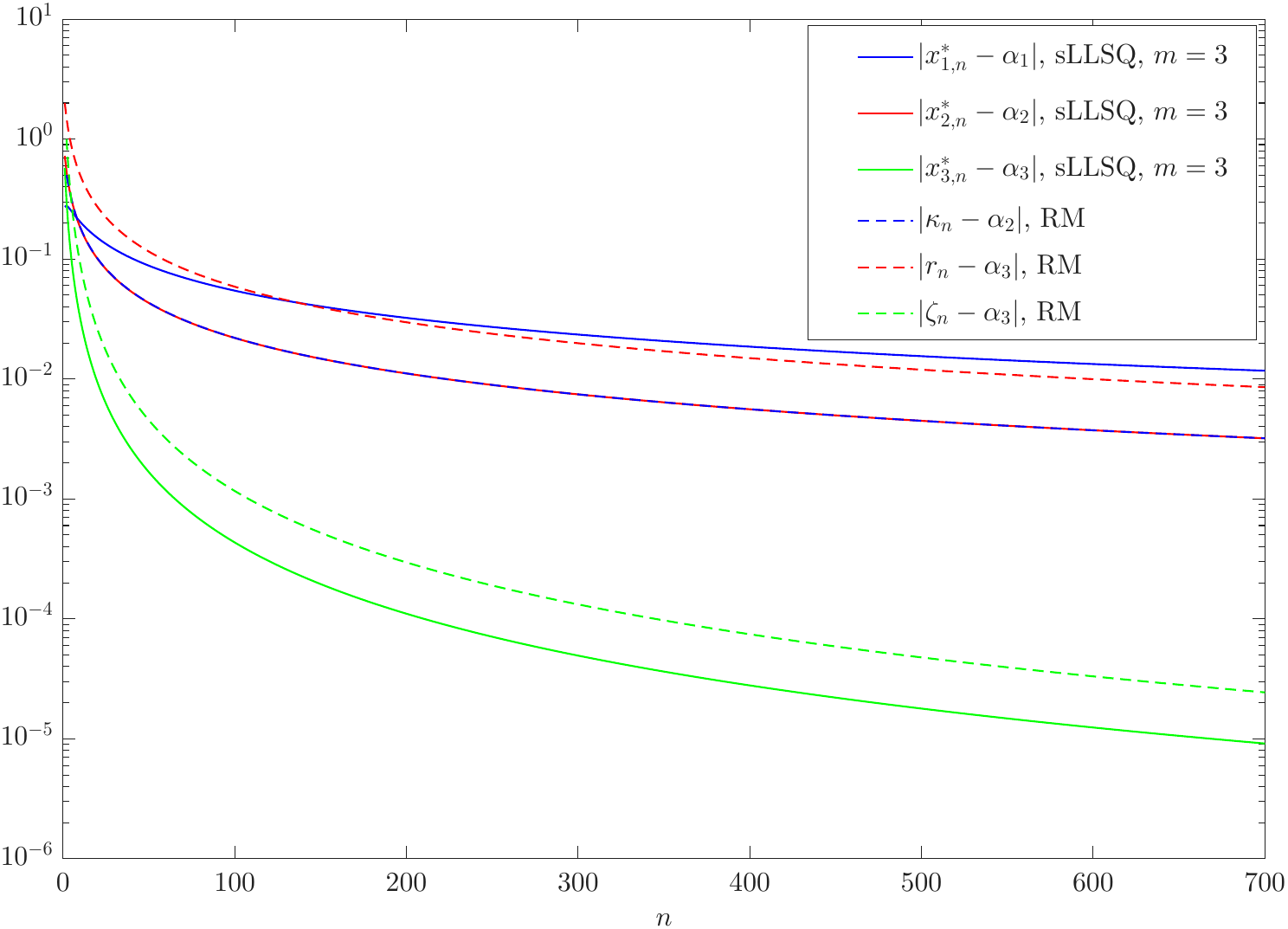}
		\caption{Absolute errors of sLLSQ and RM.}
		\label{fig:Catalan_Error_x_star_sLLSQ_and_RM}
	\end{subfigure}
	~
	\begin{subfigure}[t]{0.5\textwidth}
		\centering
		\includegraphics[width=\textwidth]{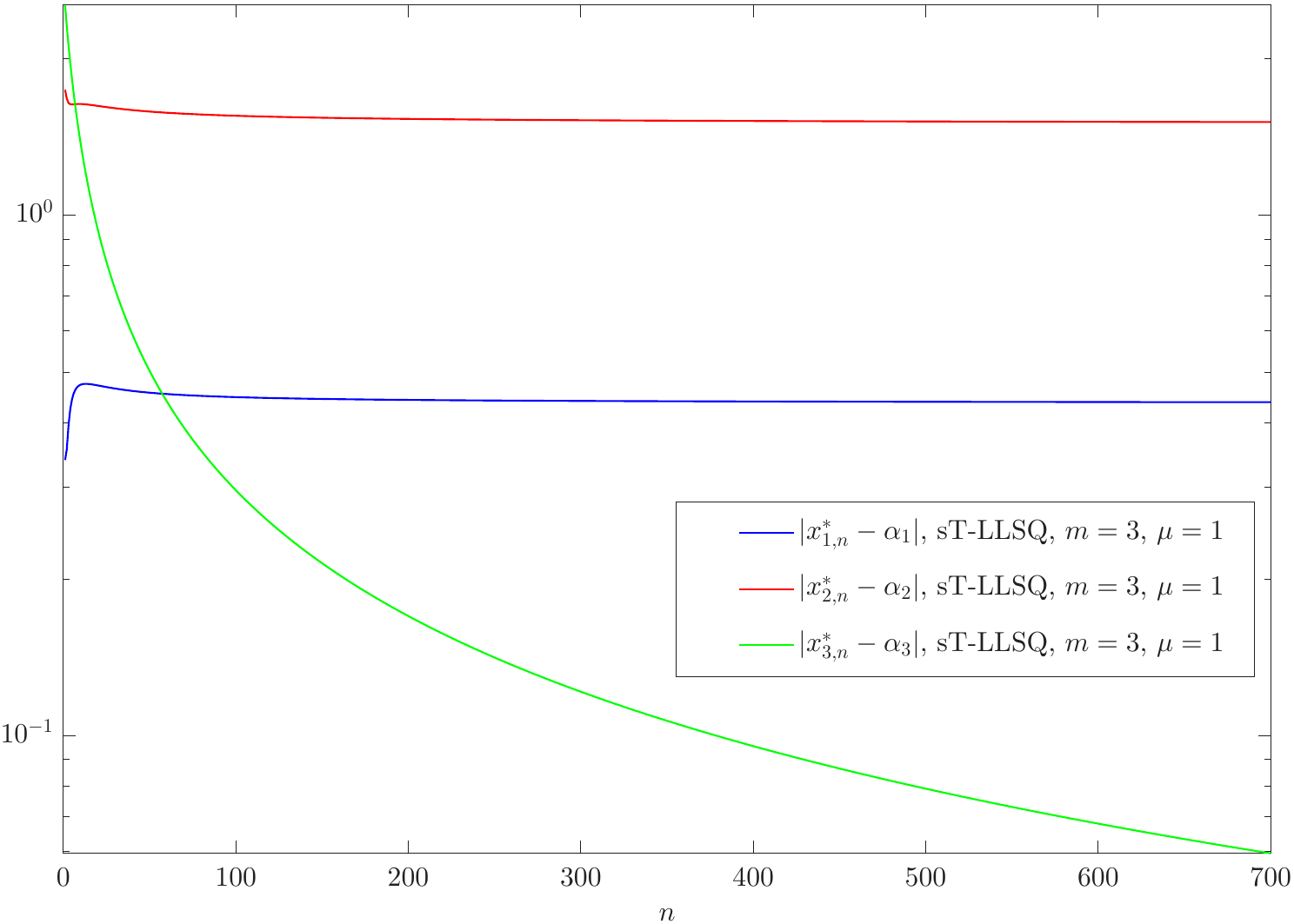}
		\caption{Absolute errors of sT-LLSQ.}
		\label{fig:Catalan_Error_x_star_sTikhonov_LLSQ}
	\end{subfigure}
	\\
	\begin{subfigure}[t]{0.5\textwidth}
		\centering
		\includegraphics[width=\textwidth]{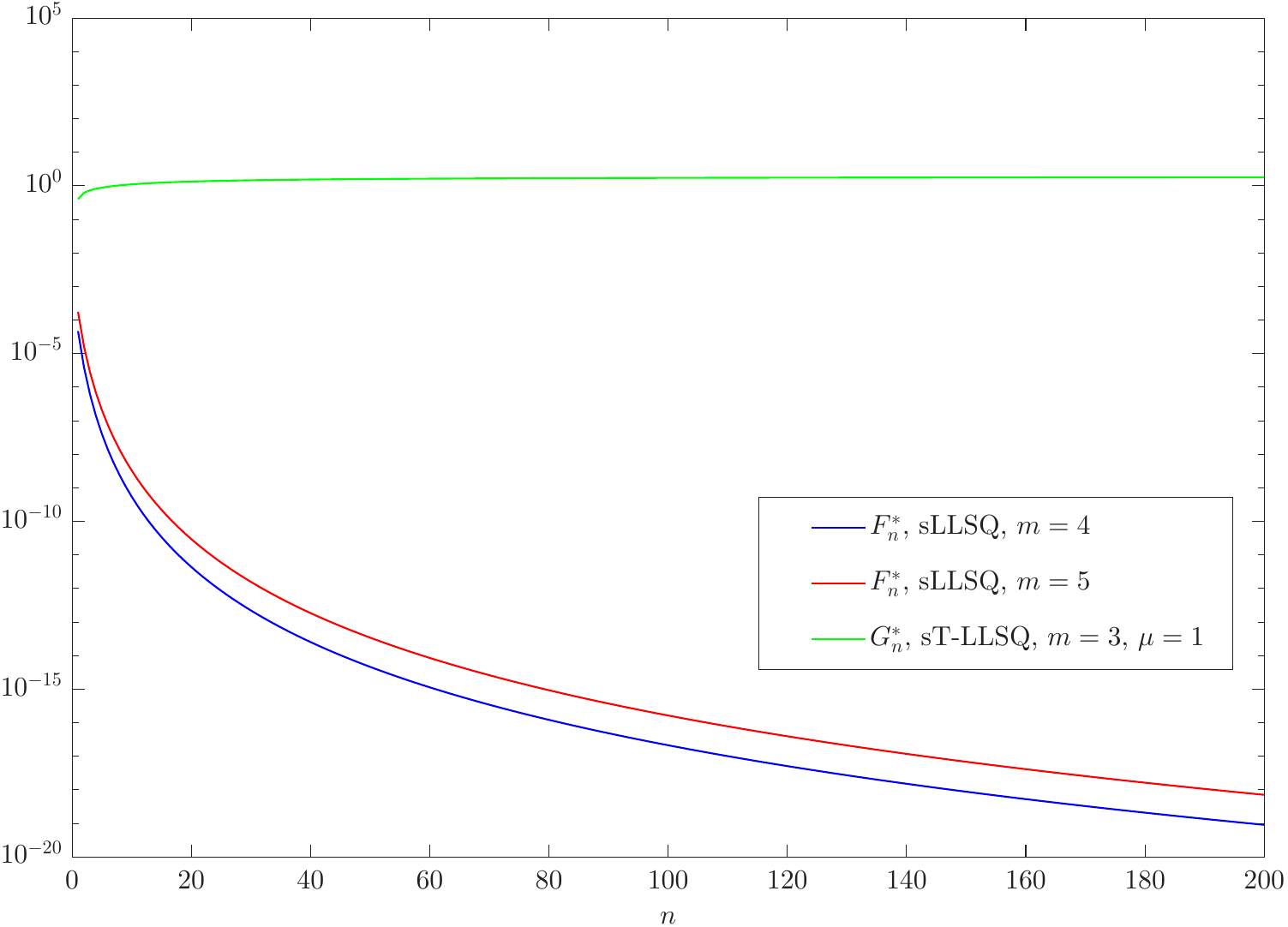}
		\caption{Global minima achieved by sLLSQ and sT-LLSQ. For sLLSQ with $m=k=3$, we have $F_n^* = 0$ for sufficiently large $n$ (this is not visible due to the logarithmic scale of the vertical axis).}
		\label{fig:Catalan_Objective_values}
	\end{subfigure}
	
	\caption{Application to the Catalan numbers ($k=3$).}
	\label{fig:Catalan}
\end{figure}

\subsection{The Factorial Function}
\label{subsec:Factorial_function}

The factorial of an integer $n \in \NN_0$ is defined by
\begin{equation*}
	f(n) = n! \defeq \prod_{i=1}^n i ,
\end{equation*}
which also satisfies the recurrence relation
\begin{equation*}
	f(n) = n f(n-1) ,
\end{equation*}
for all $n \in \NN$ with initial condition $f(0) = 1$.

Moreover, by extending the de Moivre-Stirling approximation, the factorial function has a third-order asymptotic expansion as $n \to \infty$,\footnote{In general, the factorial function admits a complete asymptotic series, that is, 
\begin{equation*}
	n! \simeq \sqrt{2 \pi n} \left( \frac{n}{e} \right)^n \left( 1 + \sum_{l=1}^{\infty} \frac{\beta_l}{n^l} \right)  , \quad \mathrm{as}\ n \to \infty ,
\end{equation*}
where $\{\beta_l\}_{l \in \NN}$ are real coefficients. For more details, see the papers of Brassesco~and~M{\'e}ndez~\cite{Brassesco-Mendez_2011},  G.~and~J.~Marsaglia~\cite{Marsaglia-Marsaglia_1990}, and Nemes~\cite{Nemes_2010}.}
\begin{equation*}
	n! = \sqrt{2 \pi n} \left( \frac{n}{e} \right)^n \left( 1 + \frac{1}{12n} + \frac{1}{288 n^2} - \frac{139}{51840 n^3} + O\left( \frac{1}{n^4} \right) \right) .
\end{equation*}
This expression is of the asymptotic form AF-1 in Table~\ref{tbl:Common_asymptotic_forms} with $k=4$, $\boldsymbol{\alpha} = \transp{\left[ \sqrt{2\pi}, \frac{1}{2}, \frac{1}{e}, 1 \right]}$, $p = 3$ and $g_l(n) = \frac{1}{n^l}$, for all $l \in \{1,2,3,4\}$. 

According to Remark~\ref{rem:AF-1_sLLSQ_convergence} and the convergence-rate invariant~\eqref{eq:Convergence-rate_invariant_y_to_x_with_h}, the sLLSQ technique with $m=k=4$ and knowledge of third-order asymptotics for $n!$ is provably convergent and yields the asymptotic estimates 
\begin{align*}
	x_{1,n}^* - \alpha_1 & = O\left( \frac{\log(n)}{n} \right) = o(1)  , \\
	x_{2,n}^* - \alpha_2 & = O\left( \frac{1}{n} \right) = o(1) , \\
	x_{3,n}^* - \alpha_3 & = O\left( \frac{\log(n)}{n^2} \right) = o(1)  ,  \\
	x_{4,n}^* - \alpha_4 & = O\left( \frac{1}{n^2} \right) = o(1) .
\end{align*}

Nevertheless, based on Table~\ref{tbl:Fundamental_results_for_common_asymptotic_forms_sT-LLSQ} and the convergence-rate invariant~\eqref{eq:Convergence-rate_invariant_y_to_x_with_h}, the sT-LLSQ method exhibits a very slow convergence of $x_{4,n}^*$, that is, 
\begin{equation*}
	x_{4,n}^* - \alpha_4 = O\left( \frac{1}{\log(n)} \right) = o(1) ,
\end{equation*}
as well as bounded absolute errors for the remaining components, i.e., $x_{j,n}^* - \alpha_j = O(1)$ for all $j \in \{1,2,3\}$.

Finally, Figure~\ref{fig:Factorial} presents numerical experiments regarding the application of sLLSQ and sT-LLSQ to factorial asymptotics. Note that RM is not included, since it is \emph{not} directly applicable in this case. The theoretical analysis is again consistent with the numerical results.

\begin{figure}[t!]
	\centering
	\begin{subfigure}[t]{0.5\textwidth}
		\centering
		\includegraphics[width=\textwidth]{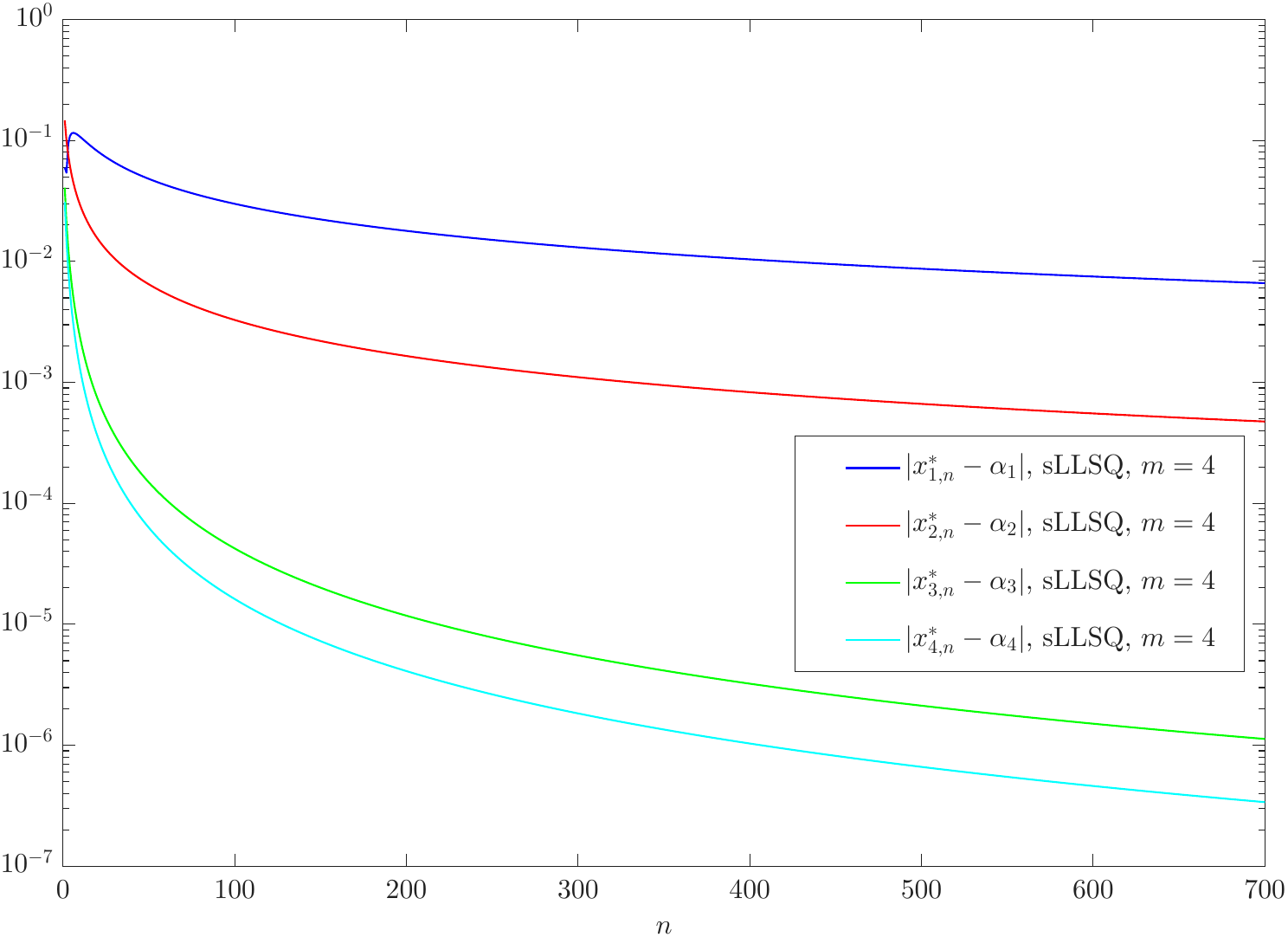}
		\caption{Absolute errors of sLLSQ.}
		\label{fig:Factorial_Error_x_star_sLLSQ}
	\end{subfigure}
	~
	\begin{subfigure}[t]{0.5\textwidth}
		\centering
		\includegraphics[width=\textwidth]{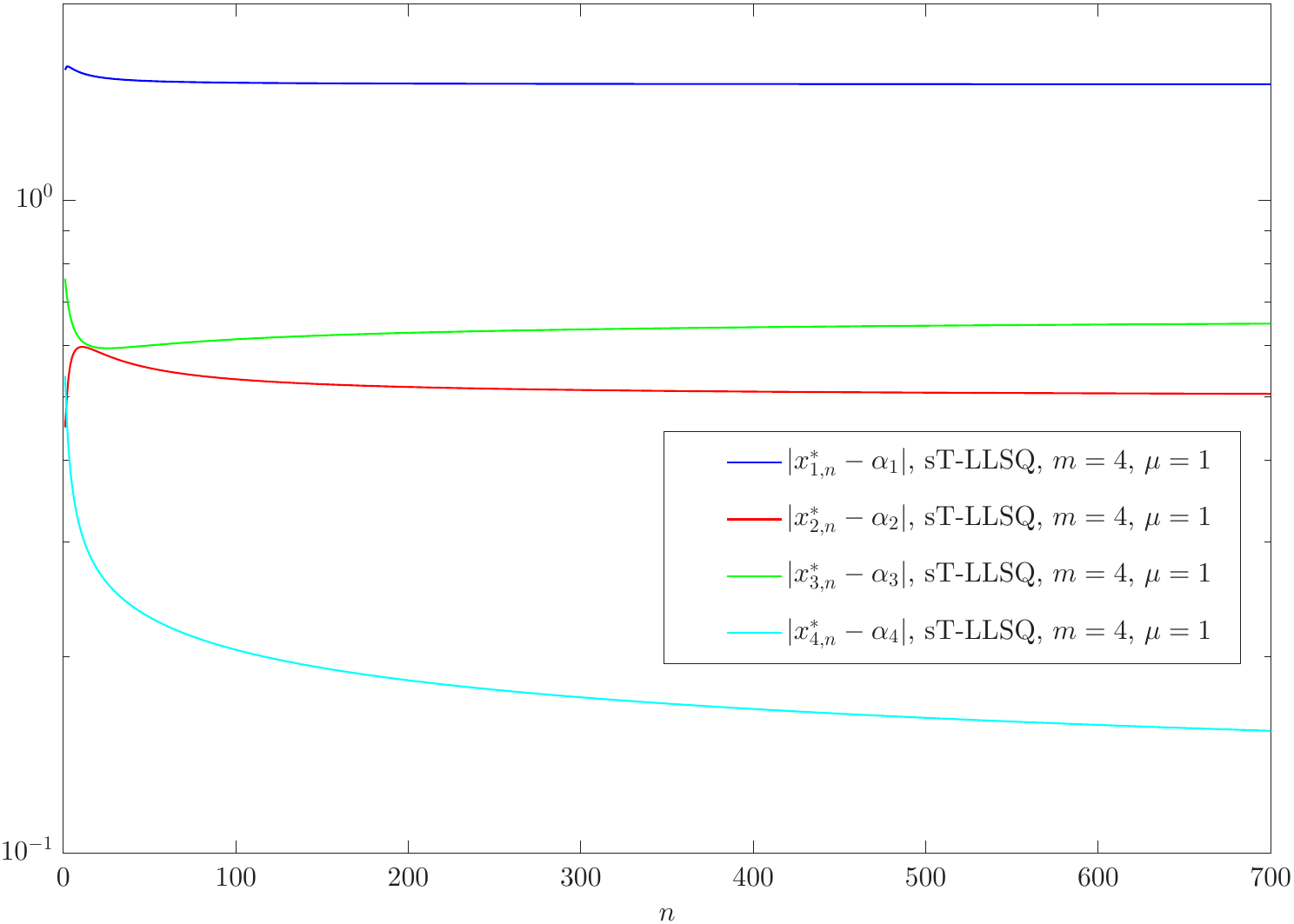}
		\caption{Absolute errors of sT-LLSQ.}
		\label{fig:Factorial_Error_x_star_sTikhonov_LLSQ}
	\end{subfigure}
	\\
	\begin{subfigure}[t]{0.5\textwidth}
		\centering
		\includegraphics[width=\textwidth]{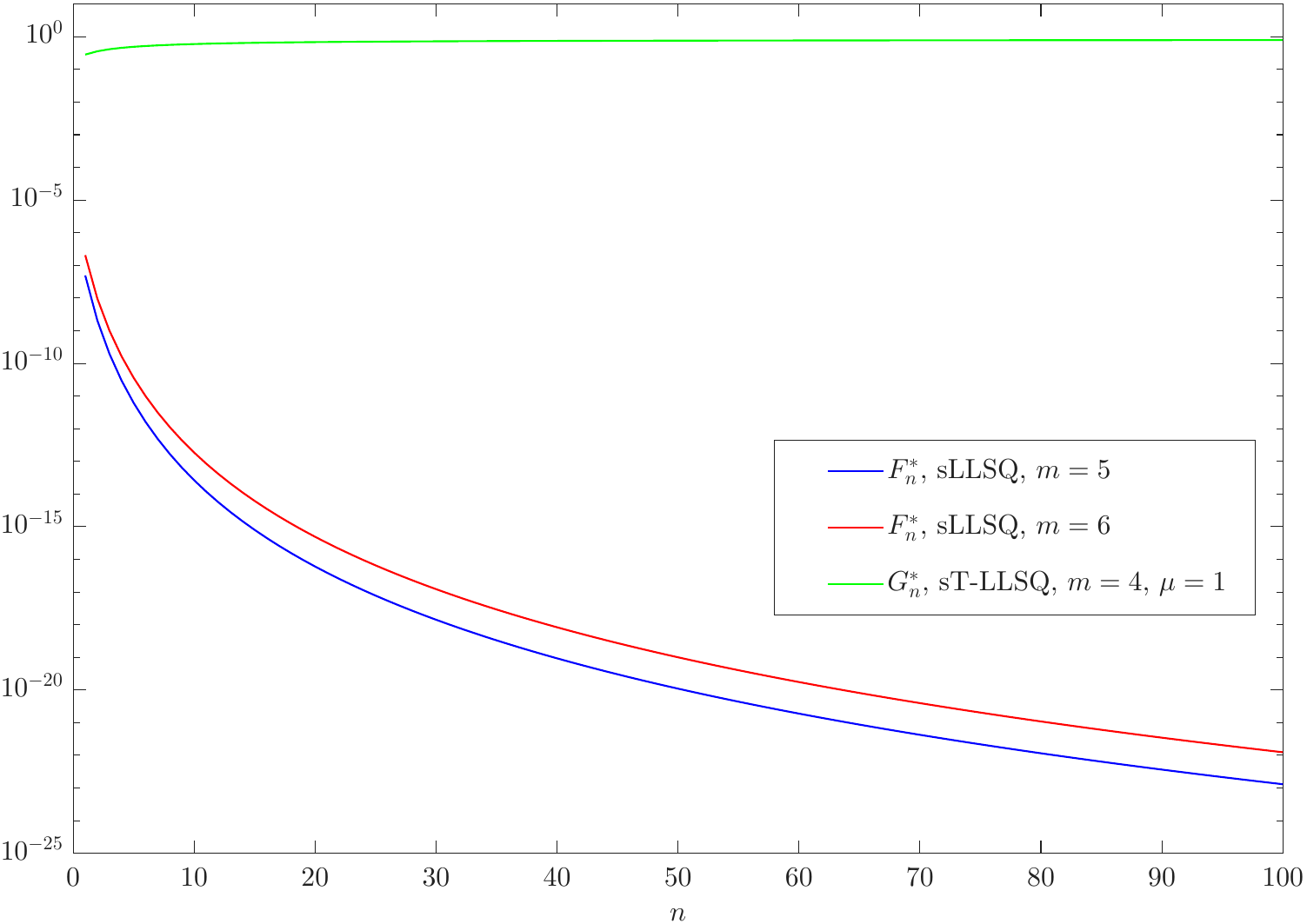}
		\caption{Global minima achieved by sLLSQ and sT-LLSQ. For sLLSQ with $m=k=4$, we have $F_n^* = 0$ for sufficiently large $n$ (this is not visible due to the logarithmic scale of the vertical axis).}
		\label{fig:Factorial_Objective_values}
	\end{subfigure}
	
	\caption{Application to the factorial function ($k=4$).}
	\label{fig:Factorial}
\end{figure}

\subsection{A Counterintuitive Example of Complete Divergence of sLLSQ}
\label{subsec:Counterintuitive_example}

Afterwards, we study a \emph{counterexample} to the assertion that ``sLLSQ can always compute the unknown parameters'', thus revealing its \emph{fundamental limitations} as a numerical method. Interestingly, we will see that sLLSQ may not converge at all to the desired values of parameters. In other words, it may be the case of \emph{complete divergence}, that is, $y_{j,n}^* - \gamma_j \neq o(1)$ for all $j \in \{1,\dots,k\}$. Contrary to naive intuition, this can happen \emph{even when} the global minima achieved by sLLSQ are \emph{identically zero} for sufficiently large $n$, e.g., when $m=k$ (cf.~Lemma~\ref{lem:Sufficient_condition_for_unique_solution_LLSQ}).

First of all, let us consider the simple sequence $\{f(n)\}_{n \in \NN}$ defined by 
\begin{equation*}
	f(n) \defeq \alpha_1 n^{\alpha_2} \alpha_3^n \left( 1 + h(n) \right) = \alpha_1 n^{\alpha_2} \alpha_3^n \left( 1 + O(g_1(n)) \right)  ,
\end{equation*} 
where $\boldsymbol{\alpha} = \transp{[\alpha_1,\alpha_2,\alpha_3]} = \transp{\left[ 7,\frac{1}{3},\frac{5}{2} \right]}$, $h(n) \defeq \frac{(-1)^n}{\sqrt{n}} = O(g_1(n))$ and $g_1(n) \defeq \frac{1}{\sqrt{n}} = o(1)$. According to Table~\ref{tbl:Common_asymptotic_forms}, the function $f$ admits the asymptotic form AF-2 with $k=3$ and $p=0$. Note that the function $h$ is \emph{unknown}, whereas the function $g_1$ is \emph{known} to us when applying the sLLSQ and sT-LLSQ techniques. In the following analysis, we provide $\Theta$-estimates based on the explicit expression of $h$, whereas $O$-estimates are solely based on the function $g_1$. Normally, in ALT we can only derive $O$-estimates. However, we give both estimates in order to evaluate the tightness/sharpness of the latter and to establish the convergence or divergence of numerical methods.  

In general, the sLLSQ analysis (cf.~Theorem~\ref{thm:The_sLLSQ_method} and \eqref{eq:Vector_b_asymptotics}) gives  
\begin{equation*}
	\mathbf{y}_n^* = \pseudoinv{\mathbf{A}} \mathbf{b} = \pseudoinv{\mathbf{A}} \left( \mathbf{A} \boldsymbol{\gamma} + \mathbf{w} \right)  =  \boldsymbol{\gamma} + \pseudoinv{\mathbf{A}} \mathbf{w} \overset{m=k}{=} \boldsymbol{\gamma} + \inv{\mathbf{A}} \mathbf{w}  ,
\end{equation*} 
where $\mathbf{w} = \mathbf{w}(n) \defeq \transp{[\log(1+h(n)),\dots,\log(1+h(n+m-1))]}$. For the particular example with $m=k=3$ and by using a symbolic computation toolbox, we have   
\begin{align*}
	\mathbf{y}_n^* - \boldsymbol{\gamma} = \inv{\mathbf{A}} \mathbf{w} & = {(-1)^n} 4 \sqrt{n} \, \transp{[n\log(n), -n, 1]} + \mathbf{l.o.t.}  \\
	& = \transp{[\Theta(n^{3/2} \log(n)), \Theta(n^{3/2}), \Theta(\sqrt{n})]} , 
\end{align*}
as $n \to \infty$. Specifically, we have
\begin{align*}
\norm{\mathbf{y}_n^* - \boldsymbol{\gamma}} & = \Theta(n^{3/2} \log(n)) \neq o(1) ,  \\
y_{1,n}^* - \gamma_1 & = \Theta(n^{3/2} \log(n))  \neq o(1) , \\
y_{2,n}^* - \gamma_2 & = \Theta(n^{3/2}) \neq o(1) , \\
y_{3,n}^* - \gamma_3 & = \Theta(\sqrt{n}) \neq o(1) ,
\end{align*}
and therefore $\abs{y_{j,n}^* - \gamma_j} \to \infty$, as $n \to \infty$, for all $j \in \{1,\dots,k\}$. The complete divergence of sLLSQ occurs \emph{despite the fact that its global minima $\{F_n^*\}_{n \in \NN}$ are identically zero (i.e., perfect curve fitting) for large enough $n$}. In fact, the \emph{oscillating nature} of the function $h$ is responsible for the sLLSQ divergence. Moreover, from Table~\ref{tbl:Fundamental_results_for_common_asymptotic_forms_sLLSQ} we obtain 
\begin{align*}
	\norm{\mathbf{y}_n^* - \boldsymbol{\gamma}} & = O\left( n^2 \log(n) g_1(n) \right) = O(n^{3/2} \log(n)) , \\
	y_{1,n}^* - \gamma_1 & = O\left( n^2 \log(n) g_1(n) \right) = O(n^{3/2} \log(n)) , \\ 
	y_{2,n}^* - \gamma_2 & = O\left( n^2 g_1(n) \right) = O(n^{3/2}) , \\
	y_{3,n}^* - \gamma_3 & = O\left( n g_1(n) \right) = O(\sqrt{n}) ,
\end{align*}
all of them being in agreement with the previous $\Theta$-estimates; in fact, they are tight. The numerical results shown in Figure~\ref{fig:Counterexample_Error_y_star_sLLSQ} support these observations. For better visualization, we have plotted the absolute errors with respect to $\mathbf{y}_n^*$ instead of $\mathbf{x}_n^*$, because $\abs{x_{1,n}^* - \alpha_1} = \abs{e^{y_{1,n}^*} - \alpha_1} = \alpha_1 \abs{e^{(-1)^n 4 n^{3/2} \log(n) + \mathrm{l.o.t.}} - 1}$ and $\abs{x_{3,n}^* - \alpha_3} = \abs{e^{y_{3,n}^*} - \alpha_3} = \alpha_3 \abs{e^{(-1)^n 4 \sqrt{n} + \mathrm{l.o.t.}} - 1}$ exhibit increasing fluctuations as $n \to \infty$; although, $\abs{x_{2,n}^* - \alpha_2} = \abs{y_{2,n}^* - \gamma_2} \sim 4 n^{3/2}$ does not. Specifically, we have $\lim_{\substack{n \to \infty \\ n \: \mathrm{even}}} \abs{x_{j,n}^* - \alpha_j} = \infty$ and $\lim_{\substack{n \to \infty \\ n \: \mathrm{odd}}} \abs{x_{j,n}^* - \alpha_j} = \alpha_j$, for every $j \in \{1,3\}$.  

In a similar way, we can derive $\Theta$-asymptotics for sT-LLSQ and RM. Firstly, based on the sT-LLSQ analysis (cf.~Theorem~\ref{thm:The_sT-LLSQ_method} and \eqref{eq:y_n_star_sT-LLSQ_rearranged}) and by using a symbolic computation toolbox, we obtain for $m=k=3$ and $\mu=1$    
\begin{align*}
	\mathbf{y}_n^* - \boldsymbol{\gamma} = \mathbf{C} \mathbf{w} - \mu \mathbf{D} \boldsymbol{\gamma} & = \transp{\left[ -\log(\alpha_1), -\alpha_2,  \frac{\alpha_2 \log(n)}{n} \right]} + \mathbf{l.o.t.}  \\
	& = \transp{\left[ \Theta(1), \Theta(1), \Theta\left( \frac{\log(n)}{n} \right) \right]} ,
\end{align*}
as $n \to \infty$. Here, we have used the fact that $\log(\alpha_1) = \log(7) \neq 0$ and $\alpha_2 = \frac{1}{3} \neq 0$. Due to the convergence-rate invariant~\eqref{eq:Convergence-rate_invariant_y_to_x_with_h}, the same estimates hold for $\mathbf{x}_n^* - \boldsymbol{\alpha}$. Secondly, for RM (by using a symbolic computation toolbox) we get the limits 
\begin{align*}
	\lim_{\substack{n \to \infty \\ n \: \mathrm{even}}} \frac{\kappa_n - \alpha_2}{n^{3/2}} = 4 ,  & \quad  \lim_{\substack{n \to \infty \\ n \: \mathrm{odd}}} \frac{\kappa_n - \alpha_2}{n^{3/2}} = -4 , \\
	\lim_{\substack{n \to \infty \\ n \: \mathrm{even}}} \frac{\zeta_n - \alpha_3}{\sqrt{n}} = -4 \alpha_3,  & \quad  \lim_{\substack{n \to \infty \\ n \: \mathrm{odd}}} \frac{\zeta_n - \alpha_3}{\sqrt{n}} = 4 \alpha_3,  \\
	\lim_{\substack{n \to \infty \\ n \: \mathrm{even}}} (r_n - \alpha_3) \sqrt{n} = -2 \alpha_3,  & \quad  \lim_{\substack{n \to \infty \\ n \: \mathrm{odd}}} (r_n - \alpha_3) \sqrt{n} = 2 \alpha_3  . 
\end{align*}
Therefore, by the continuity of absolute-value function and by the limit definition, 
\begin{align*}
	\lim_{n \to \infty} \frac{\abs{\kappa_n - \alpha_2}}{n^{3/2}}  = 4  & \implies  \kappa_n - \alpha_2 = \Theta\left( n^{3/2} \right) \neq o(1) ,  \\
	\lim_{n \to \infty} \frac{\abs{\zeta_n - \alpha_3}}{\sqrt{n}}  = 4 \abs{\alpha_3} = 4 \alpha_3  & \implies  \zeta_n - \alpha_3 = \Theta\left( \sqrt{n} \right) \neq o(1) , \\
	\lim_{n \to \infty} \abs{r_n - \alpha_3} \sqrt{n} = 2 \abs{\alpha_3} = 2 \alpha_3  & \implies  r_n - \alpha_3 = \Theta\left( \frac{1}{\sqrt{n}} \right) = o(1) .
\end{align*}
The last two implications stem from the fact that $\alpha_3 = \frac{5}{2} \neq 0$.

Furthermore, Tables~\ref{tbl:Fundamental_results_for_common_asymptotic_forms_sT-LLSQ}~and~\ref{tbl:Comparison_between_computational_methods} yield the asymptotic estimates 
\begin{equation*}
	x_{1,n}^* - \alpha_1 = O(1) , \quad x_{2,n}^* - \alpha_2 = O(1) , \quad x_{3,n}^* - \alpha_3 = O\left( \frac{\log(n)}{n} \right) = o(1)
\end{equation*} 
for sT-LLSQ, and
\begin{align*}
	\kappa_n - \alpha_2 & = O\left( \max\left( \frac{1}{n^2}, n^2 \abs{g_1(n)} \right) \right) = O\left( n^{3/2} \right) , \\
	\zeta_n - \alpha_3 & = O\left( \max\left( \frac{1}{n^2}, n \abs{g_1(n)} \right) \right) = O\left( \sqrt{n} \right) , \\
	r_n - \alpha_3 & = O\left( \max\left( \frac{1}{n}, \abs{g_1(n)} \right) \right) = O\left( \frac{1}{\sqrt{n}} \right) = o(1) 
\end{align*}
for RM. Notice that the convergence of $x_{3,n}^*$ is relatively fast in comparison with $r_n$ whose convergence is rather slow. For both methods, we can also observe that the $O$-estimates are in line with their $\Theta$-counterparts; actually, they are sharp. Figure~\ref{fig:Counterexample_Error_x_star_sTikhonov_LLSQ_and_RM} presents the absolute errors attained by sT-LLSQ and RM, thus confirming the preceding analysis.

\begin{figure}[t!]
	\centering
	\begin{subfigure}[t]{0.5\textwidth}
		\centering
		\includegraphics[width=\textwidth]{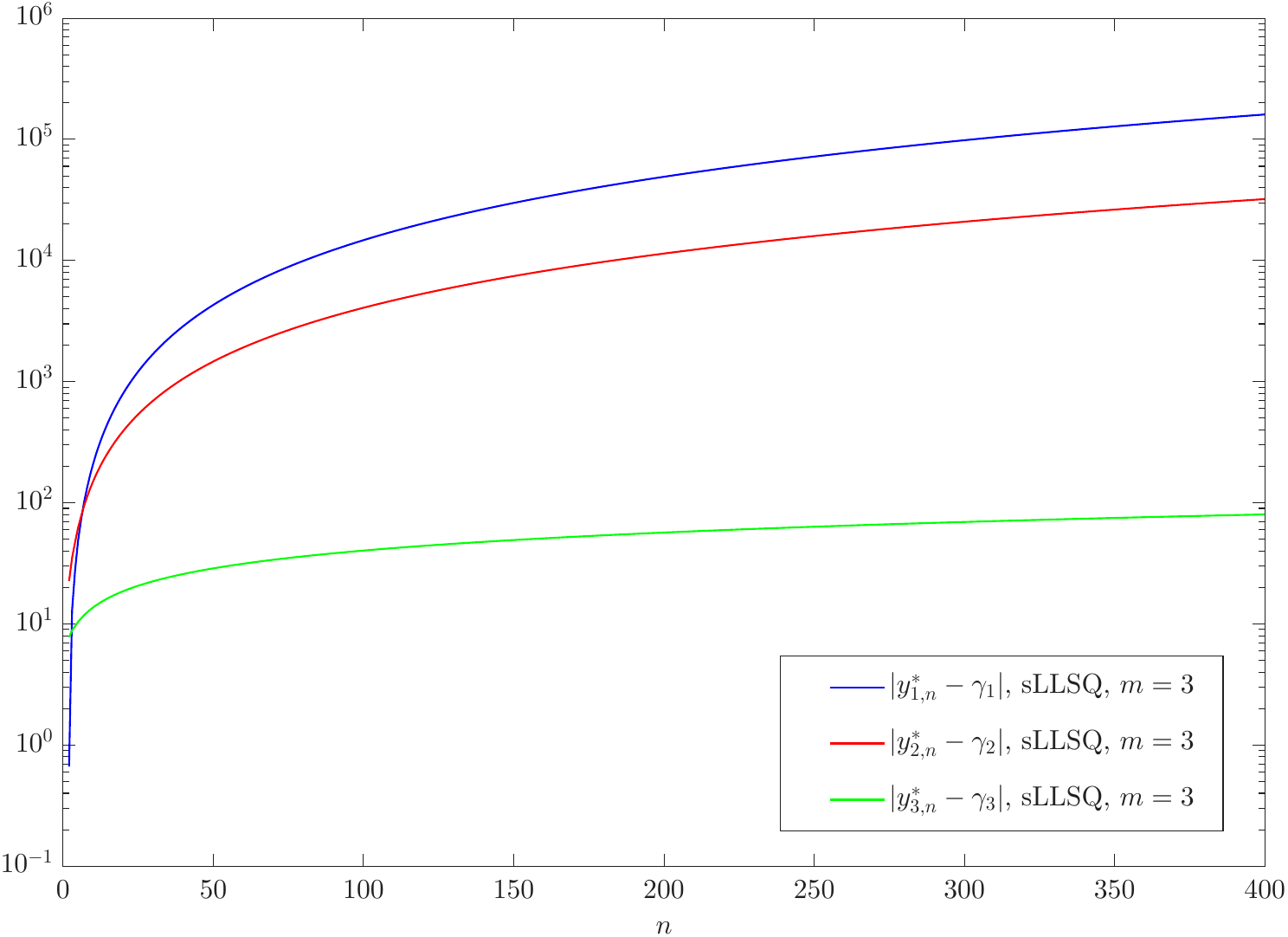}
		\caption{Absolute errors of sLLSQ.}
		\label{fig:Counterexample_Error_y_star_sLLSQ}
	\end{subfigure}
	\\
	\begin{subfigure}[t]{0.75\textwidth}
		\centering
		\includegraphics[width=\textwidth]{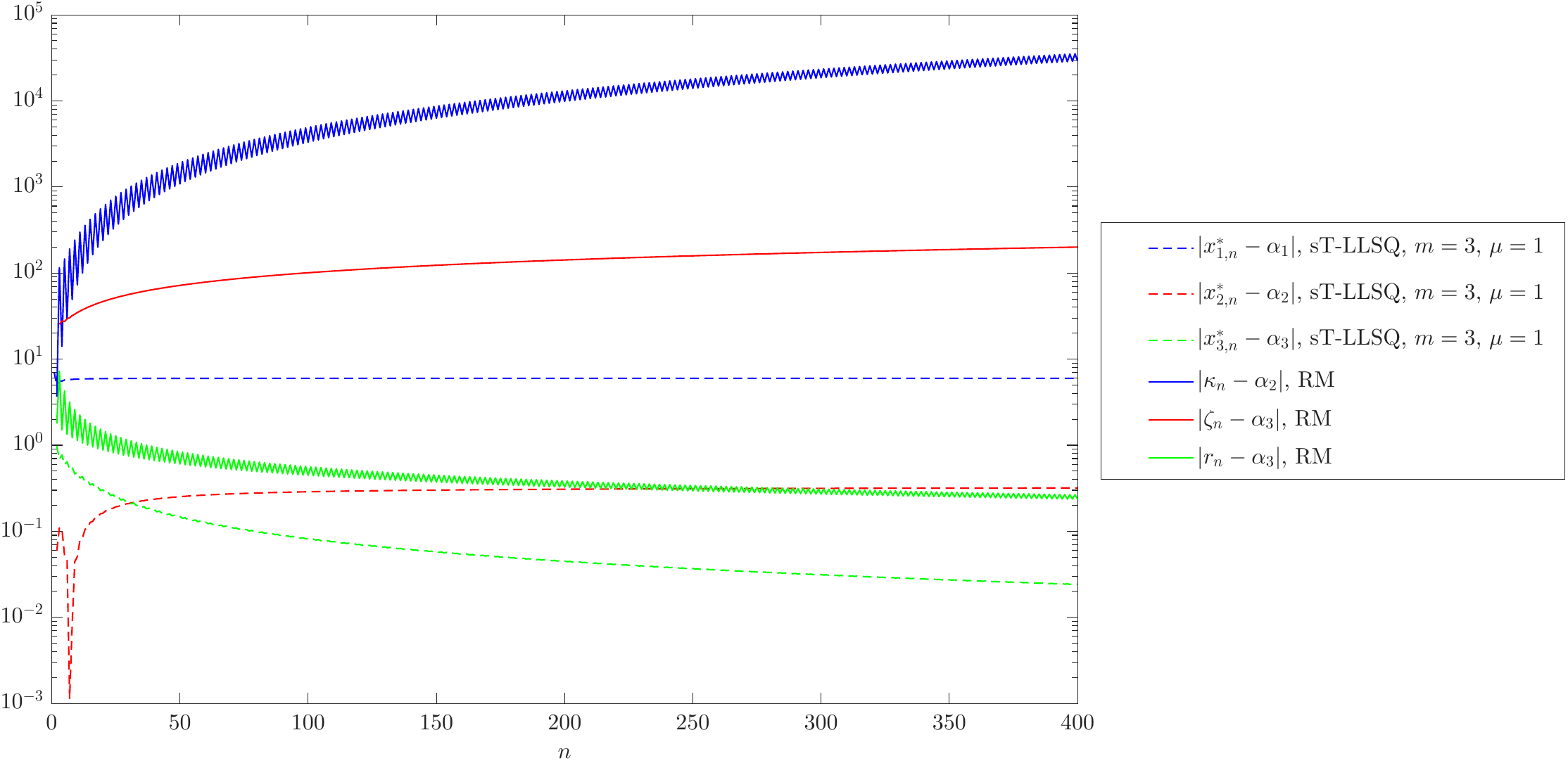}
		\caption{Absolute errors of sT-LLSQ and RM.}
		\label{fig:Counterexample_Error_x_star_sTikhonov_LLSQ_and_RM}
	\end{subfigure}
	
	\caption{An example of complete divergence of sLLSQ ($k=3$).}
	\label{fig:Counterexample_complete_sLLSQ_divergence}
\end{figure}

\begin{remark}
	If the function $h$ is \emph{explicitly known} (although this is not so common in ALT), then it is possible to achieve and prove the convergence of sLLSQ and RM by selecting/sampling \emph{only the even-indexed terms} of the sequence $\{f(n)\}_{n \in \NN}$. Specifically, we define the new functions for each $n \in \NN$:
	\begin{align*}
		& f_{\mathrm{new}}(n) \defeq f(2n) = \alpha_1 (2n)^{\alpha_2} (\alpha_3)^{2n} \left( 1 + h(2n) \right) = \alpha_1 (2n)^{\alpha_2} (\alpha_3)^{2n} \left( 1 + h_{\mathrm{new}}(n) \right)  ,  \\
		& [\varphi_{1,\mathrm{new}}(n), \varphi_{2,\mathrm{new}}(n), \varphi_{3,\mathrm{new}}(n)] \defeq [\varphi_1(2n), \varphi_2(2n), \varphi_3(2n)] =  [1,\log(2n),2n] , \\
		& \mathbf{A}_{\mathrm{new}} = \mathbf{A}_{\mathrm{new}}(n) \defeq 
		\begin{bmatrix}
			\varphi_{1,\mathrm{new}}(n) & \varphi_{2,\mathrm{new}}(n) & \varphi_{3,\mathrm{new}}(n) \\ 
			\vdots & \vdots & \vdots \\
			\varphi_{1,\mathrm{new}}(n+m-1) & \varphi_{2,\mathrm{new}}(n+m-1) & \varphi_{3,\mathrm{new}}(n+m-1)
		\end{bmatrix} , \\
		& \mathbf{w}_{\mathrm{new}} = \mathbf{w}_{\mathrm{new}}(n) \defeq \transp{[\log(1+h_{\mathrm{new}}(n)),\dots,\log(1+h_{\mathrm{new}}(n+m-1))]}  ,
	\end{align*}   
	where $h_{\mathrm{new}}(n) \defeq h(2n) = \frac{1}{\sqrt{2n}}$, and we use the sampled sequence $\{f_{\mathrm{new}}(n)\}_{n \in \NN}$. 
	
	The convergence rates achieved by sLLSQ with $m=k=3$ are given by\footnote{Note that $\determ{\transp{\mathbf{A}_{\mathrm{new}}} \mathbf{A}_{\mathrm{new}}} = \Theta\left( \frac{1}{n^4} \right)$, hence Assumption~\ref{assum:Asymptotic_rank_assumption}/\ref{assum:Equivalent_asymptotic_rank_assumption} is satisfied.} 
	\begin{align*}
		\norm{\mathbf{y}_n^* - \boldsymbol{\gamma}} & = \norm{\inv{\mathbf{A}_{\mathrm{new}}} \mathbf{w}_{\mathrm{new}}} = \Theta\left( \frac{\log(n)}{\sqrt{n}} \right) = o(1) ,  \\
		y_{1,n}^* - \gamma_1 & = \inv{\mathbf{A}_{\mathrm{new}}}(1,:) \mathbf{w}_{\mathrm{new}} = \Theta\left( \frac{\log(n)}{\sqrt{n}} \right) = o(1) , \\
		y_{2,n}^* - \gamma_2 & = \inv{\mathbf{A}_{\mathrm{new}}}(2,:) \mathbf{w}_{\mathrm{new}} = \Theta\left( \frac{1}{\sqrt{n}} \right) = o(1) , \\
		y_{3,n}^* - \gamma_3 & = \inv{\mathbf{A}_{\mathrm{new}}}(3,:) \mathbf{w}_{\mathrm{new}} = \Theta\left( \frac{1}{n^{3/2}} \right) = o(1)  .
	\end{align*}
	However, $y_{1,n}^*$ and $y_{2,n}^*$ exhibit very and quite slow convergence, respectively. Due to the convergence-rate invariant~\eqref{eq:Convergence-rate_invariant_y_to_x_with_h}, the same estimates apply to $\mathbf{x}_n^* - \boldsymbol{\alpha}$.
	
	Furthermore, the sT-LLSQ method (with $m=k=3$ and $\mu=1$) attains the asymptotic estimates
	\begin{align*}
		\mathbf{y}_n^* - \boldsymbol{\gamma} = \mathbf{C}_{\mathrm{new}} \mathbf{w}_{\mathrm{new}} - \mu \mathbf{D}_{\mathrm{new}} \boldsymbol{\gamma} & = \transp{\left[ -\log(\alpha_1), -\alpha_2, \frac{\alpha_2 \log(n)}{2n} \right]} + \mathbf{l.o.t.}  \\
		& = \transp{\left[ \Theta(1), \Theta(1), \Theta\left( \frac{\log(n)}{n} \right) \right]} ,
	\end{align*}
	since $\log(\alpha_1) = \log(7) \neq 0$ and $\alpha_2 = \frac{1}{3} \neq 0$, where $\mathbf{C}_{\mathrm{new}} \defeq \mathbf{D}_{\mathrm{new}} \transp{\mathbf{A}_{\mathrm{new}}}$ and $\mathbf{D}_{\mathrm{new}} \defeq \inv{(\transp{\mathbf{A}_{\mathrm{new}}} \mathbf{A}_{\mathrm{new}} + \mu \mathbf{I})}$.
	
	Finally, the modified RM, with $r_n \defeq \frac{f_{\mathrm{new}}(n+1)}{f_{\mathrm{new}}(n)}$, $\zeta_n \defeq n r_n - (n-1) r_{n-1}$ and $\kappa_n \defeq n^2 \left( 1 - \frac{r_n}{r_{n-1}} \right)$, achieves\footnote{Observe that $\alpha_1 (2n)^{\alpha_2} (\alpha_3)^{2n} = \left( \alpha_1 2^{\alpha_2} \right) n^{\alpha_2} \left( \alpha_3^2 \right)^n$, which implies that $\kappa_n \to \alpha_2$ and $\zeta_n, r_n \to \alpha_3^2$ as $n \to \infty$. Therefore, both $\sqrt{\zeta_n}$ and $\sqrt{r_n}$ converge to $\alpha_3$, since the square-root function is continuous.} 
	\begin{align*}
		\lim_{n \to \infty} \left( \kappa_n - \alpha_2 \right) \sqrt{n} = - \frac{3 \sqrt{2}}{8} & \implies \kappa_n - \alpha_2 = \Theta\left( \frac{1}{\sqrt{n}} \right) = o(1) , \\
		\lim_{n \to \infty} \left( \sqrt{\zeta_n} - \alpha_3 \right) n^{3/2} = \frac{\sqrt{2} \alpha_3}{16} & \implies \sqrt{\zeta_n} - \alpha_3 = \Theta\left( \frac{1}{n^{3/2}} \right) = o(1)  ,  \\
		\lim_{n \to \infty} \left( \sqrt{r_n} - \alpha_3 \right) n = \frac{\alpha_2 \alpha_3}{2} & \implies \sqrt{r_n} - \alpha_3 = \Theta\left( \frac{1}{n} \right) = o(1) .
	\end{align*}
	The last two implications are valid since $\alpha_2 = \frac{1}{3} \neq 0$ and $\alpha_3 = \frac{5}{2} \neq 0$. Observe that the convergence rates of $\kappa_n$ and $\sqrt{\zeta_n}$ are (up to multiplicative constants) identical to those of sLLSQ. All theoretical estimates are in full agreement with the numerical results shown in Figure~\ref{fig:Counterexample_Error_x_star_sLLSQ_sTikhonov_LLSQ_and_RM_only_even_terms}.  	
\end{remark}

\begin{figure}[t!]
	\centering
	\includegraphics[width=0.75\textwidth]{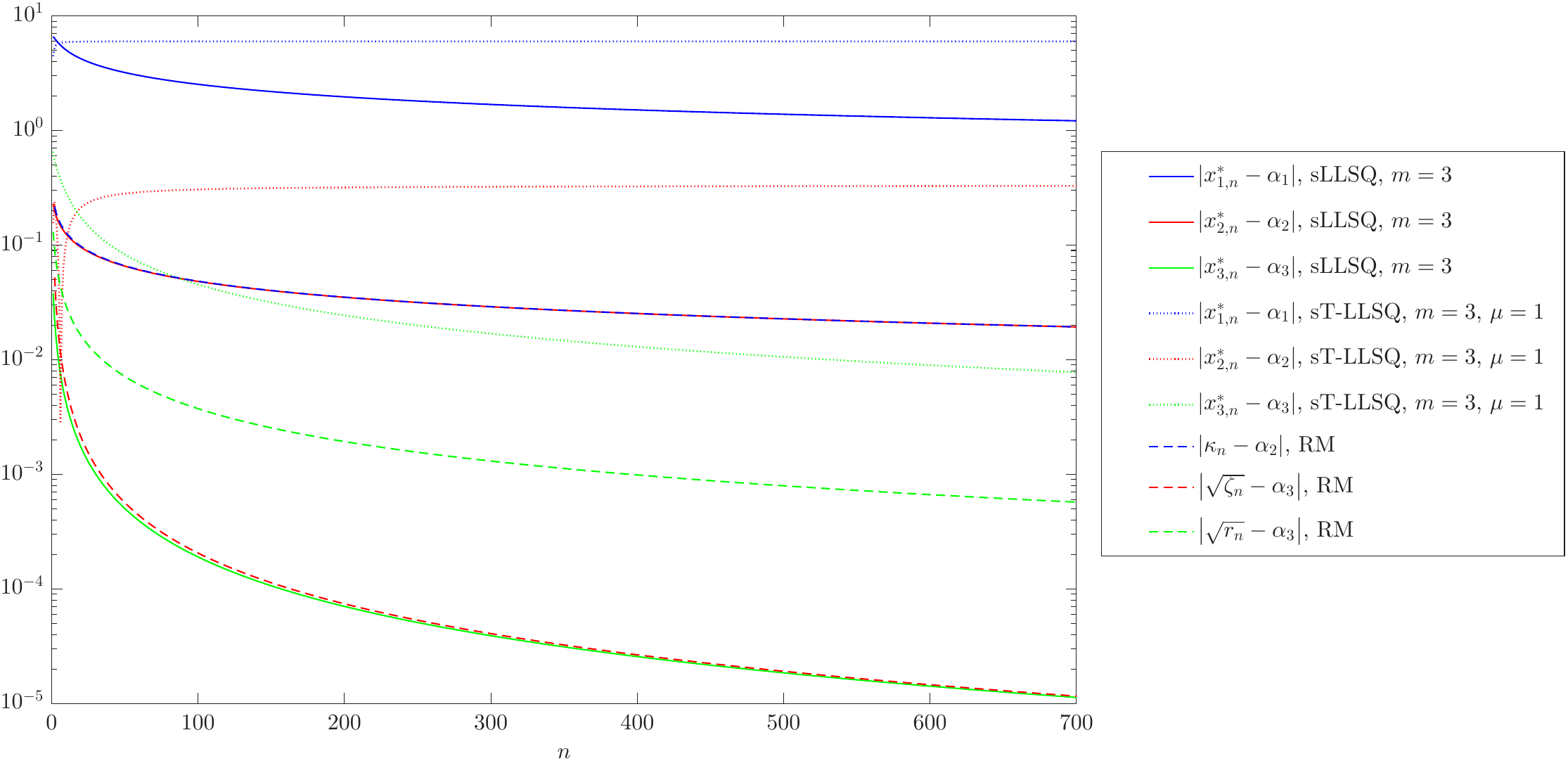}
	\caption{Turning the complete divergence of sLLSQ into convergence by using the sampled sequence $\{f(2n)\}_{n \in \NN}$. Absolute errors of sLLSQ, sT-LLSQ and RM ($k=3$).}
	\label{fig:Counterexample_Error_x_star_sLLSQ_sTikhonov_LLSQ_and_RM_only_even_terms}
\end{figure}

\section{Final Remarks and Open Problems}

ALT is an exciting research area at the intersection of optimization theory and asymptotic analysis. Potential applications span diverse and seemingly different fields where asymptotics plays a central role, e.g., analytic combinatorics (asymptotic enumeration), analysis of algorithms (time/space complexity described by recurrence relations), probability theory (limiting distributions), statistical physics (phase transitions and critical phenomena), fluid mechanics (singularity analysis, perturbation theory), and asymptotic properties of solutions to partial differential equations (long-time behavior). 

In this paper, we have mainly focused on two fundamental numerical techniques, namely, sLLSQ and sT-LLSQ. Both methods are useful and complement existing techniques, such as the ratio method and its variants. Nevertheless, the proposed approaches suffer from \emph{inherent weaknesses}, e.g., slow convergence in some cases, or even divergence (especially when \emph{oscillating/fluctuating behavior} is hidden in the asymptotics of the considered sequence). In particular, the sLLSQ technique should be applied \emph{cautiously}, since its convergence---contrary to naive intuition---is \emph{not} always guaranteed, but only under appropriate conditions. In addition, knowledge of higher-order asymptotics (whenever possible) is crucial to derive tighter asymptotic bounds, or even to ensure the convergence of sLLSQ.  

Furthermore, there exist challenging research questions in ALT that remain to be answered. Some interesting directions include:
\begin{itemize}
	\item Design and analyze efficient algorithms (numerical methods) that are suitable for other asymptotic forms $\widehat{f}$ in the asymptotic expansion~\eqref{eq:Proven_asymptotic_expansion}, i.e., different from~\eqref{eq:Proven_asymptotic_form}. This requires the selection of an appropriate loss function and then the solution of an optimization problem. The choice of loss function affects the difficulty of finding optimal solutions (i.e., the complexity of optimization problem) as well as the convergence analysis. 
		
	\item If only approximate/sub-optimal solutions to the optimization problems can be obtained (e.g., due to non-convexity), then can we still guarantee convergence to the unknown parameters?  
	
	\item Numerical stability/sensitivity: Do small errors/perturbations in the input data (e.g., due to numerical inaccuracies of finite precision) result in small changes in the final solution? In other words, how does the quality of input data affect the (asymptotic) convergence rate? Note that ALT involves both numerical and asymptotic features, therefore classical (non-asymptotic) error analysis is probably not appropriate in this case. 
	
	\item In general, error estimation remains an important and, at the same time, wide-open problem. According to Guttmann~\cite[\S{10}]{Guttmann_1989}, error estimation (from a purely-numerical point of view) has enormous value for the field and would lead to a conventional branch of numerical analysis. Nevertheless, it is not yet clear how to deal with the asymptotic nature of the problem. 	
\end{itemize}

Concerning the last point, we should first distinguish between (and then combine the benefits of) asymptotic and numerical analysis. On the one hand, asymptotic analysis deals with convergence/growth/decay rates in the limit as $n \to \infty$, without considering explicit constants (these quantities are hidden in the asymptotic notation), and therefore it is more theoretical. On the other hand, numerical (non-asymptotic) analysis focuses on explicit error bounds for particular values/ranges of the variable $n$, thus being more practical and suitable for real-world applications. As a result, ALT poses a \emph{unique challenge}: how can we exploit the advantages of both worlds? There is no easy answer according to de Bruijn~\cite[\S{1.7}]{deBruijn_1958} who discussed, in a simple and beautiful way, the basic difference between asymptotic and numerical analysis: 
\begin{displayquote}
	$[\dots]$ neither $O$-formulas nor $o$-formulas have, as they stand, any direct value for numerical purposes. However, in almost all cases where such formulas have been derived, it is possible to retrace the proof, replacing all $O$-formulas by definite estimates involving explicit numerical constants. $[\dots]$ 
	
	In most cases, the final estimates obtained in this way are rather weak, with constants a thousand times, say, greater than they could be. The reason is, of course, that such estimates are obtained by means of a considerable number of steps, and in each step a factor~$2$ or so is easily lost. Quite often it is possible to reduce such errors by a more careful examination. 
	
	But even if the asymptotic result is presented in its best possible explicit form, it need not be satisfactory from the numerical point of view.  
\end{displayquote}

Nevertheless, the following idea could be useful. Suppose that the convergence of a numerical method  is theoretically guaranteed, i.e., $\norm{\mathbf{x}_n^* - \boldsymbol{\alpha}} = O(h(n)) = o(1)$ for some function $h(n) = o(1)$. If all the hidden constants within $O$-terms can be made \emph{explicit} by a more diligent investigation, including the proven asymptotic expansion~\eqref{eq:Proven_asymptotic_expansion} and the convergence analysis of the numerical method, then we eventually obtain $\norm{\mathbf{x}_n^* - \boldsymbol{\alpha}} \leq c \abs{h(n)}$ for all integers $n \geq n_0$, where $c \in \RR_{>0}$ and $n_0 \in \NN$ are explicitly known. Note that the integer $n_0$ can usually be reduced by increasing $c$ accordingly.\footnote{In particular, if $\abs{f(n)} \leq c_1 \abs{g(n)}$ for all integers $n \geq n_1$, and $g(n) \neq 0$ for all integers $n \geq n_2$, with $n_1 > n_2$, then $\abs{f(n)} \leq c_2 \abs{g(n)}$ for all integers $n \geq n_2$, where $c_2 \defeq \max(c_1,c_3)$ and $c_3 \defeq \max\limits_{n_2 \leq n \leq n_1} \frac{\abs{f(n)}}{\abs{g(n)}}$. Observe that $c_2 \geq c_1$.} Therefore, given any numerical precision $\varepsilon > 0$, by the limit definition (since $\norm{\mathbf{x}_n^* - \boldsymbol{\alpha}} = o(1)$) there exists an $n_{\varepsilon} \in \NN$ such that $\norm{\mathbf{x}_n^* - \boldsymbol{\alpha}} \leq \varepsilon$ for all integers $n \geq n_{\varepsilon}$. Specifically, the integer $n_{\varepsilon}$ can be determined by solving the inequality $c \abs{h(n)} \leq \varepsilon \iff \abs{h(n)} \leq \frac{\varepsilon}{c}$, under the constraint $n \geq n_0$. Note that there exists an  $n'_0 = n'_0(\varepsilon,c) \in \NN$ such that $\abs{h(n)} \leq \frac{\varepsilon}{c}$ for all integers $n \geq n'_0$, because $h(n) = o(1)$. If the integer $n'_0$ can be easily found (e.g., when $\abs{h}$ is a simple function), then we just select $n_{\varepsilon} \defeq \max(n_0,n'_0)$ for which we have $\norm{\mathbf{x}_{n_{\varepsilon}}^* - \boldsymbol{\alpha}} \leq \varepsilon$. In this way, we are able to determine \emph{in advance} a sufficient number of iterations in order to achieve the desired accuracy $\varepsilon > 0$. In fact, a \emph{single iteration} for $n = n_{\varepsilon}$ is required, since each iteration is independent of the previous ones. Observe that there is a \emph{trade-off} between $n_0$ and $n'_0$: the decrease of $n_0$ results in the increase of $c$, which normally leads to the increase of $n'_0$. In order to minimize $n_{\varepsilon}$ for a given $\varepsilon > 0$, we should find an optimal pair $(n_0^*,c^*)$; this constitutes an optimization problem on its own. The above discussion motivates another research direction, i.e., \emph{automatic tuning/adjustment of the $(n_0,c)$-pair to minimize $n_{\varepsilon}$}. In other words, does there exist an efficient algorithm for this task, or at least for obtaining a sub-optimal (but relatively small) $n_{\varepsilon}$?   

Finally, by providing a new perspective through the lens of optimization, it is hoped that this work will stimulate further interest and research activity in ALT.

\appendix{

\section{Mathematical Background}
\label{sec:Mathematical_background}

\subsection{Vector and Matrix Norms}

The notion of vector/matrix norm can be defined on finite/infinite-dimensional vector spaces over the real or complex numbers. Nevertheless, in this paper we focus on real vector spaces of finite dimension. 

\begin{definition}
	Let $m,k \in \NN$. A function $f \colon \RR^m \to \RR$ (resp. $g \colon \RR^{m \times k} \to \RR$) is called a \emph{vector norm} on $\RR^m$ (resp. a \emph{matrix norm} on $\RR^{m \times k}$) if, and only if, it satisfies the following properties, for all vectors $\mathbf{x},\mathbf{y} \in \RR^m$ (resp. for all matrices $\mathbf{A},\mathbf{B} \in \RR^{m \times k}$) and all scalars $\alpha \in \RR$: 
	\begin{enumerate}[label={\arabic*.}]
		\item $f(\mathbf{x}) \geq 0$ $\quad$ (resp. $g(\mathbf{A}) \geq 0$),
		
		\item $f(\mathbf{x}) = 0 \iff \mathbf{x} = \mathbf{0}$ $\quad$ (resp. $g(\mathbf{A}) = 0 \iff \mathbf{A} = \mathbf{O}$),
		
		\item $f(\alpha \mathbf{x}) = \abs{\alpha} f(\mathbf{x})$ $\quad$ (resp. $g(\alpha \mathbf{A}) = \abs{\alpha} g(\mathbf{A})$), and 
		
		\item $f(\mathbf{x} + \mathbf{y}) \leq f(\mathbf{x}) + f(\mathbf{y})$ $\quad$ (resp. $g(\mathbf{A} + \mathbf{B}) \leq g(\mathbf{A}) + g(\mathbf{B})$). 
	\end{enumerate}
	Such a function is denoted by $\norm{\cdot}$, usually with a subscript.  
\end{definition}

For example, given a real number $1 \leq p \leq \infty$, the vector $p$-norm is defined by
\begin{equation*}
	\norm{\mathbf{x}}_p \defeq \left( \sum_{i=1}^m \abs{x_i}^p \right)^{1/p} ,
\end{equation*}
for every $\mathbf{x} \in \RR^m$, where $m \in \NN$. More specifically, for $p = 1,2,\infty$ we obtain the most useful norms, that is, the absolute-sum/Manhattan norm 
\begin{equation*}
	\onenorm{\mathbf{x}} \defeq \sum_{i=1}^m \abs{x_i}  ,
\end{equation*}
the Euclidean norm
\begin{equation*}
	\twonorm{\mathbf{x}} \defeq \left( \sum_{i=1}^m x_i^2 \right)^{1/2} = \sqrt{\transp{\mathbf{x}} \mathbf{x}} 
\end{equation*}
and the maximum/infinity norm 
\begin{equation*}
	\maxnorm{\mathbf{x}} = \infnorm{\mathbf{x}} \defeq \max\limits_{1 \leq i \leq m} {\abs{x_i}} .
\end{equation*}

The most common matrix norms (for each $\mathbf{A} \in \RR^{m \times k}$, where $m,k \in \NN$) are the $p$-norm induced by the vector $p$-norm, with $1 \leq p \leq \infty$, and defined by
\begin{equation*}
	\norm{\mathbf{A}}_p \defeq \sup_{\mathbf{x} \in \RR^k,\, \mathbf{x} \ne \mathbf{0}} \frac{\norm{\mathbf{A} \mathbf{x}}_p}{\norm{\mathbf{x}}_p} = \max_{\mathbf{x} \in \RR^k,\, \norm{\mathbf{x}}_p = 1} {\norm{\mathbf{A} \mathbf{x}}_p}  ,   
\end{equation*}   
the maximum norm 
\begin{equation*}
	\maxnorm{\mathbf{A}} \defeq \max_{\substack{1 \leq i \leq m, \\ 1 \leq j \leq k}} {\abs{a_{i,j}}}  ,
\end{equation*}
and the Frobenius norm 
\begin{equation*}
	\Fnorm{\mathbf{A}} \defeq \sqrt{\sum_{i=1}^m \sum_{j=1}^k {\abs{a_{i,j}}}^2} = \sqrt{\tr{\transp{\mathbf{A}} \mathbf{A}}} = \sqrt{\sum_{j=1}^k \lambda_j(\transp{\mathbf{A}} \mathbf{A})} = \sqrt{\sum_{j=1}^k \sigma_j^2(\mathbf{A})}   ,
\end{equation*}
where $\lambda_j(\cdot)$ and $\sigma_j(\cdot)$ are the $j$-th eigenvalue and singular value, respectively.\footnote{For every index $j \in \{1,\dots,k\}$ it holds that $\sigma_j(\mathbf{A}) = \sqrt{\lambda_j(\transp{\mathbf{A}} \mathbf{A})}$, with $\RR \ni \lambda_j(\transp{\mathbf{A}} \mathbf{A}) \geq 0$ since $\transp{\mathbf{A}} \mathbf{A}$ is (symmetric) positive semi-definite: $\transp{\mathbf{x}} (\transp{\mathbf{A}} \mathbf{A}) \mathbf{x} = \transp{(\mathbf{A} \mathbf{x})} (\mathbf{A} \mathbf{x}) = \twonorm{\mathbf{A} \mathbf{x}}^2 \geq 0$ for all $\mathbf{x} \in \RR^k$. It is also known that the number of nonzero singular values of a matrix is equal to its rank, i.e., $\card{ \left\{ j \in \{1,\dots,k\} : \sigma_j(\mathbf{A}) \neq 0 \right\} } = \rank{\mathbf{A}}$.} In particular, the matrix $p$-norm for $p = 1,2,\infty$ reduces to the maximum absolute column-sum norm
\begin{equation*}
	\onenorm{\mathbf{A}} \defeq \max_{1 \leq j \leq k} {\sum_{i=1}^m \abs{a_{i,j}}} ,
\end{equation*}
the spectral norm
\begin{equation*}
	\twonorm{\mathbf{A}} \defeq \sigma_{\max} (\mathbf{A}) = \sqrt{\lambda_{\max}(\transp{\mathbf{A}} \mathbf{A})} = \sqrt{\lambda_{\max}(\mathbf{A} \transp{\mathbf{A}})}  ,
\end{equation*}
where $\sigma_{\max} (\cdot)$ (resp. $\lambda_{\max} (\cdot)$) represents the largest singular value (resp. eigenvalue), and finally the maximum absolute row-sum norm
\begin{equation*}
	\infnorm{\mathbf{A}} \defeq \max_{1 \leq i \leq m} {\sum_{j=1}^k \abs{a_{i,j}}}  .
\end{equation*}

For all matrices $\mathbf{A} \in \RR^{m \times k}$ and $\mathbf{B} \in \RR^{k \times r}$, where $m,k,r \in \NN$, we have the following (up-to-a-constant) sub-multiplicative properties:
\begin{equation} \label{eq:Sub-multiplicative_properties}
	\begin{gathered} 
		\norm{\mathbf{A} \mathbf{B}}_p \leq \norm{\mathbf{A}}_p  \norm{\mathbf{B}}_p , \\ 
		\Fnorm{\mathbf{A} \mathbf{B}} \leq \Fnorm{\mathbf{A}} \Fnorm{\mathbf{B}} ,  \\
		\maxnorm{\mathbf{A} \mathbf{B}} \leq k \maxnorm{\mathbf{A}} \maxnorm{\mathbf{B}} .
	\end{gathered}
\end{equation}

An important result for finite-dimensional real vector spaces is the following. 

\begin{lemma}[Norm equivalence] \label{lem:Norm_equivalence}
	Let $m,k \in \NN$. Then, all vector (resp. matrix) norms on $\RR^m$ (resp. $\RR^{m \times k}$) are \emph{equivalent}, that is, for every two such norms $\norm{\cdot}$ and $\norm{\cdot}'$ there exist constants $c_1, c_2 \in \RR_{>0}$, which may depend on $m$ (resp. $m$ and $k$), such that $c_1 \norm{\mathbf{x}} \leq \norm{\mathbf{x}}' \leq c_2 \norm{\mathbf{x}}$ for all $\mathbf{x} \in \RR^m$ (resp. $c_1 \norm{\mathbf{A}} \leq \norm{\mathbf{A}}' \leq c_2 \norm{\mathbf{A}}$ for all $\mathbf{A} \in \RR^{m \times k}$). 
\end{lemma}

For example, for each vector $\mathbf{x} \in \RR^m$, the following inequalities hold 
\begin{gather*} 
	\twonorm{\mathbf{x}} \leq \onenorm{\mathbf{x}} \leq \sqrt{m} \twonorm{\mathbf{x}} , \\
	\maxnorm{\mathbf{x}} \leq \onenorm{\mathbf{x}} \leq m \maxnorm{\mathbf{x}} , \\
	\maxnorm{\mathbf{x}} \leq \twonorm{\mathbf{x}} \leq \sqrt{m} \maxnorm{\mathbf{x}} .
\end{gather*}
In addition, for every matrix $\mathbf{A} \in \RR^{m \times k}$, we have
\begin{gather*} 
	\twonorm{\mathbf{A}} \leq \Fnorm{\mathbf{A}} \leq \sqrt{\rank{\mathbf{A}}} \twonorm{\mathbf{A}} \leq \sqrt{\min(m,k)} \twonorm{\mathbf{A}} , \\ 
	\maxnorm{\mathbf{A}} \leq \Fnorm{\mathbf{A}} \leq \sqrt{m k} \maxnorm{\mathbf{A}} , \\
	\maxnorm{\mathbf{A}} \leq \twonorm{\mathbf{A}} \leq \sqrt{m k} \maxnorm{\mathbf{A}} , \\
	\frac{1}{\sqrt{k}} \twonorm{\mathbf{A}} \leq \onenorm{\mathbf{A}} \leq \sqrt{m} \twonorm{\mathbf{A}} , \\
	\frac{1}{\sqrt{m}} \twonorm{\mathbf{A}} \leq \infnorm{\mathbf{A}} \leq \sqrt{k} \twonorm{\mathbf{A}}  .
\end{gather*}

In the case of infinite-dimensional vector spaces, however, there exist norms that are \emph{not} equivalent. See the books by Golub~and~Van~Loan~\cite[\S{2.2--2.3}]{Golub-VanLoan_1996} as well as Horn~and~Johnson~\cite[\S{5}]{Horn-Johnson_1985} for further details about norms.

Furthermore, by virtue of Lemma~\ref{lem:Norm_equivalence}, the (up-to-a-constant) sub-multiplicative property in \eqref{eq:Sub-multiplicative_properties} can be extended to \emph{any} matrix norm on finite-dimensional real vector spaces.  

\begin{lemma}[Universal sub-multiplicative property] \label{lem:Universal_sub-multiplicative_property}
	Let $\norm{\cdot}$ be a matrix norm on $\RR^{m \times k}$, $\RR^{k \times r}$ and $\RR^{m \times r}$, where $m,k,r \in \NN$. Then, there exists a constant $c = c(m,k,r) \in \RR_{>0}$ such that $\norm{\mathbf{A} \mathbf{B}} \leq c \norm{\mathbf{A}} \norm{\mathbf{B}}$, for all matrices $\mathbf{A} \in \RR^{m \times k}$ and $\mathbf{B} \in \RR^{k \times r}$.
\end{lemma}

\begin{proof}
	By Lemma~\ref{lem:Norm_equivalence}, there are constants $c_1 = c_1(m,r) \in \RR_{>0}$, $c_2 = c_2(m,k) \in \RR_{>0}$ and $c_3 = c_3(k,r) \in \RR_{>0}$ such that 
	\begin{equation*}
		\norm{\mathbf{A} \mathbf{B}} \leq c_1 \Fnorm{\mathbf{A} \mathbf{B}} \overset{\eqref{eq:Sub-multiplicative_properties}}{\leq} c_1 \Fnorm{\mathbf{A}} \Fnorm{\mathbf{B}} \leq c_1 (c_2 \norm{\mathbf{A}}) (c_3 \norm{\mathbf{B}})    .
	\end{equation*}
	Finally, we choose $c = c(m,k,r) \defeq c_1 c_2 c_3 \in \RR_{>0}$. 
\end{proof}

\begin{remark}
	In asymptotic notation we do \emph{not} need to specify the norm involved, as long as the matrix dimensions are \emph{independent} of the considered variable. For example, given a matrix $\mathbf{A} = \mathbf{A}(n) \in \RR^{m \times k}$ whose elements depend on the variable $n \in \NN$, it holds that $\norm{\mathbf{A}}' = \Theta(\norm{\mathbf{A}})$ as $n \to \infty$, for all matrix norms $\norm{\cdot}$ and $\norm{\cdot}'$ on $\RR^{m \times k}$. This is because of the norm equivalence in $\RR^{m \times k}$ (Lemma~\ref{lem:Norm_equivalence}). Moreover, due to Lemma~\ref{lem:Universal_sub-multiplicative_property}, we have $\norm{\mathbf{A} \mathbf{B}} = O(\norm{\mathbf{A}} \norm{\mathbf{B}})$ as $n \to \infty$, whenever the dimensions of matrices $\mathbf{A} = \mathbf{A}(n)$ and $\mathbf{B} = \mathbf{B}(n)$ are independent of $n$. 
\end{remark}

\subsection{Convergence of Vector Sequences}

We begin with the following definition, which is very useful throughout the paper. 

\begin{definition}[Componentwise convergence and convergence in norm] \label{def:Convergence_of_sequence_of_vectors}
	Let $m \in \NN$ and $\norm{\cdot}$ be a vector norm on $\RR^m$. We say that the sequence $\{\mathbf{x}_n\}_{n \in \NN_0}$ of vectors in $\RR^m$ \emph{converges componentwise} (resp. \emph{converges in $\norm{\cdot}$-norm}) to $\overline{\mathbf{x}} \in \RR^m$ as $n \to \infty$ if, and only if, $\lim_{n \to \infty} x_{i,n} = \overline{x}_i$ for all $i \in \{1,\dots,m\}$ (resp. $\lim_{n \to \infty} \norm{\mathbf{x}_n - \overline{\mathbf{x}}} = 0$). 
\end{definition}

Now, we can state an immediate consequence of the norm equivalence in the vector space $\RR^m$ according to Lemma~\ref{lem:Norm_equivalence}. 
\begin{lemma} \label{lem:Choice_of_vector_norm}
	Let $\norm{\cdot}$ and $\norm{\cdot}'$ be two vector norms on $\RR^m$. If a sequence of \linebreak $m$-dimensional real vectors converges in $\norm{\cdot}$-norm to a vector $\overline{\mathbf{x}} \in \RR^m$, then the sequence also converges in $\norm{\cdot}'$-norm to the same vector and vice versa. 
\end{lemma}

Due to Lemma~\ref{lem:Choice_of_vector_norm}, the choice of vector norm in Definition~\ref{def:Convergence_of_sequence_of_vectors} is \emph{irrelevant}, thus $\norm{\cdot}$ can be replaced by \emph{any} norm on $\RR^m$. Since all norms on $\RR^m$ are equivalent to the maximum/infinity vector norm $\maxnorm{\cdot}$, we have the following result.  

\begin{lemma} \label{lem:Convergence_equivalence}
	A sequence of real vectors in $\RR^m$ converges in (any) norm to a vector $\overline{\mathbf{x}} \in \RR^m$ if and only if the sequence converges componentwise to $\overline{\mathbf{x}}$.   
\end{lemma}

\begin{remark}
	In view of Lemma~\ref{lem:Convergence_equivalence}, when a sequence $\{\mathbf{x}_n\}_{n \in \NN_0}$ of finite-dimensional real vectors converges componentwise (or, equivalently, with respect to any vector norm), we just say that $\mathbf{x}_n$ \emph{converges} to $\overline{\mathbf{x}}$ as $n \to \infty$ and write $\lim_{n \to \infty} {\mathbf{x}_n} = \overline{\mathbf{x}}$, without ambiguity regarding the type of convergence. 
\end{remark}

\subsection{Linear Least Squares and Tikhonov Regularization}

Let $m,k \in \NN$, $\mathbf{A} \in \RR^{m \times k}$ and $\mathbf{b} \in \RR^m$. The \emph{Linear Least Squares (LLSQ)} problem is formulated as follows
\begin{equation} \label{eq:LLSQ_problem} 
	\minimize_{\mathbf{x} \in \RR^k} \; F(\mathbf{x}) \defeq \twonorm{\mathbf{A} \mathbf{x} - \mathbf{b}}^2 = \sum_{i=1}^m \left( \sum_{j=1}^k {a_{i,j} x_j } - b_i \right)^2  . 
\end{equation}
This is an (unconstrained) convex optimization problem whose set of (globally) optimal solutions coincides with the set of solutions to the so-called \emph{normal equations}\footnote{For each convex optimization problem (constrained or unconstrained), every local optimum is a global optimum. In addition, given an unconstrained convex optimization problem with objective function $f(\mathbf{x})$ (convex if minimization, and concave if maximization problem), a necessary and sufficient condition for a vector $\mathbf{x} \in \RR^k$ to be a (globally) optimal solution is that it satisfies the condition $\nabla f(\mathbf{x}) \defeq \transp{[{\partial f(\mathbf{x})}/{\partial x_1},\dots,{\partial f(\mathbf{x})}/{\partial x_k}]} = \mathbf{0}$. The latter condition for the LLSQ problem~\eqref{eq:LLSQ_problem} can be written in matrix form as shown in \eqref{eq:Normal_equations}.}  
\begin{equation} \label{eq:Normal_equations}
	(\transp{\mathbf{A}} \mathbf{A}) \mathbf{x} = \transp{\mathbf{A}} \mathbf{b} .
\end{equation}
Note that there is always at least one solution that satisfies the normal equations. Specifically, such a system of linear equations has either a unique solution or infinitely many solutions.\footnote{A system of linear equations $\mathbf{B} \mathbf{x} = \mathbf{d}$ can only have: (i) an infinite number of solutions, (ii) a unique solution, or (iii) no solution. A necessary and sufficient condition for the existence of a solution is that $\mathbf{B} \pseudoinv{\mathbf{B}} \mathbf{d} = \mathbf{d}$. In the case of normal equations \eqref{eq:Normal_equations}, we have $\mathbf{B} \defeq \transp{\mathbf{A}} \mathbf{A}$ and $\mathbf{d} \defeq \transp{\mathbf{A}} \mathbf{b}$. By the properties of the pseudoinverse for real matrices: $\pseudoinv{(\transp{\mathbf{A}} \mathbf{A})} \transp{\mathbf{A}} = \pseudoinv{\mathbf{A}}$ and $\transp{\mathbf{A}} \mathbf{A} \pseudoinv{\mathbf{A}} = \transp{\mathbf{A}}$, we obtain $\mathbf{B} \pseudoinv{\mathbf{B}} \mathbf{d} = \transp{\mathbf{A}} \mathbf{A} \pseudoinv{(\transp{\mathbf{A}} \mathbf{A})} \transp{\mathbf{A}} \mathbf{b} = \transp{\mathbf{A}} \mathbf{A}  \pseudoinv{\mathbf{A}} \mathbf{b} = \transp{\mathbf{A}} \mathbf{b} = \mathbf{d}$. Since the condition holds, there exists at least one solution to \eqref{eq:Normal_equations}; thus only cases (i) and (ii) are possible.} The uniqueness of a solution is guaranteed when the columns of $\mathbf{A}$ are \emph{linearly independent} (equivalently, $\rank{\mathbf{A}} = k$) and so $m \geq k$. 

\begin{remark}
	The condition $\rank{\mathbf{A}} = k$ is equivalent to each of the following: \linebreak i) $\rank{\transp{\mathbf{A}} \mathbf{A}} = k$, ii) $\determ{\transp{\mathbf{A}} \mathbf{A}} \neq 0$ (equivalently, $\determ{\transp{\mathbf{A}} \mathbf{A}} > 0$ since $\transp{\mathbf{A}} \mathbf{A}$ is symmetric and positive semi-definite), iii) $\transp{\mathbf{A}} \mathbf{A}$ is positive definite, and iv) $\transp{\mathbf{A}} \mathbf{A}$ is invertible, i.e., $\inv{(\transp{\mathbf{A}} \mathbf{A})}$ exists.
\end{remark}

\begin{lemma}[Sufficient condition for unique solution to LLSQ] \label{lem:Sufficient_condition_for_unique_solution_LLSQ}
	Assume that the matrix $\mathbf{A} \in \RR^{m \times k}$ has full column rank, i.e., $\rank{\mathbf{A}} = k$. Then, there exists a unique (globally) optimal solution to the LLSQ problem \eqref{eq:LLSQ_problem} that is given by 
	\begin{equation} \label{eq:Unique_optimal_solution_to_LLSQ}
		\mathbf{x}^* = \inv{(\transp{\mathbf{A}} \mathbf{A})} \transp{\mathbf{A}} \mathbf{b} = \pseudoinv{\mathbf{A}} \mathbf{b}. 
	\end{equation}
	Moreover, the global minimum (i.e., the optimal objective value) is given by 
	\begin{equation} \label{eq:Global_minimum_LLSQ}
		F^* \defeq F(\mathbf{x}^*) = \twonorm{\left( \mathbf{A} \pseudoinv{\mathbf{A}} - \mathbf{I} \right) \mathbf{b}}^2   .
	\end{equation}
	If $m = k$,\footnote{In this case, the condition $\rank{\mathbf{A}} = k$ is equivalent to $\determ{\mathbf{A}} \neq 0$.} then \eqref{eq:Unique_optimal_solution_to_LLSQ} and \eqref{eq:Global_minimum_LLSQ} reduce to 
	\begin{equation*}
		\mathbf{x}^* = \inv{\mathbf{A}} \mathbf{b} \quad \mathrm{and} \quad F^* = 0  .
	\end{equation*}
\end{lemma}

The overall time complexity for the computation of $\mathbf{x}^*$ in \eqref{eq:Unique_optimal_solution_to_LLSQ} is $O(k^2 m + k^3 + k^2 m + k m) = O(k^2 m + k^3) = O(k^2 m)$, since $m \geq k$. In particular, the product $\transp{\mathbf{A}} \mathbf{A}$ requires $O(k^2 m)$ time, the inverse $\mathbf{B} \defeq \inv{(\transp{\mathbf{A}} \mathbf{A})}$ needs $O(k^3)$ time, the pseudoinverse $\pseudoinv{\mathbf{A}} = \mathbf{B} \transp{\mathbf{A}}$ can be computed in $O(k^2 m)$ time, and the final product $\pseudoinv{\mathbf{A}} \mathbf{b}$ has complexity $O(k m)$. Here we have assumed a \emph{uniform/unit-cost model} of computation, where basic arithmetic operations (e.g., addition and multiplication) take $O(1)$ time, regardless of the size of the numbers involved (i.e., the number of bits in their representation). In practice, this is the case for standard-precision floating-point numbers.\footnote{In other cases (e.g., arbitrary/variable-precision and symbolic arithmetic), the assumption of uniform cost is no longer true and more suitable models should be used. For example, in the \emph{logarithmic-cost model}, every basic arithmetic operation is assigned a cost that is proportional to the size of the numbers involved. Although the latter model of computation is more accurate, it is more difficult to use in general.}   

A method that is closely related to LLSQ is the \emph{Tikhonov regularization}~\cite{Tikhonov_1963}, also known as \emph{ridge regression} in statistics. The \emph{Tikhonov LLSQ (T-LLSQ)} problem  is formulated as follows 
\begin{equation} \label{eq:T-LLSQ_problem} 
	\minimize_{\mathbf{x} \in \RR^k} \; G(\mathbf{x}) \defeq \twonorm{\mathbf{A} \mathbf{x} - \mathbf{b}}^2 + \mu \twonorm{\mathbf{x}}^2  , 
\end{equation}
where $\mu \in \RR_{\geq 0}$ is called the regularization parameter; note that $\mu = 0$ corresponds to the standard LLSQ problem \eqref{eq:LLSQ_problem}. This (unconstrained) convex optimization problem can be reformulated as a classical LLSQ problem, that is,  
\begin{equation*}
	\minimize_{\mathbf{x} \in \RR^k} \; \twonorm{\mathbf{B} \mathbf{x} - \mathbf{d}}^2 ,
\end{equation*}
where $\mathbf{B} \defeq \big[ \begin{smallmatrix} \mathbf{A} \\ \sqrt{\mu} \;\! \mathbf{I} \end{smallmatrix} \big] \in \RR^{(m+k) \times k}$ and $\mathbf{d} \defeq \big[ \begin{smallmatrix} \mathbf{b} \\ \mathbf{0} \end{smallmatrix} \big] \in \RR^{m+k}$. Based on \eqref{eq:Normal_equations}, the \emph{regularized normal equations} are expressed as   
\begin{equation*}
	(\transp{\mathbf{B}} \mathbf{B}) \mathbf{x} = \transp{\mathbf{B}} \mathbf{d} \iff (\transp{\mathbf{A}} \mathbf{A} + \mu \mathbf{I}) \mathbf{x} = \transp{\mathbf{A}} \mathbf{b}   .
\end{equation*}
Observe that the matrix $\transp{\mathbf{A}} \mathbf{A} + \mu \mathbf{I}$ is symmetric and positive definite whenever $\mu > 0$, because $\transp{\mathbf{x}} (\transp{\mathbf{A}} \mathbf{A} + \mu \mathbf{I}) \mathbf{x} = \twonorm{\mathbf{A} \mathbf{x}}^2 + \mu \twonorm{\mathbf{x}}^2 > 0$ for all $\mathbf{x} \in \RR^k \setminus \{\mathbf{0}\}$; note that $\twonorm{\mathbf{x}} > 0$ if and only if $\mathbf{x} \neq \mathbf{0}$. Equivalently, all its eigenvalues are positive and so the matrix is invertible.\footnote{Recall that the determinant of a matrix equals the product of all its eigenvalues, and therefore a matrix is invertible if and only if all its eigenvalues are nonzero.} As a consequence, we have the following useful result.  

\begin{lemma}[Sufficient condition for unique solution to T-LLSQ] \label{lem:Sufficient_condition_for_unique_solution_T-LLSQ}
	Suppose that the regularization parameter $\mu$ is positive, that is, $\mu \in \RR_{>0}$. Then, the matrix $\transp{\mathbf{A}} \mathbf{A} + \mu \mathbf{I}$ is invertible and there exists a unique (globally) optimal solution to the T-LLSQ problem \eqref{eq:T-LLSQ_problem} which is given by 
	\begin{equation} \label{eq:Unique_optimal_solution_to_T-LLSQ}
		\mathbf{x}^* = \inv{(\transp{\mathbf{A}} \mathbf{A} + \mu \mathbf{I})} \transp{\mathbf{A}} \mathbf{b} = \mathbf{C} \mathbf{b} , 
	\end{equation}
	where $\mathbf{C} \defeq \inv{(\transp{\mathbf{A}} \mathbf{A} + \mu \mathbf{I})} \transp{\mathbf{A}} = \mathbf{D} \transp{\mathbf{A}} \in \RR^{k \times m}$ and $\mathbf{D} \defeq \inv{(\transp{\mathbf{A}} \mathbf{A} + \mu \mathbf{I})} \in \RR^{k \times k}$. Furthermore, the global minimum is expressed as 
	\begin{equation*}
		G^* \defeq G(\mathbf{x}^*) = \twonorm{\left( \mathbf{A} \mathbf{C} - \mathbf{I} \right) \mathbf{b}}^2 + \mu \twonorm{\mathbf{C} \mathbf{b}}^2 .
	\end{equation*}
\end{lemma}

As in the case of LLSQ, the overall time complexity to compute $\mathbf{x}^*$ in \eqref{eq:Unique_optimal_solution_to_T-LLSQ} is $O(k^2 m + k^3 + k^2 m + k m) = O(k^2 m + k^3) = O(k^2 \max(m,k))$, since the matrix $\transp{\mathbf{A}} \mathbf{A} + \mu \mathbf{I}$ requires $O(k^2 m + k^2) = O(k^2 m)$ time. 

\begin{remark}
	The uniqueness of optimal solution in Lemma~\ref{lem:Sufficient_condition_for_unique_solution_T-LLSQ} does \emph{not} require any assumption on the matrix $\mathbf{A}$, in contrast to Lemma~\ref{lem:Sufficient_condition_for_unique_solution_LLSQ} that requires linear independence for the columns of $\mathbf{A}$ (i.e., a rank assumption). In essence, the problem reformulation via Tikhonov regularization achieves uniqueness of the optimal solution, even when the original LLSQ problem (without regularization) has infinitely many optimal solutions. 
\end{remark}

Additional details on LLSQ can be found in the books of Boyd~and~Vandenberghe~\cite[\S{1.2.1},~\S{6.1.1}]{Boyd-Vandenberghe_2004}, Golub~and~Van~Loan~\cite[\S{5}]{Golub-VanLoan_1996}, and Higham~\cite[\S{20}]{Higham_2002}. See also \cite[\S{6.3.2}]{Boyd-Vandenberghe_2004} and \cite[\S{12.1}]{Golub-VanLoan_1996} about regularization techniques.

\subsection{Asymptotic Expansions}

In order to define an asymptotic expansion for a given function, we first need the concept of asymptotic scale. The following notions are based on the books by de~Bruijn~\cite[\S{1.5}]{deBruijn_1958}, Temme~\cite[\S{1.4}, Definitions~1.2--1.3]{Temme_2015}, and Wong~\cite[\S{I.3}, Definitions~2--3]{Wong_2001}. 

\begin{definition}[Asymptotic scale]
	Let $\{g_l\}_{l \in \NN_0}$ be a sequence of real-valued functions defined on $\NN$. We say that such a sequence is a \emph{decreasing} (resp. an \emph{increasing}) \emph{asymptotic scale} as $n \to \infty$ if, and only if, it satisfies
	\begin{equation} \label{eq:Asymptotic_scale}
		g_{l+1}(n) = o(g_l(n)) \quad \left( \mathrm{resp.}\ g_l(n) = o(g_{l+1}(n)) \right),  \quad \mathrm{as}\ n \to \infty ,
	\end{equation}
	for every $l \in \NN_0$.
\end{definition}

\begin{definition}[Poincar{\'e}-type asymptotic expansion/series]
	Let $f \colon \NN \to \RR$, \linebreak $\{g_l\}_{l \in \NN_0}$ be a decreasing asymptotic scale, and $\{\beta_l\}_{l \in \NN_0}$ be a sequence of real numbers. Given a $p \in \NN_0$, we say that $\sum_{l=0}^p {\beta_l g_l(n)}$ is a \emph{$p$-order asymptotic expansion/series} of $f$ with respect to the decreasing asymptotic scale $\{g_l\}_{l \in \NN_0}$ as $n \to \infty$ if, and only if,
	\begin{equation} \label{eq:p-order_asymptotic_expansion}
		f(n) = \sum_{l=0}^p {\beta_l g_l(n)} + o(g_p(n)) ,   \quad \mathrm{as}\ n \to \infty .
	\end{equation}
	If the above condition holds for every fixed $p \in \NN_0$, then we write 
	\begin{equation} \label{eq:Complete_asymptotic_expansion}
		f(n) \simeq \sum_{l=0}^{\infty} {\beta_l g_l(n)} , \quad \mathrm{as}\ n \to \infty ,
	\end{equation}
	where the right-hand side is called a \emph{complete asymptotic expansion/series} for the function $f$.
\end{definition}

Note that \eqref{eq:p-order_asymptotic_expansion} is sometimes (e.g., in \cite[\S{1.5}]{deBruijn_1958}) replaced by  
\begin{equation} \label{eq:p-order_asymptotic_expansion_stronger_condition}
	f(n) = \sum_{l=0}^p {\beta_l g_l(n)} + O(g_{p+1}(n)), \quad \mathrm{as}\ n \to \infty ,
\end{equation}
which is a stronger condition than \eqref{eq:p-order_asymptotic_expansion}, because if $h(n) = O(g_{p+1}(n))$ then \linebreak $h(n) = o(g_p(n))$ due to the decreasing asymptotic scale \eqref{eq:Asymptotic_scale} for $l = p$.

\begin{remark}
	If $\sum_{l=0}^p {\beta_l g_l(n)}$ is a $p$-order asymptotic expansion of a function $f$, for some $p \in \NN_0$, then $\sum_{l=0}^{p'} {\beta_l g_l(n)}$ is a $p'$-order asymptotic expansion of the same function, for all $p' \in \{0,\dots,p\}$. This can be easily proved by induction on $p'$, starting from $p' = p$ and working backwards until $p' = 0$, since the decreasing asymptotic scale $\eqref{eq:Asymptotic_scale}$ yields: $\beta_{p'} g_{p'}(n) + o(g_{p'}(n)) = o(g_{p'-1}(n))$, for all $p' \in \{1,\dots,p\}$. The same conclusion can be drawn if we use the stronger condition~\eqref{eq:p-order_asymptotic_expansion_stronger_condition}, because $\beta_{p'} g_{p'}(n) + O(g_{p'+1}(n)) = O(g_{p'}(n))$, again for all $p' \in \{1,\dots,p\}$. 
\end{remark}

\begin{lemma}[Uniqueness of coefficients in asymptotic expansions]
	Given a decreasing asymptotic scale $\{g_l\}_{l \in \NN_0}$, if \eqref{eq:p-order_asymptotic_expansion} or \eqref{eq:p-order_asymptotic_expansion_stronger_condition} holds for some $p \in \NN_0$, then the real coefficients $\{\beta_l\}_{l=0}^p$ are uniquely determined by 
	\begin{equation*}
		\beta_l = \lim_{n \to \infty} \frac{f(n) - \sum_{i=0}^{l-1} {\beta_i g_i(n)}}{g_l(n)}  ,
	\end{equation*}
	for all $l \in \{0,\dots,p\}$. If \eqref{eq:Complete_asymptotic_expansion} is also true, then the above expression holds for all $l \in \NN_0$. In particular, for $l=0$ we have $\beta_0 = \lim_{n \to \infty} \frac{f(n)}{g_0(n)}$. 
\end{lemma}

\section{Omitted Proofs}

\subsection{Proof of Theorem~\ref{thm:The_sLLSQ_method}} 
\label{subsec:Proof_of_the_sLLSQ_theorem}

\subsubsection*{Proof of Statement 1}

Given Assumption~\ref{assum:Asymptotic_rank_assumption}, the uniqueness of optimal solution $\mathbf{y}_n^*$ and its closed-form expression~\eqref{eq:y_n_star_sLLSQ}, for large enough $n$, result from Lemma~\ref{lem:Sufficient_condition_for_unique_solution_LLSQ}. Also, the uniqueness of $\mathbf{x}_n^*$ and \eqref{eq:x_n_star} follow from Proposition~\ref{prop:Equivalence_of_optimization_problems}.

\subsubsection*{Proof of Statement 2}

Regarding the global minimum $F_n^* \defeq F_n(\mathbf{x}_n^*) = F_n(\mathbf{u}^{-1}(\mathbf{y}_n^*))$, which satisfies $F_n^* \leq F_n(\mathbf{x})$ for all $\mathbf{x} \in \mathcal{S}$, we have 
\begin{equation*}
	0 \leq F_n^* \leq F_n(\boldsymbol{\alpha}) = \sum_{i=0}^{m-1} \left( \log\left( \frac{\widehat{f}(n+i;\boldsymbol{\alpha})}{f(n+i)} \right) \right)^2 = \sum_{i=0}^{m-1} \left( \log\left( \frac{f(n+i)}{\widehat{f}(n+i;\boldsymbol{\alpha})} \right) \right)^2 ,
\end{equation*}
because $\boldsymbol{\alpha} \in \mathcal{S}$ and due to~\eqref{eq:Loss_function_F_n}. Based on~\eqref{eq:Proven_asymptotic_expansion}, we deduce that 
\begin{equation*}
	\log\left( \frac{f(n+i)}{\widehat{f}(n+i;\boldsymbol{\alpha})} \right) = \log\left( 1 + O(g_1(n+i)) \right) =  O(g_1(n+i)) 
\end{equation*}  
for all $i \in \{0,\dots,m-1\}$, since $g_1(n) = o(1)$, as $n \to \infty$, and $\log(1+x) = O(x)$, as $x \to 0$. Therefore, 
\begin{equation*}
	F_n(\boldsymbol{\alpha}) = \sum_{i=0}^{m-1} O(g_1^2(n+i)) = O\left( \sum_{i=0}^{m-1} g_1^2(n+i) \right) = O(\twonorm{\mathbf{g}_1}^2) = O(\norm{\mathbf{g}_1}^2) ,
\end{equation*}
due to the norm equivalence (Lemma~\ref{lem:Norm_equivalence}). By the continuity of vector norms and the square function, it also holds that $\lim_{n \to \infty} \norm{\mathbf{g}_1}^2 = \norm{\mathbf{0}}^2 = 0$, i.e., $\norm{\mathbf{g}_1}^2 = o(1)$, thus yielding \eqref{eq:F_n_star_asymptotic_estimate_sLLSQ}.\footnote{In fact, the global minimum is explicitly given by $F_n^* = \twonorm{\mathbf{A} \mathbf{y}_n^* - \mathbf{b}}^2 = \twonorm{\left( \mathbf{A} \pseudoinv{\mathbf{A}} - \mathbf{I} \right) \mathbf{b}}^2$. Based on this expression, however, it is \emph{not} easy to prove the asymptotic property of $F_n^*$.}  

In order to prove \eqref{eq:y_n_star_asymptotic_estimate_sLLSQ} and \eqref{eq:y_jn_star_asymptotic_estimate_sLLSQ}, recall that the vector $\boldsymbol{\gamma}$ is \emph{unique} by Proposition~\ref{prop:Uniqueness_of_alpha_and_gamma}, given Assumptions~\labelcref{assum:Asymptotic_expansion_and_form,assum:Functions_varphi_j,assum:Functions_u_j}. Based on the asymptotic expansion~\eqref{eq:Proven_asymptotic_expansion} and the asymptotic form~\eqref{eq:Proven_asymptotic_form_with_gamma}, we can write  
\begin{equation*}
	f(n) = \exp\left( \sum_{j=1}^k \varphi_j(n) \gamma_j \right) \left(1 + \sum_{l=1}^p \beta_l g_l(n) + O(g_{p+1}(n)) \right) .
\end{equation*}
Moreover, because $\log(1+x) = \Theta(x) = O(x)$ as $x \to 0$,\footnote{A complete asymptotic expansion of $\log(1+x)$ is given by: $\log(1+x) \simeq \sum_{\nu=1}^{\infty} {\frac{(-1)^{\nu+1}}{\nu} x^{\nu}}$, as $x \to 0$. In particular, its $\sigma$-order ($\sigma \in \NN_0$) asymptotic expansion is $\log(1+x) = \sum_{\nu=1}^{\sigma} {\frac{(-1)^{\nu+1}}{\nu} x^{\nu}} + O(x^{\sigma+1})$, as $x \to 0$.} we obtain 
\begin{align*}
	& \log\left(1 + \sum_{l=1}^p \beta_l g_l(n) + O(g_{p+1}(n)) \right)  \\
	= & \log\left(1 + \sum_{l=1}^p \beta_l g_l(n) \right) + \log\left(1 + \frac{O(g_{p+1}(n))}{1 + \sum_{l=1}^p \beta_l g_l(n)} \right)  \\
	= & \log\left(1 + \sum_{l=1}^p \beta_l g_l(n) \right) +  \log\left(1 + \frac{O(g_{p+1}(n))}{\Omega(1)} \right)     \\
	= & \log\left(1 + \sum_{l=1}^p \beta_l g_l(n) \right) + \log\left(1 + O(g_{p+1}(n)) \right)     \\
	= & \log\left(1 + \sum_{l=1}^p \beta_l g_l(n) \right) +  O(g_{p+1}(n))   ,
\end{align*}
as $n \to \infty$. Here, we have used the fact that $\lim_{n \to \infty} \left( 1 + \sum_{l=1}^p \beta_l g_l(n) \right) = 1$, which implies that $1 + \sum_{l=1}^p \beta_l g_l(n) = \Theta(1) = \Omega(1)$ by the limit definition. Therefore, the vector $\mathbf{b}$ in~\eqref{eq:Vector_b} can be written as 
\begin{equation} \label{eq:Vector_b_asymptotics}
	\mathbf{b} = \mathbf{A} \boldsymbol{\gamma} + \log\left( 1 + \sum_{l=1}^p {\beta_l \mathbf{g}_l} \right) + \mathbf{z}   ,
\end{equation}
where $\mathbf{g}_l$ is given by~\eqref{eq:Vector_g_l} and  
\begin{equation*} 
	\mathbf{z} = \mathbf{z}(n) = \transp{\left[ O( g_{p+1}(n) ),\dots,O( g_{p+1}(n+m-1) ) \right]}  \in \RR^m  
\end{equation*} 
or, equivalently, $z_i = z_i(n) = O(g_{p+1}(n+i-1))$ for all $i \in \{1,\dots,m\}$. As a result, due to the norm equivalence (Lemma~\ref{lem:Norm_equivalence}), we obtain 
\begin{align} \label{eq:Vector_z_asymptotics}
	\begin{split}
		\norm{\mathbf{z}} & = O(\maxnorm{\mathbf{z}}) = O\left( \max\limits_{0 \leq i \leq m-1} {O(g_{p+1}(n+i))} \right)  \\
		& = O\left( \max\limits_{0 \leq i \leq m-1} {\abs{g_{p+1}(n+i)}} \right) = O( \maxnorm{\mathbf{g}_{p+1}} )  \\
		& = O( \norm{\mathbf{g}_{p+1}} )  = o(1).
	\end{split}
\end{align} 	
By using \eqref{eq:Vector_b_asymptotics}, the optimal solution $\mathbf{y}_n^*$ is expressed as follows 
\begin{equation*}
	\mathbf{y}_n^* = \pseudoinv{\mathbf{A}} \mathbf{b} = \boldsymbol{\gamma} + \pseudoinv{\mathbf{A}} \log\left( 1 + \sum_{l=1}^p {\beta_l \mathbf{g}_l} \right) + \pseudoinv{\mathbf{A}} \mathbf{z} , 
\end{equation*}
since $\pseudoinv{\mathbf{A}} \mathbf{A} = \inv{(\transp{\mathbf{A}} \mathbf{A})} \transp{\mathbf{A}} \mathbf{A} = \mathbf{I}$. Due to the triangle inequality, Lemma~\ref{lem:Universal_sub-multiplicative_property}, and~\eqref{eq:Vector_z_asymptotics}, we conclude that 
\begin{align*}
	\norm{\mathbf{y}_n^* - \boldsymbol{\gamma}} & = \norm{\pseudoinv{\mathbf{A}} \log\left( 1 + \sum_{l=1}^p {\beta_l \mathbf{g}_l} \right) + \pseudoinv{\mathbf{A}} \mathbf{z}}   \\
	& \leq \norm{\pseudoinv{\mathbf{A}} \log\left( 1 + \sum_{l=1}^p {\beta_l \mathbf{g}_l} \right)} + \norm{\pseudoinv{\mathbf{A}} \mathbf{z}}  \\
	& = \norm{\pseudoinv{\mathbf{A}} \log\left( 1 + \sum_{l=1}^p {\beta_l \mathbf{g}_l} \right)} + O\left( \norm{\pseudoinv{\mathbf{A}}} \norm{\mathbf{z}} \right) \\
	& = \norm{\pseudoinv{\mathbf{A}} \log\left( 1 + \sum_{l=1}^p {\beta_l \mathbf{g}_l} \right)} + O\left( \norm{\pseudoinv{\mathbf{A}}} \norm{\mathbf{g}_{p+1}} \right)    \\
	& = O\left( \max\left( \norm{\pseudoinv{\mathbf{A}} \log\left( 1 + \sum_{l=1}^p {\beta_l \mathbf{g}_l} \right)}, \norm{\pseudoinv{\mathbf{A}}} \norm{\mathbf{g}_{p+1}} \right) \right)  \\
	& = O( \eta(n) )  ,
\end{align*}
thus establishing the overall asymptotic estimate~\eqref{eq:y_n_star_asymptotic_estimate_sLLSQ}. Note that all the hidden constants within $O$-terms are \emph{independent} of $n$ (they depend only on $m$ and $k$).

Since $y_{j,n}^* = \pseudoinv{\mathbf{A}}(j,:) \mathbf{b}$, we similarly deduce that  
\begin{align*}
	\abs{y_{j,n}^* - \gamma_j} & = O\left( \max\left( \abs{\pseudoinv{\mathbf{A}}(j,:) \log\left( 1 + \sum_{l=1}^p {\beta_l \mathbf{g}_l} \right)}, \norm{\pseudoinv{\mathbf{A}}(j,:)} \norm{\mathbf{g}_{p+1}} \right) \right)   \\
	& = O( \vartheta_j(n) )  ,  
\end{align*}
for all $j \in \{1,\dots,k\}$, thus proving the per-component asymptotic estimate~\eqref{eq:y_jn_star_asymptotic_estimate_sLLSQ}. The absolute value in $\abs{y_{j,n}^* - \gamma_j}$ can be omitted, because it is implicitly included in the $O$-notation.

\subsubsection*{Proof of Statement 3}

Recall that the vector $\boldsymbol{\alpha}$ is \emph{unique} according to Proposition~\ref{prop:Uniqueness_of_alpha_and_gamma}. Based on~\eqref{eq:y_jn_star_asymptotic_estimate_sLLSQ} and Proposition~\ref{prop:Asymptotic_invariants}, the implication~\eqref{eq:Implication_o_sLLSQ} follows from the convergence-rate invariant~\eqref{eq:Convergence-rate_invariant_y_to_x_with_h}, while the implication~\eqref{eq:Implication_O_sLLSQ} is a straightforward consequence of the boundedness invariant~\eqref{eq:Boundedness_invariant}. This concludes the proof.

\subsection{Proof of Proposition~\ref{prop:Impact_of_the_sliding_window_length}}
\label{subsec:Proof_of_proposition_for_sliding_window_length}

In order to prove this result, we first need a definition and two useful lemmas.  

\begin{definition}[Loewner's partial order of real matrices~{\cite[p.~469, Definition~7.7.1]{Horn-Johnson_1985}}]
	Let $k \in \NN$ and $\mathbf{A}, \mathbf{B} \in \RR^{k \times k}$ be symmetric matrices. We write $\mathbf{A} \succeq \mathbf{B}$ (resp. $\mathbf{A} \succ \mathbf{B}$) if, and only if, the matrix $\mathbf{A} - \mathbf{B}$ is positive semi-definite (resp. positive definite). 
\end{definition}

\begin{lemma}[{\cite[p.~398, Observation~7.1.3]{Horn-Johnson_1985}}] \label{lem:Sum_of_positive_definite_and_positive_semi-definite_matrices}
	Let $\mathbf{A}, \mathbf{B} \in \RR^{k \times k}$ be symmetric matrices. If $\mathbf{A} \succ \mathbf{O}$ and $\mathbf{B} \succeq \mathbf{O}$, then $\mathbf{A} + \mathbf{B} \succ \mathbf{O}$. 
\end{lemma}

\begin{lemma}[{\cite[p.~471, Corollary~7.7.4]{Horn-Johnson_1985}}] \label{lem:Properties_of positive_definite_matrices}
	Let $\mathbf{A}, \mathbf{B} \in \RR^{k \times k}$ be (symmetric) positive definite matrices, i.e., $\mathbf{A}, \mathbf{B} \succ \mathbf{O}$. Then $\inv{\mathbf{A}}$ and $\inv{\mathbf{B}}$ exist, and also $\inv{\mathbf{A}}, \inv{\mathbf{B}} \succ \mathbf{O}$.\footnote{Recall that if $\{\lambda_1,\dots,\lambda_k\}$ are the eigenvalues of an invertible matrix $\mathbf{A} \in \RR^{k \times k}$ (so $\lambda_j \neq 0$ for all $j \in \{1,\dots,k\}$), then the eigenvalues of its inverse, $\inv{\mathbf{A}}$, are $\left\{\frac{1}{\lambda_1},\dots,\frac{1}{\lambda_k}\right\}$.} In addition,  
	\begin{enumerate}[label={\arabic*.}]
		\item $\mathbf{A} \succeq \mathbf{B}$ $\iff$ $\inv{\mathbf{B}} \succeq \inv{\mathbf{A}}$.
		
		\item $\mathbf{A} \succeq \mathbf{B}$ $\implies$ $\lambda_j(\mathbf{A}) \geq \lambda_j(\mathbf{B})$, $\forall j \in \{1,\dots,k\}$ $\implies$ $\determ{\mathbf{A}} \geq \determ{\mathbf{B}}$ and $\tr{\mathbf{A}} \geq \tr{\mathbf{B}}$, where $\lambda_j(\cdot)$ is the $j$-th largest eigenvalue of the respective matrix, i.e., $\lambda_1(\cdot) \geq \cdots \geq \lambda_k(\cdot)$. 
	\end{enumerate}
\end{lemma}

\noindent Based on~\eqref{eq:Matrix_A}, the matrix $\mathbf{A}_{\mathrm{new}}$ can be written as
\begin{equation*}
	\mathbf{A}_{\mathrm{new}} = \mathbf{A}_{\mathrm{new}}(n) \defeq 
	\begin{bmatrix}
		\varphi_1(n) & \cdots & \varphi_k(n) \\ 
		\vdots & \ddots & \vdots \\
		\varphi_1(n+m'-1) & \cdots & \varphi_k(n+m'-1)
	\end{bmatrix}
	= \begin{bmatrix} \mathbf{A} \\ \mathbf{A}' \end{bmatrix} \in \RR^{m' \times k} ,
\end{equation*}
where 
\begin{equation*}
	\mathbf{A}' = \mathbf{A}'(n) \defeq 
	\begin{bmatrix}
		\varphi_1(n+m) & \cdots & \varphi_k(n+m) \\ 
		\vdots & \ddots & \vdots \\
		\varphi_1(n+m'-1) & \cdots & \varphi_k(n+m'-1)
	\end{bmatrix}
	\in \RR^{\nu \times k} 
\end{equation*} 
and $\NN \ni \nu \defeq m'-m$. Moreover, Assumption~\ref{assum:Asymptotic_rank_assumption} on the matrix $\mathbf{A}$ implies that the same is true for the matrix $\mathbf{A}_{\mathrm{new}}$, i.e., $\rank{\mathbf{A}} = \rank{\mathbf{A}_{\mathrm{new}}} = k$, for sufficiently large $n$. For the sake of brevity, we will omit the phrase ``for sufficiently large $n$'' in the rest of this proof; all statements hold under this condition.

\subsubsection*{Proof of \eqref{eq:O-estimate_on_A_new}} 

For convenience, we will first use the spectral norm $\twonorm{\cdot}$ for matrices, and then generalize the result for arbitrary matrix norms. In order to compare the largest singular values of $\pseudoinv{\mathbf{A}}$ and $\pseudoinv{\mathbf{A}_{\mathrm{new}}}$, observe that 
\begin{equation*}
	\pseudoinv{\mathbf{A}} \transp{(\pseudoinv{\mathbf{A}})} = \inv{(\transp{\mathbf{A}} \mathbf{A})} \transp{\mathbf{A}} \mathbf{A} \transp{(\inv{(\transp{\mathbf{A}} \mathbf{A})})}  = \mathbf{I} \inv{(\transp{(\transp{\mathbf{A}} \mathbf{A})})} = \inv{(\transp{\mathbf{A}} \mathbf{A})} .
\end{equation*}
In a similar way,  
\begin{equation*}
	\pseudoinv{\mathbf{A}_{\mathrm{new}}} \transp{(\pseudoinv{\mathbf{A}_{\mathrm{new}}})} = \inv{(\transp{\mathbf{A}_{\mathrm{new}}} \mathbf{A}_{\mathrm{new}})} .
\end{equation*}

Since $\transp{\mathbf{A}_{\mathrm{new}}} = \left[ \transp{\mathbf{A}}, \transp{(\mathbf{A}')} \right]$, we can write $\transp{\mathbf{A}_{\mathrm{new}}} \mathbf{A}_{\mathrm{new}} = \transp{\mathbf{A}} \mathbf{A} + \transp{(\mathbf{A}')} \mathbf{A}'$. Furthermore, it holds that $\transp{\mathbf{A}} \mathbf{A} \succ \mathbf{O}$ due to Assumption~\ref{assum:Asymptotic_rank_assumption}, while $\transp{(\mathbf{A}')} \mathbf{A}' \succeq \mathbf{O}$ because $\transp{\mathbf{x}} (\transp{(\mathbf{A}')} \mathbf{A}') \mathbf{x} = \transp{(\mathbf{A}' \mathbf{x})} (\mathbf{A}' \mathbf{x}) = \twonorm{\mathbf{A}' \mathbf{x}}^2 \geq 0$ for all $\mathbf{x} \in \RR^k$. By Lemma~\ref{lem:Sum_of_positive_definite_and_positive_semi-definite_matrices}, we conclude that $\transp{\mathbf{A}_{\mathrm{new}}} \mathbf{A}_{\mathrm{new}} \succ \mathbf{O}$. Now, by the first part of Lemma~\ref{lem:Properties_of positive_definite_matrices}, $\transp{\mathbf{A}} \mathbf{A}$ and $\transp{\mathbf{A}_{\mathrm{new}}} \mathbf{A}_{\mathrm{new}}$ are invertible, $\inv{(\transp{\mathbf{A}} \mathbf{A})}, \inv{(\transp{\mathbf{A}_{\mathrm{new}}} \mathbf{A}_{\mathrm{new}})} \succ \mathbf{O}$, and since $\transp{\mathbf{A}_{\mathrm{new}}} \mathbf{A}_{\mathrm{new}} \succeq \transp{\mathbf{A}} \mathbf{A}$, we have $\inv{(\transp{\mathbf{A}} \mathbf{A})} \succeq \inv{(\transp{\mathbf{A}_{\mathrm{new}}} \mathbf{A}_{\mathrm{new}})}$. Therefore, by the second part of Lemma~\ref{lem:Properties_of positive_definite_matrices}, we obtain 
\begin{equation*}
	\sigma_{\max}^2(\pseudoinv{\mathbf{A}}) = \lambda_{\max}(\inv{(\transp{\mathbf{A}} \mathbf{A})}) \geq \lambda_{\max}(\inv{(\transp{\mathbf{A}_{\mathrm{new}}} \mathbf{A}_{\mathrm{new}})}) = \sigma_{\max}^2(\pseudoinv{\mathbf{A}_{\mathrm{new}}})  
\end{equation*}
or, equivalently, 
\begin{equation*}
	\twonorm{\pseudoinv{\mathbf{A}}} \geq \twonorm{\pseudoinv{\mathbf{A}_{\mathrm{new}}}} .
\end{equation*}
Finally, by combining the above inequality with the norm equivalence in Lemma~\ref{lem:Norm_equivalence}, we deduce that  
\begin{equation*}
	\norm{\pseudoinv{\mathbf{A}_{\mathrm{new}}}} = O\left( \twonorm{\pseudoinv{\mathbf{A}_{\mathrm{new}}}} \right) = O\left( \twonorm{\pseudoinv{\mathbf{A}}} \right) = O\left( \norm{\pseudoinv{\mathbf{A}}} \right) ,
\end{equation*}  
for each matrix norm $\norm{\cdot}$ on $\RR^{k \times m'}$ and $\RR^{k \times m}$. It is emphasized that all the hidden constants in $O$-estimates are independent of $n$; they depend only on $k$, $m$ and $m'$.

\subsubsection*{Proof of \eqref{eq:Implication_on_y_n_new_star}} 

Let the vector $\mathbf{g}_{1,\mathrm{new}} = \mathbf{g}_{1,\mathrm{new}}(n) \in \RR^{m'}$ be similarly defined as $\mathbf{g}_1 = \mathbf{g}_1(n) \in \RR^m$ in~\eqref{eq:Vector_g_l}, that is, $\mathbf{g}_{1,\mathrm{new}} \defeq \transp{[g_1(n),\dots,g_1(n+m'-1)]}$. By the asymptotic regularity of $g_1$ and Lemma~\ref{lem:Asymptotic_regularity}, we obtain $\norm{\mathbf{g}_1} = \Theta(g_1(n))$ and $\norm{\mathbf{g}_{1,\mathrm{new}}} = \Theta(g_1(n))$ as $n \to \infty$.

Under Assumptions~\labelcref{assum:Asymptotic_expansion_and_form,assum:Functions_varphi_j,assum:Functions_u_j} and \ref{assum:Asymptotic_rank_assumption}/\ref{assum:Equivalent_asymptotic_rank_assumption}, the sLLSQ method is applicable according to Theorem~\ref{thm:The_sLLSQ_method}. In particular, its optimal solutions $\mathbf{y}_n^*$ and $\mathbf{y}_{n,\mathrm{new}}^*$, corresponding to $m$ and $m'$, satisfy the following estimates 
\begin{align*}
	\norm{\mathbf{y}_n^* - \boldsymbol{\gamma}} & = O\left( \norm{\pseudoinv{\mathbf{A}}} \norm{\mathbf{g}_1} \right) = O\left( \norm{\pseudoinv{\mathbf{A}}} g_1(n) \right)  ,  \\
	\norm{\mathbf{y}_{n,\mathrm{new}}^* - \boldsymbol{\gamma}} & = O\left( \norm{\pseudoinv{\mathbf{A}_{\mathrm{new}}}} \norm{\mathbf{g}_{1,\mathrm{new}}} \right) = O\left( \norm{\pseudoinv{\mathbf{A}_{\mathrm{new}}}} g_1(n) \right)  ,
\end{align*}
based on~\eqref{eq:y_n_star_asymptotic_estimate_sLLSQ} and Proposition~\ref{prop:Impact_of_the_asymptotics_order}. Consequently, if $\norm{\pseudoinv{\mathbf{A}}} g_1(n) = o(1)$, then $\norm{\mathbf{y}_n^* - \boldsymbol{\gamma}} = O\left( \norm{\pseudoinv{\mathbf{A}}} g_1(n) \right) = o(1)$ and $\norm{\mathbf{y}_{n,\mathrm{new}}^* - \boldsymbol{\gamma}} = O\left( \norm{\pseudoinv{\mathbf{A}_{\mathrm{new}}}} g_1(n) \right) = O\left( \norm{\pseudoinv{\mathbf{A}}} g_1(n) \right) = o(1)$ as $n \to \infty$, because of~\eqref{eq:O-estimate_on_A_new}. This completes the proof.

\subsection{Proof of Lemma~\ref{lem:Regularization_with_arbitrary_norm}}
\label{subsec:Proof_of_the_sT-LLSQ_lemma}

\subsubsection*{Proof of Statement 1}

Recall that every vector norm is continuous and convex. The composition of continuous functions is also continuous, provided that such a composition is well-defined. Every composition with an affine mapping preserves (continuity and) convexity \cite[\S{3.2.2}]{Boyd-Vandenberghe_2004}, thus $\twonorm{\mathbf{A} \mathbf{y} - \widetilde{\mathbf{b}}}$ is (continuous and) convex. Moreover, if $h$ is a convex and nonnegative function (e.g., a vector norm) and $\nu' \geq 1$, then $h^{\nu'}$ is also convex \cite[\S{3.2.4}, Example~3.13]{Boyd-Vandenberghe_2004}. Therefore, the objective function $R_n$ is \emph{continuous} and \emph{convex} as the sum of two such functions. Since $R_n$ is a convex function (which should be minimized) and $\RR^k$ is a convex set, \eqref{eq:Regularized_optimization_problem} is a \emph{convex optimization problem}. 

In addition, the objective function $R_n$ is \emph{coercive}, i.e., $\lim_{\norm{\mathbf{y}} \to \infty} R_n(\mathbf{y}) = \infty$, since it holds that
\begin{align*}
	& R_n(\mathbf{y}) \geq \mu \norm{\mathbf{y}}^{\nu} \implies  \infty = \liminf_{\norm{\mathbf{y}} \to \infty} \mu \norm{\mathbf{y}}^{\nu} \leq \liminf_{\norm{\mathbf{y}} \to \infty} R_n(\mathbf{y}) \left( \leq \limsup_{\norm{\mathbf{y}} \to \infty} R_n(\mathbf{y}) \leq \infty \right)   \\
	& \implies  \liminf_{\norm{\mathbf{y}} \to \infty} R_n(\mathbf{y}) = \limsup_{\norm{\mathbf{y}} \to \infty} R_n(\mathbf{y}) = \infty  \implies  \lim_{\norm{\mathbf{y}} \to \infty} R_n(\mathbf{y}) = \infty .
\end{align*}
Note that the particular norm in the limit (i.e., $\norm{\mathbf{y}} \to \infty$) does not matter because of the norm equivalence (Lemma~\ref{lem:Norm_equivalence}). For each $n \in \NN$, the existence of a global minimizer ($\mathbf{y}_n^*$) for problem~\eqref{eq:Regularized_optimization_problem} is guaranteed by the \emph{Weierstrass theorem}~\cite[Proposition~A.8]{Bertsekas_1999},\footnote{Let $\mathcal{D}$ be a nonempty and closed subset of $\RR^k$ and let $f \colon \mathcal{D} \to \RR$ be lower semicontinuous and coercive. Then, the optimization problem 
	\begin{equation*}
		\minimize_{\mathbf{y} \in \mathcal{D}} \; f(\mathbf{y})
	\end{equation*}
admits a globally optimal solution, i.e., there exists a vector $\mathbf{y}^* \in \mathcal{D}$ such that $f(\mathbf{y}^*) \leq f(\mathbf{y})$, for all $\mathbf{y} \in \mathcal{D}$.} because the objective function ($R_n$) is continuous---thus lower semicontinuous---and coercive, and the feasible set ($\RR^k$) is nonempty and closed.

Furthermore, the global minimum $R_n^* \defeq R_n(\mathbf{y}_n^*)$ has the property: $R_n^* \leq R_n(\mathbf{y})$, for all $\mathbf{y} \in \RR^k$. In particular, $\widetilde{\boldsymbol{\gamma}} \in \RR^k$ and therefore 
\begin{align*}
	0 \leq R_n^* \leq R_n(\widetilde{\boldsymbol{\gamma}}) & = \twonorm{\mathbf{A} \widetilde{\boldsymbol{\gamma}} - \widetilde{\mathbf{b}}}^2 + \mu \norm{\widetilde{\boldsymbol{\gamma}}}^{\nu} = \twonorm{\mathbf{w}}^2 + \mu \norm{\widetilde{\boldsymbol{\gamma}}}^{\nu}  \\
	& =  O(1) + \mu \norm{\widetilde{\boldsymbol{\gamma}}}^{\nu} = O(1) , \quad \mathrm{as}\ n \to \infty .
\end{align*}
Moreover, since $R_n^* = R_n(\mathbf{y}_n^*) \geq \mu \norm{\mathbf{y}_n^*}^{\nu}$, we obtain $\norm{\mathbf{y}_n^*} \leq \left( \frac{1}{\mu} R_n^* \right)^{1/\nu} = O(1)$; note that the function $h(x) \defeq x^{1/\nu}$, with $\nu \geq 1$, is increasing for $x \geq 0$. As a result, by the triangle inequality, 
\begin{equation*}
	\norm{\mathbf{y}_n^* - \widetilde{\boldsymbol{\gamma}}} \leq \norm{\mathbf{y}_n^*} + \norm{\widetilde{\boldsymbol{\gamma}}} = O(1) 
\end{equation*}
and, by the norm equivalence (Lemma~\ref{lem:Norm_equivalence}),
\begin{equation*}
	\abs{y_{j,n}^* - \widetilde{\gamma}_j} \leq \maxnorm{\mathbf{y}_n^* - \widetilde{\boldsymbol{\gamma}}} = O(\norm{\mathbf{y}_n^* - \widetilde{\boldsymbol{\gamma}}}) = O(1) , \quad \forall j \in \{1,\dots,k\} .
\end{equation*}
We can also deduce that 
\begin{align*}
	R_n^* = R_n(\mathbf{y}_n^*) & \geq \twonorm{\mathbf{A} \mathbf{y}_n^* - \widetilde{\mathbf{b}}}^2 = \twonorm{\mathbf{A} (\mathbf{y}_n^* - \widetilde{\boldsymbol{\gamma}}) - \mathbf{w}}^2   \\
	& \geq \left( \mathbf{A}(1,:) (\mathbf{y}_n^* - \widetilde{\boldsymbol{\gamma}}) - w_1 \right)^2   \\ 
	& = \left( \sum_{j=1}^k {\varphi_j(n) (y_{j,n}^* - \widetilde{\gamma}_j)} - w_1 \right)^2  . 
\end{align*}
Due to the fact that $R_n^* = O(1)$ and the square-root function is increasing, we conclude that 
\begin{equation*}
	\sum_{j=1}^k {\varphi_j(n) (y_{j,n}^* - \widetilde{\gamma}_j)} - w_1 = O(1)  \implies  \sum_{j=1}^k {\varphi_j(n) (y_{j,n}^* - \widetilde{\gamma}_j)} = O(1) . 
\end{equation*}
Here, we have used the property: $\abs{w_1} \leq \maxnorm{\mathbf{w}} = O(\norm{\mathbf{w}}) = O(1)$, that is, $w_1 = O(1)$. The above expression can be rearranged as follows\footnote{Division by $\varphi_k(n)$ is asymptotically possible, since $\varphi_k(n) = \Omega(\varphi_1(n)) = \Omega(1)$, as $n \to \infty$, by Assumption~\ref{assum:Functions_varphi_j}, and therefore $\varphi_k(n)$ is nonzero for sufficiently large $n$.} 
\begin{align*}
	y_{k,n}^* - \widetilde{\gamma}_k & = \sum_{j=1}^{k-1} {\frac{\varphi_j(n)}{\varphi_k(n)} (\widetilde{\gamma}_j - y_{j,n}^*)} + O\left( \frac{1}{\varphi_k(n)} \right)   \\
	& = \sum_{j=1}^{k-1} {O\left( \frac{\varphi_j(n)}{\varphi_k(n)} \right)} + O\left( \frac{1}{\varphi_k(n)} \right)  ,
\end{align*}
due to the fact $y_{j,n}^* - \widetilde{\gamma}_j = O(1)$ for all $j \in \{1,\dots,k-1\}$. Now, we distinguish two cases for the above expression. Firstly, when $k = 1$ the sum $\sum_{j=1}^{k-1} = \sum_{j=1}^0 = 0$, thus yielding 
\begin{equation*}
	y_{1,n}^* - \widetilde{\gamma}_1 = O\left( \frac{1}{\varphi_1(n)} \right) = o(1) ,
\end{equation*}
where the last equality follows from Assumption~\ref{assum:sT-LLSQ_when_k_equals_1}. Secondly, when $k \geq 2$ we get 
\begin{align*}
	y_{k,n}^* - \widetilde{\gamma}_k & = O\left( \frac{\varphi_{k-1}(n)}{\varphi_k(n)} \right) + O\left( \frac{1}{\varphi_k(n)} \right)  \\
	& = O\left( \frac{\varphi_{k-1}(n)}{\varphi_k(n)} \right) = o(1)  ,
\end{align*}
because of the increasing asymptotic scale $\{\varphi_j\}_{j=1}^k$ and the fact that $1 = O(\varphi_1(n)) = O(\varphi_{k-1}(n))$, according to Assumption~\ref{assum:Functions_varphi_j}. Finally, by combining these two cases with the convention $\varphi_0(n) \defeq 1$, we obtain the universal convergence rate \eqref{eq:Universal_convergence_rate_for_index_k}.

\subsubsection*{Proof of Statement 2}

The uniqueness and closed-form expression of the optimal solution follows from 	Lemma~\ref{lem:Sufficient_condition_for_unique_solution_T-LLSQ}. Moreover, concerning the asymptotics of matrices $\mathbf{C}$ and $\mathbf{D}$, a direct approach is to use their explicit definitions in \eqref{eq:Matrix_C} and \eqref{eq:Matrix_D}, which involve the inverse of a matrix. However, such an approach is quite cumbersome. In order to overcome this difficulty, we will instead give an \emph{indirect}~proof, based on the asymptotics of the optimal solution $\mathbf{y}_n^*$. 

In particular, using \eqref{eq:Matrix_C} and \eqref{eq:Matrix_D}, the optimal solution can be written as 
\begin{align*}
	\mathbf{y}_n^* & = \mathbf{C} \widetilde{\mathbf{b}} = \mathbf{C} \left( \mathbf{A} \widetilde{\boldsymbol{\gamma}} + \mathbf{w} \right) = \mathbf{D} \transp{\mathbf{A}} \mathbf{A} \widetilde{\boldsymbol{\gamma}} + \mathbf{C} \mathbf{w}  \\
	& = \mathbf{D} ( \transp{\mathbf{A}} \mathbf{A} + \mu \mathbf{I} ) \widetilde{\boldsymbol{\gamma}} + \mathbf{C} \mathbf{w} - \mu \mathbf{D} \widetilde{\boldsymbol{\gamma}}   \\
	& = \widetilde{\boldsymbol{\gamma}} + \mathbf{C} \mathbf{w} - \mu \mathbf{D} \widetilde{\boldsymbol{\gamma}}  ,
\end{align*}
which implies that 
\begin{equation} \label{eq:y_n_star_minus_gamma_with_C_and_D}
	\mathbf{y}_n^* - \widetilde{\boldsymbol{\gamma}} = \mathbf{C} \mathbf{w} - \mu \mathbf{D} \widetilde{\boldsymbol{\gamma}} .
\end{equation}
From the established asymptotics of $\mathbf{y}_n^*$, we know that $y_{j,n}^* - \widetilde{\gamma}_j = O(1)$, for all $j \in \{1,\dots,k\}$, and $y_{k,n}^* - \widetilde{\gamma}_k = O\left( \frac{\varphi_{k-1}(n)}{\varphi_k(n)} \right) = o(1)$. Therefore, by exploiting~\eqref{eq:y_n_star_minus_gamma_with_C_and_D}, we have 
\begin{align} \label{eq:C_and_D_asymptotic_relations}
	\begin{split}
		\mathbf{C}(j,:) \mathbf{w} - \mu \mathbf{D}(j,:) \widetilde{\boldsymbol{\gamma}} = O(1) , \quad \forall j \in \{1,\dots,k\}  ,  \\
		\mathbf{C}(k,:) \mathbf{w} - \mu \mathbf{D}(k,:) \widetilde{\boldsymbol{\gamma}} = O\left( \frac{\varphi_{k-1}(n)}{\varphi_k(n)} \right) = o(1)  .
	\end{split}
\end{align}

A crucial observation here is that $\widetilde{\boldsymbol{\gamma}}$ and $\mathbf{w}$ are particular but \emph{arbitrarily chosen} vectors. As a consequence, all the results hold for \emph{every} $\widetilde{\boldsymbol{\gamma}} \in \RR^k$ and \emph{every} $\mathbf{w} = \mathbf{w}(n) \in \RR^m$ such that $\norm{\mathbf{w}} = O(1)$ as $n \to \infty$. Observe that the matrices $\mathbf{C}$ and~$\mathbf{D}$ are \emph{independent} of $\widetilde{\boldsymbol{\gamma}}$ and $\mathbf{w}$; they depend only on the matrix $\mathbf{A}$ (which contains the functions $\{\varphi_j\}_{j=1}^k$) and the regularization parameter $\mu$. Next we proceed as follows, where we denote by $\boldsymbol{\varepsilon}_i$ the standard unit vector (of appropriate dimension), i.e., the vector that has one in its $i$-th position and zeros elsewhere. Note that $\norm{\boldsymbol{\varepsilon}_i} = \Theta(\twonorm{\boldsymbol{\varepsilon}_i}) = \Theta(1)$, due to the norm equivalence (Lemma~\ref{lem:Norm_equivalence}). 

Firstly, by setting $\widetilde{\boldsymbol{\gamma}} = \mathbf{0}$ and successively $\mathbf{w} = \boldsymbol{\varepsilon}_i$, for each $i \in \{1,\dots,m\}$, in~\eqref{eq:C_and_D_asymptotic_relations}, we obtain 
\begin{align*}
	c_{j,i} = c_{j,i}(n) = O(1) , \quad \forall j \in \{1,\dots,k\}, \ \forall i \in \{1,\dots,m\} ,  \\
	c_{k,i} = c_{k,i}(n) = O\left( \frac{\varphi_{k-1}(n)}{\varphi_k(n)} \right) = o(1) , \quad \forall i \in \{1,\dots,m\} .  
\end{align*}
Consequently, 
\begin{equation*}
	\norm{\mathbf{C}} = O(\maxnorm{\mathbf{C}}) = O(1) , \quad \norm{\mathbf{C}(j,:)} = O(\maxnorm{\mathbf{C}(j,:)}) = O(1) ,
\end{equation*}
for all $j \in \{1,\dots,k\}$, and 
\begin{equation*}
	\norm{\mathbf{C}(k,:)} = O(\maxnorm{\mathbf{C}(k,:)}) = O\left( \frac{\varphi_{k-1}(n)}{\varphi_k(n)} \right) = o(1) .
\end{equation*}

Secondly, by choosing $\mathbf{w} = \mathbf{0}$ and successively $\widetilde{\boldsymbol{\gamma}} = \boldsymbol{\varepsilon}_{j'}$, for each $j' \in \{1,\dots,k\}$, \eqref{eq:C_and_D_asymptotic_relations} gives (recall that $\mu \neq 0$) 
\begin{align*}
	d_{j,j'} = d_{j,j'}(n) = O(1) , \quad \forall j \in \{1,\dots,k\}, \ \forall j' \in \{1,\dots,k\} ,  \\
	d_{k,j'} = d_{k,j'}(n) = O\left( \frac{\varphi_{k-1}(n)}{\varphi_k(n)} \right) = o(1) , \quad \forall j' \in \{1,\dots,k\} .  
\end{align*}
As a result, 
\begin{equation*}
	\norm{\mathbf{D}} = O(\maxnorm{\mathbf{D}}) = O(1) , \quad \norm{\mathbf{D}(j,:)} = O(\maxnorm{\mathbf{D}(j,:)}) = O(1) ,
\end{equation*}
for all $j \in \{1,\dots,k\}$, and 
\begin{equation*}
	\norm{\mathbf{D}(k,:)} = O(\maxnorm{\mathbf{D}(k,:)}) = O\left( \frac{\varphi_{k-1}(n)}{\varphi_k(n)} \right) = o(1) .
\end{equation*} 
This concludes the proof.

\subsection{Proof of Theorem~\ref{thm:The_sT-LLSQ_method}} 
\label{subsec:Proof_of_the_sT-LLSQ_theorem}

\subsubsection*{Proof of Statement 1}

Direct consequence of Lemma~\ref{lem:Sufficient_condition_for_unique_solution_T-LLSQ}.

\subsubsection*{Proof of Statement 2}

Remember that the vector $\boldsymbol{\gamma}$ is unique by Proposition~\ref{prop:Uniqueness_of_alpha_and_gamma}. Now, given $\boldsymbol{\gamma}$, $\mathbf{A}$ and $\mathbf{b}$, let us first choose $\widetilde{\boldsymbol{\gamma}} \defeq \boldsymbol{\gamma}$ and $\mathbf{w} \defeq \mathbf{b} - \mathbf{A} \boldsymbol{\gamma}$, and then $\widetilde{\mathbf{b}} \defeq \mathbf{A} \widetilde{\boldsymbol{\gamma}} + \mathbf{w} = \mathbf{b}$. According to \eqref{eq:Vector_b_asymptotics}, the vector $\mathbf{w}$ is expressed as $\mathbf{w} = \log\left( 1 + \sum_{l=1}^p {\beta_l \mathbf{g}_l} \right) + \mathbf{z}$, thus satisfying $\norm{\mathbf{w}} = o(1) = O(1)$ as $n \to \infty$. Under these choices of $\widetilde{\boldsymbol{\gamma}}$, $\mathbf{w}$ and $\widetilde{\mathbf{b}}$, problem~\eqref{eq:Regularized_optimization_problem} with the substitution $\norm{\mathbf{y}}^{\nu} \mapsto  \twonorm{\mathbf{y}}^2$ coincides with problem~\eqref{eq:sT-LLSQ_problem}. Therefore, Lemma~\ref{lem:Regularization_with_arbitrary_norm} implies that the global minimum is asymptotically bounded, i.e., $G_n^* = O(1)$ as $n \to \infty$, thus proving~\eqref{eq:G_n_star_asymptotic_estimate_sT-LLSQ}.

In addition, based on~\eqref{eq:Vector_b_asymptotics}, the optimal solution in~\eqref{eq:y_n_star_sT-LLSQ} can be written as 
\begin{align} \label{eq:y_n_star_sT-LLSQ_rearranged}
	\begin{split}
		\mathbf{y}_n^* & = \mathbf{C} \mathbf{b} = \mathbf{C} \left( \mathbf{A} \boldsymbol{\gamma} + \log\left( 1 + \sum_{l=1}^p {\beta_l \mathbf{g}_l} \right) + \mathbf{z} \right)  \\
		& = \mathbf{D} ( \transp{\mathbf{A}} \mathbf{A} + \mu \mathbf{I} ) \boldsymbol{\gamma} - \mu \mathbf{D} \boldsymbol{\gamma} + \mathbf{C} \log\left( 1 + \sum_{l=1}^p {\beta_l \mathbf{g}_l} \right) + \mathbf{C} \mathbf{z}   \\
		& = \boldsymbol{\gamma} + \mathbf{C} \log\left( 1 + \sum_{l=1}^p {\beta_l \mathbf{g}_l} \right) + \mathbf{C} \mathbf{z} - \mu \mathbf{D} \boldsymbol{\gamma} . 
	\end{split} 		
\end{align}
By the triangle inequality, Lemma~\ref{lem:Universal_sub-multiplicative_property}, and~\eqref{eq:Vector_z_asymptotics}, we get 
\begin{align*}
	\norm{\mathbf{y}_n^* - \boldsymbol{\gamma}} & = \norm{\mathbf{C} \log\left( 1 + \sum_{l=1}^p {\beta_l \mathbf{g}_l} \right) + \mathbf{C} \mathbf{z} - \mu \mathbf{D} \boldsymbol{\gamma}}   \\
	& \leq \norm{\mathbf{C} \log\left( 1 + \sum_{l=1}^p {\beta_l \mathbf{g}_l} \right)} + \norm{\mathbf{C} \mathbf{z}} + \mu \norm{\mathbf{D} \boldsymbol{\gamma}}   \\ 
	& = \norm{\mathbf{C} \log\left( 1 + \sum_{l=1}^p {\beta_l \mathbf{g}_l} \right)} + O(\norm{\mathbf{C}} \norm{\mathbf{z}})  + O(\norm{\mathbf{D}} \norm{\boldsymbol{\gamma}})  \\
	& = \norm{\mathbf{C} \log\left( 1 + \sum_{l=1}^p {\beta_l \mathbf{g}_l} \right)} + O(\norm{\mathbf{C}}  \norm{\mathbf{g}_{p+1}}) + O(\norm{\mathbf{D}})   \\
	& = O\left( \max\left( \norm{\mathbf{C} \log\left( 1 + \sum_{l=1}^p {\beta_l \mathbf{g}_l} \right)}, \norm{\mathbf{C}} \norm{\mathbf{g}_{p+1}} , \norm{\mathbf{D}}  \right) \right)  \\
	& = O(\chi(n))  . 
\end{align*} 
Note that $\norm{\boldsymbol{\gamma}} = O(1)$, since $\boldsymbol{\gamma} \in \RR^k$. Similarly, for each $j \in \{1,\dots,k\}$, we have 
\begin{align*}
	\abs{y_{j,n}^* - \gamma_j} & = O\left( \max\left( \abs{\mathbf{C}(j,:) \log\left( 1 + \sum_{l=1}^p {\beta_l \mathbf{g}_l} \right)}, \norm{\mathbf{C}(j,:)}  \norm{\mathbf{g}_{p+1}} , \norm{\mathbf{D}(j,:)}  \right) \right)  \\
	& = O(\psi_j(n))  , 
\end{align*}
because $y_{j,n}^* = \mathbf{C}(j,:) \mathbf{b} = \mathbf{D}(j,:) \transp{\mathbf{A}} \mathbf{b}$. The absolute value in $\abs{y_{j,n}^* - \gamma_j}$ can be omitted, since it is implicitly included in the $O$-notation.  

Furthermore, observe that $\norm{\log\left( 1 + \sum_{l=1}^p {\beta_l \mathbf{g}_l} \right)} = o(1) = O(1)$ and $\norm{\mathbf{g}_{p+1}} = o(1) = O(1)$ due to the continuity of vector norms. Therefore, by Lemma~\ref{lem:Universal_sub-multiplicative_property} and the asymptotics of matrices $\mathbf{C}$ and $\mathbf{D}$ in Lemma~\ref{lem:Regularization_with_arbitrary_norm}, we obtain  
\begin{align*}
	\chi(n) & = \max\left( \norm{\mathbf{C} \log\left( 1 + \sum_{l=1}^p {\beta_l \mathbf{g}_l} \right)}, \norm{\mathbf{C}} \norm{\mathbf{g}_{p+1}} , \norm{\mathbf{D}}  \right)  \\
	& =  O\left( \max\left( \norm{\mathbf{C}} \norm{\log\left( 1 + \sum_{l=1}^p {\beta_l \mathbf{g}_l} \right)}, \norm{\mathbf{C}} \norm{\mathbf{g}_{p+1}} , \norm{\mathbf{D}}  \right) \right)  \\
	& = O\left( \max\left( \norm{\mathbf{C}}, \norm{\mathbf{D}}  \right) \right) = O(1)   ,
\end{align*}
and 	
\begin{align*}
	\psi_j(n) & =  \max\left( \abs{\mathbf{C}(j,:) \log\left( 1 + \sum_{l=1}^p {\beta_l \mathbf{g}_l} \right)}, \norm{\mathbf{C}(j,:)} \norm{\mathbf{g}_{p+1}} , \norm{\mathbf{D}(j,:)}  \right)  \\ 
	& = O\left( \max\left( \norm{\mathbf{C}(j,:)} \norm{\log\left( 1 + \sum_{l=1}^p {\beta_l \mathbf{g}_l} \right)}, \norm{\mathbf{C}(j,:)} \norm{\mathbf{g}_{p+1}} , \norm{\mathbf{D}(j,:)}  \right) \right)  \\
	& = O\left( \max\left( \norm{\mathbf{C}(j,:)}, \norm{\mathbf{D}(j,:)}  \right) \right) = O(1) ,
\end{align*}
for all $j \in \{1,\dots,k\}$. In particular, for $j=k$,
\begin{equation*}
	\psi_k(n) = O\left( \max\left( \norm{\mathbf{C}(k,:)}, \norm{\mathbf{D}(k,:)}  \right) \right) = O\left( \frac{\varphi_{k-1}(n)}{\varphi_k(n)} \right) = o(1)  . 
\end{equation*}
The combination of the above results gives the asymptotic estimates~\labelcref{eq:y_n_star_asymptotic_estimate_sT-LLSQ,eq:Function_chi,eq:y_jn_star_asymptotic_estimate_sT-LLSQ,eq:Function_psi_j,eq:y_kn_star_asymptotic_estimate_sT-LLSQ}.

\subsubsection*{Proof of Statement 3}

Recall that the vector $\boldsymbol{\alpha}$ is unique by Proposition~\ref{prop:Uniqueness_of_alpha_and_gamma}. Based on~\eqref{eq:y_jn_star_asymptotic_estimate_sT-LLSQ},~\eqref{eq:y_kn_star_asymptotic_estimate_sT-LLSQ} and Proposition~\ref{prop:Asymptotic_invariants}, the estimate~\eqref{eq:x_jn_star_asymptotic_estimate_sT-LLSQ} follows from the boundedness invariant~\eqref{eq:Boundedness_invariant}, while the implication~\eqref{eq:Implication_o_sT-LLSQ} and the estimate~\eqref{eq:x_kn_star_asymptotic_estimate_sT-LLSQ} result from the convergence-rate invariant~\eqref{eq:Convergence-rate_invariant_y_to_x_with_h}. This completes the proof.

}

\section*{Acknowledgments}

The author acknowledges the symbolic computation toolbox in MATLAB, which was very useful in several parts of this paper.

\vspace{-2mm}
\bibliographystyle{amsplain}
\bibliography{References} 

\vspace{7mm}

\end{document}